\documentclass{article}

\PassOptionsToPackage{numbers,compress}{natbib}
\usepackage[preprint]{neurips_2026}

\usepackage{hyperref}
\usepackage{url}
\usepackage{enumitem}
\usepackage{tcolorbox}

\usepackage{amsmath,amssymb,amsthm}

\usepackage{amsmath,amsfonts,bm}

\newtcolorbox{takeawaybox}{
    enhanced,
    breakable,
    colback=gray!3,
    colframe=black!60,
    boxrule=0.5pt,
    arc=1mm,
    left=2mm,
    right=2mm,
    top=2mm,
    bottom=2mm,
    overlay={
        \node[
            anchor=west,
            fill=white,
            inner xsep=1mm,
            font=\bfseries\small
        ] at ([xshift=2mm]frame.north west) {Takeaway};
    },
}

\def\secref#1{section~\ref{#1}}
\def\Secref#1{Section~\ref{#1}}

\def\eqref#1{equation~\ref{#1}}

\def\algref#1{algorithm~\ref{#1}}

\def\1{\bm{1}}

\DeclareMathAlphabet{\mathsfit}{\encodingdefault}{\sfdefault}{m}{sl}
\SetMathAlphabet{\mathsfit}{bold}{\encodingdefault}{\sfdefault}{bx}{n}

\newcommand{\E}{\mathbb{E}}

\newcommand{\R}{\mathbb{R}}

\providecommand{\R}{\mathbb{R}} %

\providecommand{\cN}{\mathcal{N}}

\newenvironment{talign*}
{\csname align*\endcsname}
{\endalign}

\usepackage[utf8]{inputenc}         %
\usepackage[T1]{fontenc}            %
\usepackage{url}                    %
\usepackage{booktabs}               %
\usepackage{tabularx}               %
\usepackage{amsfonts}               %
\usepackage{nicefrac}               %
\usepackage{microtype}              %
\usepackage{xcolor}                 %
\usepackage{algorithm}
\usepackage{algpseudocode}
\usepackage{graphicx}
\usepackage{subcaption}
\usepackage[flushleft]{threeparttable}
\usepackage{float}
\usepackage{multirow}
\usepackage{makecell}
\usepackage{xspace}
\usepackage{enumitem}
\usepackage[font=small]{caption}
\usepackage{autobreak}
\usepackage{sidecap}
\usepackage{wrapfig}
\usepackage{bbding}
\usepackage[toc, page, header]{appendix}
\usepackage{tikz}
\usetikzlibrary{calc,trees,positioning,arrows,chains,shapes.geometric,%
    decorations.pathreplacing,decorations.pathmorphing,shapes,%
    matrix,shapes.symbols}
\usepackage{xcolor}
\usepackage{pifont}
\usepackage{mdframed}
\usepackage{colortbl}

\providecommand{\fullcheck}{\mbox{\ding{52}}}
\providecommand{\fullcross}{\mbox{\ding{55}}}
\providecommand{\halfcheck}{\mbox{\ding{52}\rotatebox[origin=c]{-9.2}{\kern-0.7em\ding{55}}}}

\usepackage{tcolorbox}
\tcbuselibrary{skins, breakable, theorems}
\usepackage{empheq}
\usepackage{arydshln}
\usepackage{bm}
\usepackage[capitalize]{cleveref}
\usepackage{listings}

\hypersetup{
    colorlinks=true,
    linkcolor=blue,
    citecolor=blue,
    urlcolor=blue
}

\definecolor{coral}{RGB}{255,127,80}
\definecolor{darkgreen}{RGB}{0,100,0}
\definecolor{darkyellow}{RGB}{204,153,0}
\definecolor{salmon}{RGB}{250,128,114}
\definecolor{darkred}{RGB}{150,0,0}
\newcommand{\darkredtext}[1]{{\color{darkred}#1}}

\definecolor{eqbg}{gray}{0.95}
\definecolor{indomaincolor}{rgb}{0.9, 0.95, 1.0}
\definecolor{oodcolor}{rgb}{1.0, 0.9, 0.9}

\newcommand{\transparentgray}[1]{%
    \tikz[baseline=(X.base)] \node[fill=gray, fill opacity=0.1, text opacity=1, inner sep=2pt, outer sep=0pt] (X) {#1};%
}

\newcommand{\transparentyellow}[1]{%
    \tikz[baseline=(X.base)] \node[fill=yellow, fill opacity=0.1, text opacity=1, inner sep=2pt, outer sep=0pt] (X) {#1};%
}

\renewcommand{\secref}[1]{\hyperref[#1]{\darkredtext{Sec.~\ref*{#1}}}}
\renewcommand{\Secref}[1]{\hyperref[#1]{\darkredtext{Sec.~\ref*{#1}}}}
\providecommand{\thmref}[1]{\hyperref[#1]{\darkredtext{Thm.~\ref*{#1}}}}
\providecommand{\defref}[1]{\hyperref[#1]{\transparentgray{Definition~\ref*{#1}}}}
\providecommand{\propref}[1]{\hyperref[#1]{\darkredtext{Prop.~\ref*{#1}}}}
\providecommand{\assumpref}[1]{\hyperref[#1]{\darkredtext{Assump.~\ref*{#1}}}}
\providecommand{\remarkref}[1]{\hyperref[#1]{\transparentyellow{Remark~\ref*{#1}}}}
\providecommand{\conjref}[1]{\hyperref[#1]{\darkredtext{Conj.~\ref*{#1}}}}
\providecommand{\lemref}[1]{\hyperref[#1]{\darkredtext{Lem.~\ref*{#1}}}}
\providecommand{\corref}[1]{\hyperref[#1]{\darkredtext{Cor.~\ref*{#1}}}}
\providecommand{\noteref}[1]{\hyperref[#1]{\darkredtext{Nota.~\ref*{#1}}}}
\providecommand{\claimref}[1]{\hyperref[#1]{\darkredtext{Clm.~\ref*{#1}}}}
\providecommand{\algref}[1]{\hyperref[#1]{\darkredtext{Alg.~\ref*{#1}}}}
\providecommand{\algmref}[1]{\hyperref[#1]{\darkredtext{Alg.~\ref*{#1}}}}
\providecommand{\figref}[1]{\hyperref[#1]{\darkredtext{Fig.~\ref*{#1}}}}
\providecommand{\tabref}[1]{\hyperref[#1]{\darkredtext{Tab.~\ref*{#1}}}}
\providecommand{\appref}[1]{\hyperref[#1]{\darkredtext{App.~\ref*{#1}}}}

\newtheoremstyle{professional}
{10pt} 
{10pt} 
{\itshape} 
{} 
{\bfseries} 
{.} 
{.5em} 
{} 

\theoremstyle{professional}
\newtheorem{myth}{Theorem}[section]
\newtheorem{myprop}[myth]{Proposition}
\newtheorem{mylem}[myth]{Lemma}
\newtheorem{mycor}[myth]{Corollary}
\newtheorem{mydef}[myth]{Definition}
\newtheorem{myassump}[myth]{Assumption}
\newtheorem{myrem}[myth]{Remark}
\newtheorem{myhyp}[myth]{Hypothesis}
\newtheorem{myconj}[myth]{Conjecture}
\newtheorem{mynota}[myth]{Notation}
\newtheorem{myclaim}[myth]{Claim}
\newtheorem{myprob}[myth]{Problem}
\newtheorem{myobs}[myth]{Observation}

\tcbset{
    thmbox/.style={
            enhanced,
            breakable,
            sharp corners,
            boxrule=0pt,
            leftrule=3pt,
            top=0pt,
            bottom=0pt,
            left=5pt,
            right=5pt,
            before skip=10pt,
            after skip=10pt,
        }
}

\newenvironment{theorem}{\begin{tcolorbox}[thmbox, colback=red!5!white, colframe=red!75!black]\begin{myth}}{\end{myth}\end{tcolorbox}}
\newenvironment{proposition}{\begin{tcolorbox}[thmbox, colback=blue!5!white, colframe=blue!75!black]\begin{myprop}}{\end{myprop}\end{tcolorbox}}
\newenvironment{lemma}{\begin{tcolorbox}[thmbox, colback=cyan!5!white, colframe=cyan!75!black]\begin{mylem}}{\end{mylem}\end{tcolorbox}}

\newenvironment{assumption}{\begin{tcolorbox}[thmbox, colback=green!5!white, colframe=green!75!black]\begin{myassump}}{\end{myassump}\end{tcolorbox}}

\newtheorem{innernote}{Note}

\newtheorem{innerexercise}{Exercise}

\definecolor{codegreen}{rgb}{0,0.6,0}
\definecolor{codegray}{rgb}{0.5,0.5,0.5}
\definecolor{codepurple}{rgb}{0.58,0,0.82}
\definecolor{backcolour}{rgb}{0.95,0.95,0.92}

\lstdefinestyle{mystyle}{
    backgroundcolor=\color{backcolour},
    commentstyle=\color{codegreen},
    keywordstyle=\color{magenta},
    numberstyle=\tiny\color{codegray},
    stringstyle=\color{codepurple},
    basicstyle=\ttfamily\footnotesize,
    breakatwhitespace=false,
    breaklines=true,
    captionpos=b,
    keepspaces=true,
    numbers=left,
    numbersep=5pt,
    showspaces=false,
    showstringspaces=false,
    showtabs=false,
    tabsize=2
}
\hypersetup{
  pdftitle={When Can Conditional Flow Matching Replace Pointwise Negative Log-Likelihood?},
  pdfauthor={Yansen Han, Hongxin Sun, Tao Lin}
}

\title{When Can Conditional Flow Matching Replace Pointwise Negative Log-Likelihood?}

\author{
Yansen Han$^{1,3}$ \quad
Hongxin Sun$^{2,3}$ \quad
Tao Lin$^{3}$\thanks{Corresponding author}
\\[0.5em]
$^1$Zhejiang University \quad
$^2$Fudan University \quad
$^3$Westlake University
}

\date{}

\begin{document}

\raggedbottom

\maketitle

\begin{abstract}
    Flow matching enables likelihood-free training, yet alignment methods increasingly reuse conditional flow matching (CFM) losses as endpoint negative log-likelihoods (NLLs) and their old/new differences as log-likelihood ratios. We characterize when these substitutions are valid.  For linear Gaussian paths, we exactly decompose endpoint NLL into entropy, a weighted CFM objective, an interior velocity--score residual, and a boundary residual.  Thus CFM-only estimates and differences are exact only when the corresponding residuals cancel.  At the off-policy population optimum, ordinary CFM is not generally a pointwise NLL estimator, whereas \(w_{\mathrm{sc}}(t)=(1-t)/t\) removes the interior residual; this positive result does not extend generally to training or on-policy alignment.  On-policy log-ratios can remain biased even for identical endpoint laws or after surrogate optimization.  Experiments across dimensions, distributions, and geometries support these conclusions and the mechanisms that make inexact ratios useful.  \textbf{More broadly, the decomposition provides a theoretical basis for adapting likelihood-based LLM methods to flow matching, while distinguishing exact substitutions from controlled surrogates.}
\end{abstract}

\begin{table*}[htbp]
    \centering
    \scriptsize
    \setlength{\tabcolsep}{3.2pt}
    \renewcommand{\arraystretch}{1.18}
    
    \caption{
    Summary of when CFM-only objectives or ratios can replace clean NLL or clean likelihood ratios;
    \secref{sec:experiments} summarizes the empirical support. \fullcheck{}: identity-level replacement up to the endpoint boundary term
    \(\mathcal B_\varepsilon\to0\);
    \halfcheck{}: task-level or local surrogate validity;
    \fullcross{}: no general guarantee.
    }
    \label{tab:cfm-nll-answer}
    
    \begin{tabularx}{\textwidth}{
    >{\raggedright\arraybackslash}p{0.12\textwidth}
    >{\raggedright\arraybackslash}X
    >{\raggedright\arraybackslash}X
    >{\raggedright\arraybackslash}X
    >{\raggedright\arraybackslash}X
    }
    \toprule
    & \multicolumn{2}{c}{\textbf{Off-policy: fixed-target CFM}}
    & \multicolumn{2}{c}{\textbf{On-policy: CFM-ratio surrogate}} \\
    \cmidrule(lr){2-3}\cmidrule(lr){4-5}
    
    \textbf{Stage}
    & \textbf{Ordinary CFM}
    & \textbf{Score-calibrated CFM}
    & \textbf{CFM-only ratio}
    & \textbf{Score-calibrated ratio} \\
    \midrule
    
    \textbf{During optimization}
    &
    \fullcross{} No identity guarantee.
    Away from the fixed-target CFM optimum, \(\mathcal G\) need not vanish.
    \newline
    \emph{See} \propref{prop:pw-linear-training-counterexample}.
    &
    \fullcross{} No identity guarantee.
    The score-calibrated weight does not remove the interior gap \(\mathcal G\) during optimization.
    \newline
    \emph{See} \propref{prop:pw-linear-training-counterexample}.
    &
    \halfcheck{} Empirical controlled-update surrogate.
    Useful in the tested settings, but biased in general.
    \newline
    \emph{See} \propref{prop:onpolicy-cfm-ratio-bias}.
    &
    \halfcheck{} Empirical controlled-update surrogate.
    Useful in the tested settings, but biased in general.
    \newline
    \emph{See} \propref{prop:onpolicy-rotational-mismatch}.
    \\
    
    \addlinespace[0.35em]
    
    \textbf{After optimization}
    &
    \fullcross{} No pointwise identity.
    Ordinary CFM recovers the population-optimal velocity, but its pointwise objective is not likelihood-calibrated.
    \newline
    \emph{See} \propref{prop:pw-linear-optimum-calibration}.
    &
    \fullcheck{} Identity-level replacement.
    At the score-calibrated CFM optimum,
    interior gap \(\mathcal G=0\) with \(w=w_{\mathrm{sc}}\).
    \newline
    \emph{See} \propref{prop:pw-linear-optimum-calibration}.
    &
    \fullcross{} No general guarantee.
    Optimizing CFM-only ratio does not enforce the relative calibration condition, i.e.,
    \(\Delta\mathcal G+\Delta\mathcal B=0\).
    \newline
    \emph{See} \propref{prop:onpolicy-after-cfm-ratio-gap-nonzero}.
    &
    \fullcross{} No general guarantee.
    Even score calibration is not enough unless the learned ratio satisfies the relative calibration condition.
    \newline
    \emph{See} \propref{prop:onpolicy-after-cfm-ratio-gap-nonzero}.
    \\
    
    \bottomrule
    \end{tabularx}
    
    \vspace{0.25em}
    \begin{minipage}{0.96\textwidth}
    \scriptsize
    \textbf{Note.}
    ``After optimization'' means different things in the two settings.
    For off-policy, it refers to convergence to the fixed-target CFM population optimum.
    For on-policy, it refers to optimizing the CFM-only ratio surrogate, which does not by itself guarantee exact clean-ratio calibration.
    Ordinary CFM uses \(w=1\), while score-calibrated CFM uses \(w=w_{\mathrm{sc}}\).
    \end{minipage}
\end{table*}    

\clearpage
\section{Introduction}
\label{sec:introduction}

Flow matching is attractive partly because it replaces likelihood-based training with supervised regression along prescribed probability paths~\citep{lipman2022flow,albergo2022building,liu2022flow,sun2025unified}. However, its likelihood-free treatment becomes bottlenecks in reward-based post-training: PPO- and GRPO-style updates require per-sample likelihood ratios~\citep{schulman2017proximal,shao2024deepseekmath}, while evaluating the endpoint density of a continuous generator is expensive. Recent forward-process methods avoid estimating likelihoods by using samplewise CFM to construct likelihood surrogates~\citep{xue2025advantage,mcallister2025fmpg,zheng2025diffusionnft,tang2026vgrpo,choi2026rethinking}. These designs are computationally appealing and stay close to pretraining, but the likelihood-surrogate case raises a basic question:

\begin{center}
\begin{tcolorbox}[
    enhanced,
    colback=gray!3,
    colframe=black!55,
    boxrule=0.45pt,
    arc=1mm,
    width=0.92\linewidth,
    left=1.8mm,
    right=1.8mm,
    top=1.2mm,
    bottom=1.2mm
]
\centering
\emph{When is replacing pointwise likelihoods by CFM-only quantities exact, and what makes the replacement useful when it is not exact?}
\end{tcolorbox}
\end{center}

The difficulty is a mismatch between likelihoods and CFM-based surrogates.  For a fixed endpoint \(x_0\), the pointwise CFM objective averages a squared velocity-regression error along noisy conditional paths, whereas the pointwise NLL \(-\log p_0^\theta(x_0)\) is determined by the endpoint marginal density. Expectation-level equivalence does not guarantee pointwise equivalence. Thus, a model may recover the population-optimal velocity while its pointwise CFM objective differs from the NLL for individual endpoints. The requirement is stronger in on-policy alignment: the difference between old and new pointwise CFM objectives must track the \emph{change} in endpoint log-density for every endpoint whose contribution is reweighted or clipped. What is missing is therefore a pointwise criterion that distinguishes an exact likelihood identity from a reward-effective but biased surrogate.

Continuous normalizing flows permit endpoint likelihood evaluation, and diffusion objectives admit ELBO or score-matching likelihood interpretations~\citep{chen2018neural,song2021maximum,lu2022maximum,kingma2023understanding,zheng2023improved}. Neither result makes a samplewise CFM objective the clean endpoint NLL. Reverse-process RL instead optimizes ratios over discretized denoising trajectories~\citep{black2023training,liu2025flow,li2025mixgrpo,he2025tempflow}, which are not generally the corresponding marginalized endpoint ratios. Forward-process methods construct CFM-based likelihood surrogates to replace the clean endpoint log-ratio~\citep{xue2025advantage,zheng2025diffusionnft,bergmeister2026reinforce,mcallister2025fmpg,tang2026vgrpo,choi2026rethinking}. Our analysis is to identify when the CFM-based surrogates are exact and when they are useful local surrogates in post-training.

We theoretically characterize this exactness of substitution by exposing the exact pointwise gap.  For the linear Gaussian path and a fixed clean endpoint \(x_0\), we prove the decomposition
\begin{align*}
    -\log p_0^\theta(x_0)
    &=
    H(q_\varepsilon^{x_0})
    +
    \E_{t \in U[\varepsilon,1], X_t\sim q_t^{x_0}}
    \bigl[ w(t)\, \|v_t^\theta(X_t;x_0) - u_t^{x_0}(X_t)\|^2 \bigr]\,\\
    &\quad
    +
    \mathcal G_{\varepsilon,w}(\theta;x_0)
    +
    \mathcal B_\varepsilon(\theta;x_0).
\end{align*}
Here \(\E\bigl[ w(t)\, \|v_t^\theta(X_t;x_0) - u_t^{x_0}(X_t)\|^2 \bigr]\) is the weighted pointwise CFM objective, \(\mathcal G_{\varepsilon,w}\) is an interior residual coupling the velocity and score gaps, and \(\mathcal B_\varepsilon\) is the boundary residual introduced by positive-time smoothing. The identity is both a diagnosis and an exactness criterion: a CFM-only estimate equals endpoint NLL precisely when \(\mathcal G_{\varepsilon,w}+\mathcal B_\varepsilon=0\), and an old/new CFM difference is a clean endpoint log-ratio precisely when the corresponding relative residual cancels. The criterion produces sharply different conclusions for off-policy and on-policy uses, summarized in \tabref{tab:cfm-nll-answer}.

Our experiments examine the theory at three levels: numerical examination of the decomposition identity, its off-policy and on-policy consequences, and the stable use of an inexact ratio.  First, across 1D and 2D distributions and high-dimensional Gaussian mixtures with tractable endpoint densities, independently estimated terms demonstrate that the decomposition identity is more accurate than the CFM-only estimate in every off-policy experiment. Using \(w_{\mathrm{sc}}\) reduces fixed-target CFM-only NLL MAE from \(0.800\) to \(0.223\) in 1D, from \(1.294\) to \(0.377\) in 2D, and from \(19.402\) to \(3.012\) in 32D. Second, the on-policy experiments show that optimization success does not imply global ratio accuracy: CFM-ratio training reaches rewards of \(0.995\), \(0.981\), and \(0.910\) in representative 1D, 2D, and 32D experiments, while the corresponding log-ratio MAE reaches \(7.900\), \(1.570\), and \(1382.955\).  Third, EMA and clipping sweeps measure associations between controlled reference and reward stability. MNIST and CIFAR-10 experiments test only whether the results for successful optimization using inexact ratios persist for real-image datasets.

Our contributions are threefold:
\begin{itemize}[leftmargin=1.35em,itemsep=0.15em,topsep=0.25em]
\item \textbf{An exact pointwise NLL decomposition.}
For Gaussian paths, we decompose endpoint NLL into entropy, CFM objective, an interior velocity--score residual, and a boundary residual.  This turns pointwise likelihood substitution from an experimental usefulness into an analytical quantity.
\item \textbf{A clear separation between off-policy and on-policy regimes.}
At the off-policy population optimum, we show that ordinary CFM is not generally a pointwise NLL estimator, whereas \(w_{\mathrm{sc}}\) can narrow the gap between CFM and NLL. For on-policy use, we derive the exact bias and gap-tilted reward of CFM-derived ratios and construct counterexamples showing that an exact endpoint log-ratio is not guaranteed.
\item \textbf{A cross-scale numerical and controlled empirical study.}
Experiments from 1D to 32D verify the decomposition identity and test its off-policy and on-policy consequences. EMA, clipping, and regularization sweeps study which mechanisms are associated with stable optimization using biased CFM-derived ratios. Raw-image experiments test whether these stablization effects can be transferred to real-image datasets.
\end{itemize}

\section{Related Work}
\label{sec:related-work}

\paragraph{Likelihood and distribution-level analyses.}
Continuous normalizing flows evaluate likelihoods by integrating vector field divergence~\citep{chen2018neural,song2020score}. Related analyses connect weighted score matching to maximum likelihood, identify limitations of first-order score objectives for ODE likelihood, develop higher-order objectives for likelihood fine-tuning, and interpret monotone-weighted diffusion objectives as ELBOs on noise-augmented data~\citep{song2021maximum,lu2022maximum,zheng2023improved,kingma2023understanding}. Other work bounds distribution or probability-path error through log-mass conservation, divergence-augmented CFM, or end-to-end Wasserstein analysis~\citep{benhamu2022matching,huang2025divergence,zhou2025error}. These results characterize ODE likelihoods, variational bounds, or distribution-level errors, but do not establish the relationship between the pointwise CFM objective and the NLL. We instead derive a pointwise identity whose residuals determine when a CFM objective represents an endpoint NLL and when an old/new difference represents an endpoint log-ratio.

\paragraph{Likelihood ratios over reverse processes.}
PPO and preference objectives rely on old/new likelihood ratios or related log-density differences~\citep{schulman2017proximal,shao2024deepseekmath,rafailov2023direct,wallace2024diffusion}. Diffusion RL commonly treats a discretized denoising process as the policy, yielding a tractable but sampler-dependent trajectory ratio~\citep{black2023training,liu2025flow,li2025mixgrpo,he2025tempflow}. This trajectory ratio need not equal the clean endpoint ratio. We take the endpoint ratio as the target and study when a CFM-only estimate can recover the endpoint ratio.

\paragraph{Likelihood ratios over forward processes.}
Closest to our setting, FPO constructs a PPO-style ratio from differences of conditional flow-matching losses~\citep{mcallister2025fmpg}, whereas other recent methods use ELBO-based likelihood surrogates~\citep{tang2026vgrpo,choi2026rethinking}. A separate line of work directly reweights matching losses, adapts preference or reward weighting, contrasts positive and negative generations, or learns reward-corrected consistency targets~\citep{xue2025advantage,liu2025improving,zheng2025diffusionnft,bergmeister2026reinforce}. These objectives do not require the matching loss to equal an endpoint ratio. Our decomposition identity specifically characterizes CFM-derived ratios and identifies when they are exact endpoint ratios.

\section{Preliminaries}
\label{sec:preliminaries}

We give the notation needed for the replacement question. Let $p_1=\mathcal N(0,I_d)$ be the base distribution, $p_0$ denote clean data, and a flow model be defined by a vector field $v_\theta:\mathbb R^d\times[0,1]\to\mathbb R^d$ whose marginals satisfy the continuity equation
\begin{equation}
\partial_t p_t^\theta+\nabla\cdot(p_t^\theta v_\theta)=0,
\qquad p_1^\theta=p_1.
\label{eq:prelim-model-continuity}
\end{equation}
The endpoint likelihood is
\begin{equation}
\ell_{\mathrm{NLL}}(\theta;x_0):=-\log p_0^\theta(x_0),
\label{eq:prelim-pointwise-nll}
\end{equation}
which can be accurately computed, but is usually avoided in training due to its expensive cost.

\subsection{Gaussian Paths and Conditional Flow Matching}
\label{subsec:gaussian-paths-cfm}

For a linear Gaussian path, $X_t= (1-t) X_0 + t X_1, X_0\sim p_0, X_1\sim p_1$, with $X_1$ independent of $X_0$, the conditional law is $q_t(\cdot\mid X_0)=\mathcal N((1-t) X_0,t^2I_d)$ and the conditional velocity is
$u_t(x\mid X_0) = \frac{1}{t}(x - X_0)$. The weighted conditional flow matching objective is
\begin{equation}
\mathcal L_{\mathrm{CFM}}(\theta)
:=
\int_0^1 w(t)\,
\mathbb E\!\left[
\|v_\theta(X_t,t)-u_t(X_t\mid X_0)\|^2
\right]dt .
\label{eq:prelim-fm-objective}
\end{equation}
Throughout, we refer to CFM with $w(t)=1$ as \textbf{ordinary CFM}, and to CFM with $w_{sc}(t)=\frac{1-t}{t}$ as \textbf{score-calibrated CFM}.

\subsection{Clean Ratios for Preference Alignment}
\label{subsec:diffusion-rl-background}

For a condition \(c\), a reward \(R(x_0,c)\), and an old model \(\theta_{\mathrm{old}}\), the ideal update increases \(\mathbb E_{c, x_0\sim p_0^\theta(\cdot\mid c)}[R(x_0,c)]\) while keeping the new endpoint law close to \(p_0^{\theta_{\mathrm{old}}}(\cdot\mid c)\). PPO/GRPO-style updates use samples from the old model, so the reward objective is reweighted by an old/new likelihood ratio.  The clean endpoint ratio is
\begin{equation}
\rho_\theta(x_0,c)
:=
\frac{p_0^\theta(x_0\mid c)}
{p_0^{\theta_{\mathrm{old}}}(x_0\mid c)} .
\label{eq:prelim-clean-ratio}
\end{equation}
Equivalently, \(\rho_\theta\) is the clean old/new endpoint ratio used for importance weighting, clipping, or trust region control.  This is the ratio that preference alignment would use if likelihoods were available.
 

\section{A Pointwise Gap Decomposition Between NLL and CFM}
\label{sec:pointwise-gap-decomposition}

This section gives the decomposition identity used by all later cases. It explains what must be true for a pointwise CFM, or a difference of two CFM, to serve as a clean endpoint NLL or clean log-ratio.

\paragraph{Setup.}
Fix a clean endpoint $x_0\in\R^d$ and specialize the Gaussian path in \secref{subsec:gaussian-paths-cfm} to
\[
X_t=(1-t)x_0+tX_1 \sim q_t^{x_0}=\cN((1-t)x_0,t^2I_d), \qquad X_1\sim\cN(0,I_d).
\]
Let $u_t^{x_0}(x)=\frac{x-x_0}{t}$ be the velocity field corresponding to $q_t^{x_0}$.  We compare this conditional path $(q_t^{x_0},u_t^{x_0})$ to the model path $(p_t^\theta,v_\theta)$ through the velocity and score gaps:
\begin{equation*}
\Delta v_t^\theta(x;x_0)
:=
v_\theta(x,t)-u_t^{x_0}(x),
\qquad
\Delta s_t^\theta(x;x_0)
:=
\nabla\log q_t^{x_0}(x) - \nabla\log p_t^\theta(x)
\end{equation*}

Because $q_t^{x_0}$ becomes a Dirac mass at $t=0$, we work at positive time
$\varepsilon\in(0,1)$ and define
\begin{equation*}
\ell_\varepsilon(\theta;x_0)
:=
\E_{X_\varepsilon\sim q_\varepsilon^{x_0}}
\bigl[-\log p_\varepsilon^\theta(X_\varepsilon)\bigr],
\quad
\mathcal J_w^{[\varepsilon,1]}(\theta;x_0)
:=
\int_\varepsilon^1
w(t)\,
\E_{X_t\sim q_t^{x_0}}
\bigl[
\|\Delta v_t^\theta(X_t;x_0)\|^2
\bigr]\,dt
\end{equation*}
We use the following regularity assumptions to guarantee the validity of differentiation under the integral sign, integration by parts, and the
endpoint limit.

\begin{assumption}[Pathwise regularity]
\label{ass:pw-pathwise-regularity}
For the fixed $x_0$, the densities $q_t^{x_0}$ and $p_t^\theta$ are strictly positive and $C^1$ in $x$ for $t\in(0,1)$.  Their continuity equations hold in strong form, $t\mapsto\mathrm{KL}(q_t^{x_0}\|p_t^\theta)$ is differentiable, and the integrations by parts used below have no boundary terms.
\end{assumption}

\begin{assumption}[Endpoint regularity]
\label{ass:pw-endpoint-regularity}
The map $f_\theta(x,t):=-\log p_t^\theta(x)$ is jointly continuous at
$(x_0,0)$.  Moreover, for some $\varepsilon_0>0$, $C>0$, and $m\ge1$,
\begin{equation*}
|f_\theta(x,t)|\le C(1+\|x\|^m),
\qquad
x\in\R^d,\quad t\in[0,\varepsilon_0].
\end{equation*}
\end{assumption}

The \eqref{eq:pw-exact-identity} shows why the CFM objective is not a likelihood. Instead, the endpoint NLL is governed by a mixed velocity--score term.

\begin{theorem}[Pointwise NLL is a CFM term plus residuals]
\label{thm:pw-practical-decomposition}
Suppose \assumpref{ass:pw-pathwise-regularity} and \assumpref{ass:pw-endpoint-regularity} hold. Then, for every positive weight $w$ and every $\varepsilon\in(0,1)$,
\begin{align}
\ell_\varepsilon(\theta;x_0)
&=
H(q_\varepsilon^{x_0})
+
\int_\varepsilon^1
\E_{q_t^{x_0}}
\bigl[
\langle \Delta v_t^\theta(X_t;x_0), \Delta s_t^\theta(X_t;x_0)\rangle
\bigr]\,dt,
\label{eq:pw-exact-identity}\\
-\log p_0^\theta(x_0)
&=
H(q_\varepsilon^{x_0})
+
\mathcal J_w^{[\varepsilon,1]}(\theta;x_0)
+
\mathcal G_{\varepsilon,w}(\theta;x_0)
+
\mathcal B_\varepsilon(\theta;x_0),
\label{eq:pw-practical-decomposition}
\end{align}
where
\begin{align*}
&\mathcal G_{\varepsilon,w}(\theta;x_0) :=\int_\varepsilon^1\E_{q_t^{x_0}}\bigl[\langle \Delta v_t^\theta(X_t;x_0), \Delta s_t^\theta(X_t;x_0)-w(t)\Delta v_t^\theta(X_t;x_0)\rangle\bigr]\,dt,\\
H(q_\varepsilon^{x_0})&=\frac d2\log(2\pi e\,\varepsilon^2), \quad 
\mathcal B_\varepsilon(\theta;x_0) :=-\log p_0^\theta(x_0)-\ell_\varepsilon(\theta;x_0), \quad
\lim_{\varepsilon\downarrow0}\mathcal B_\varepsilon(\theta;x_0)=0.
\end{align*}
\end{theorem}

The proof is in \appref{sec:proof-of-theorem-pw-practical-decomposition}.

\section{Off-Policy Setting: Fixed-Target CFM Estimates}
\label{sec:offpolicy-cfm-nll}

This section justifies the two off-policy columns of \tabref{tab:cfm-nll-answer}.  The target endpoint distribution is fixed while the velocity model is optimized.  This covers ordinary CFM pretraining and reward-weighted fixed-target variants, where the reward only changes the target endpoint law.  We omit the reward notation and write the fixed target as \(p_0\).

\paragraph{During optimization: no CFM weight gives a general identity.}
\thmref{thm:pw-practical-decomposition} shows that a CFM estimate is exact only when the interior and boundary residuals are negligible.  During optimization, the model is not constrained to satisfy the pointwise calibration relation $\Delta s_t^\theta(x;x_0)=w(t)\Delta v_t^\theta(x;x_0)$, so neither ordinary CFM nor score-calibrated CFM is a training-time pointwise NLL estimate.

\begin{proposition}[Training-time score calibration can fail]
\label{prop:pw-linear-training-counterexample}Under the one-dimensional linear Gaussian path $X_t=(1-t)X_0+tZ$ with $X_0,Z\sim\mathcal N(0,1)$ independent and $w_{\mathrm{sc}}(t)=(1-t)/t$, pointwise score calibration can fail at a simple intermediate checkpoint. If \(v_\theta(x,t)=a_\theta(x,t)x\), and \(a_\theta = 0\) at some checkpoint $\theta$, then, for $t\in(0,1)$ and $x_0 = 0$,
\[
\Delta s_t^\theta(x;0)-w_{\mathrm{sc}}(t)\Delta v_t^\theta(x;0)
=
\left(1-\frac1t\right)x
\neq 0
\]
Moreover, for every $\varepsilon\in(0,1)$,
\[
\mathcal G_{\varepsilon,w_{\mathrm{sc}}}(\theta;0)
=\frac{(1-\varepsilon)^2}{2}>0.
\]
\end{proposition}
The proof is in \appref{sec:proof-of-proposition-pw-linear-training-counterexample}. This proposition gives a simple example of why the two \fullcross{} entries in the off-policy ``during optimization'' row of \tabref{tab:cfm-nll-answer} are not identities.

\paragraph{After optimization: score calibration removes the interior gap.}
At a population optimum for the linear Gaussian interpolation path, the score-calibrated weight \(w_{\mathrm{sc}}(t)=(1-t)/t\) removes the interior velocity--score residual for any fixed endpoint distribution. The remaining difference from endpoint NLL is the positive-time boundary term and finite numerical error.

\begin{proposition}[Fixed-target population score calibration]
\label{prop:pw-linear-optimum-calibration}
Let \(X_0\sim p_0\) have finite second moment, let
\(Z\sim\mathcal N(0,I_d)\) be independent, and
\(X_t=(1-t)X_0+tX_1\). Let \(\theta^\star\) realize the population CFM optimum
\[
v_{\theta^\star}(x,t)=\mathbb E[X_1-X_0\mid X_t=x].
\]
Then, for every \(x_0,x\in\mathbb R^d\) and \(t\in(0,1)\),
\[
\Delta s_t^{\theta^\star}(x;x_0)
=\frac{1-t}{t}\Delta v_t^{\theta^\star}(x;x_0).
\]
Consequently, for \(w_{\mathrm{sc}}(t)=(1-t)/t\) and every
\(\varepsilon\in(0,1)\),
\[
\mathcal G_{\varepsilon,w_{\mathrm{sc}}}(\theta^\star;x_0)=0,
\qquad
\ell_\varepsilon(\theta^\star;x_0)
=
H(q_\varepsilon^{x_0})
+
\mathcal J_{w_{\mathrm{sc}}}^{[\varepsilon,1]}(\theta^\star;x_0).
\]
\end{proposition}
The proof is in
\appref{sec:proof-of-proposition-pw-linear-optimum-calibration}.


\section{On-Policy Setting: CFM-Ratio Surrogates}
\label{sec:onpolicy-cfm-nll}

This section justifies the two on-policy columns of \tabref{tab:cfm-nll-answer}. For condition \(c\), reward \(R(x_0,c)\), and old endpoint density \(p_0^{\theta_{\mathrm{old}}}(\cdot\mid c)\), the log-ratio used in PPO/GRPO is
\[
\log
\frac{p_0^\theta(x_0\mid c)}
{p_0^{\theta_{\mathrm{old}}}(x_0\mid c)}
\]
Forward-process alignment methods use matching losses in different ways. AWM~\cite{xue2025advantage}, for example, directly advantage-weights the pretraining loss and does not require it to be an endpoint ratio. FPO~\cite{mcallister2025fmpg} is a direct instance of the class studied here: it replaces the log-ratio by a difference of CFM estimates,
\begin{equation}
\widehat\Delta_{\theta,\theta_{\mathrm{old}}}^{\mathrm{CFM}}(x_0,c)
:=
\mathcal J_w^{[\varepsilon,1]}(\theta_{\mathrm{old}};x_0,c)
-
\mathcal J_w^{[\varepsilon,1]}(\theta;x_0,c).
\label{eq:onpolicy-cfm-log-ratio}
\end{equation}
The question is whether \(\widehat\Delta^{\mathrm{CFM}}\) is the exact log-ratio, and whether it can be used as a surrogate for the exact log-ratio.

\subsection{The Exact Bias of a CFM-Only Ratio}
\label{subsec:onpolicy-ratio-bias}

Subtracting the pointwise decomposition for \(\theta\) and \(\theta_{\mathrm{old}}\) gives the exact bias term.

\begin{proposition}[Exact bias of the CFM log-ratio surrogate]
\label{prop:onpolicy-cfm-ratio-bias}
For every \(x_0\) and \(c\), we can decompose the log-ratio into a CFM-only ratio and residual terms:
\begin{equation}
\log \frac{p_0^\theta(x_0\mid c)}{p_0^{\theta_{\mathrm{old}}}(x_0\mid c)}
=
\widehat\Delta_{\theta,\theta_{\mathrm{old}}}^{\mathrm{CFM}}(x_0,c)
-
\Delta\mathcal G_{\varepsilon,w}(x_0,c)
-
\Delta\mathcal B_{\varepsilon}(x_0,c),
\label{eq:onpolicy-clean-cfm-ratio-decomp}
\end{equation}
where $\Delta\mathcal G_{\varepsilon,w} := \mathcal G_{\varepsilon,w}(\theta;x_0,c) - \mathcal G_{\varepsilon,w}(\theta_{\mathrm{old}};x_0,c)$ and $\Delta\mathcal B_{\varepsilon} := \mathcal B_{\varepsilon}(\theta;x_0,c) - \mathcal B_{\varepsilon}(\theta_{\mathrm{old}};x_0,c)$.
Thus,
\begin{equation}
\exp\!\left(
\widehat\Delta_{\theta,\theta_{\mathrm{old}}}^{\mathrm{CFM}}(x_0,c)
\right)
=
\frac{p_0^\theta(x_0\mid c)}
{p_0^{\theta_{\mathrm{old}}}(x_0\mid c)}
\exp\!\left(\Delta\mathcal G_{\varepsilon,w}
+
\Delta\mathcal B_{\varepsilon}\right).
\label{eq:onpolicy-multiplicative-bias}
\end{equation}
\end{proposition}
The proof is in \appref{sec:proof-of-proposition-onpolicy-cfm-ratio-bias}.
\propref{prop:onpolicy-cfm-ratio-bias} implies that an unclipped CFM-ratio
objective of the form above optimizes a gap-tilted reward:
\begin{align*}
\mathbb E_{X_0\sim p_0^{\theta_{\mathrm{old}}}}
\left[
\exp\!\left(
\widehat\Delta_{\theta,\theta_{\mathrm{old}}}^{\mathrm{CFM}}(X_0,c)
\right)R(X_0,c)
\right]
&=
\mathbb E_{X_0\sim p_0^\theta}
\left[
R(X_0,c)
\exp\!\left(
\delta_{\theta,\theta_{\mathrm{old}}}(X_0,c)
\right)
\right].
\end{align*}
Therefore, a CFM-only ratio is a clean endpoint ratio if and only if the
relative residual \(\Delta\mathcal G_{\varepsilon,w} + \Delta\mathcal B_{\varepsilon}\) vanishes on the
samples that matter to the update.

\subsection{Why the On-Policy Table Has No Global Identity Guarantee}
\label{subsec:onpolicy-no-global-guarantee}

The residual condition can fail where the endpoint distribution itself does not change.

\begin{proposition}[Mismatch during optimization]
\label{prop:onpolicy-rotational-mismatch}
Let \(X_0, X_1 \sim\mathcal N(0,I_2)\), \(X_t=(1-t)X_0+tX_1\), \(s_t^2=(1-t)^2+t^2\), and $v_{\mathrm{old}}^\star(x,t)=\mathbb E[X_1-X_0\mid X_t=x] =\frac{2t-1}{s_t^2}x$. Let \(R_0=\begin{psmallmatrix}0&-1\\1&0\end{psmallmatrix}\) and \(v_a(x,t)=v_{\mathrm{old}}^\star(x,t)+aR_0x$. Then, every \(v_a\) induces the same Gaussian marginal path and the same endpoint density, so
\[
\log \frac{p_0^a(x_0)}{p_0^{\theta_{\mathrm{old}}}(x_0)} = 0
\qquad\text{for all }x_0.
\]
However, for \(x_0=0\), \(w_{\mathrm{sc}}(t)=(1-t)/t\), and
\(C_\varepsilon:=\frac13-\varepsilon^2+\frac23\varepsilon^3>0\),
\[
\widehat\Delta_{a,\mathrm{old}}^{\mathrm{CFM}}(0,0)
=-C_\varepsilon a^2 \neq 0,
\qquad \varepsilon \in (0,1).
\]

Thus, the CFM log-ratio can be nonzero while the log-ratio is zero.
\end{proposition}
The proof is in
\appref{sec:proof-of-proposition-onpolicy-rotational-mismatch}.

\begin{proposition}[After optimizing a CFM-ratio surrogate, the gap may not vanish]
\label{prop:onpolicy-after-cfm-ratio-gap-nonzero}
In the setting of \propref{prop:onpolicy-rotational-mismatch}, consider a
one-sample regularized CFM-ratio objective with \(A_0=-r<0\) at \(x_0=0\):
\[
\mathcal L(a)
=
-r\exp\!\left(
\widehat\Delta_{a,\mathrm{old}}^{\mathrm{CFM}}(0)
\right)
-
\lambda a^2 .
\]
If $0<\lambda<rC_\varepsilon$, then $\mathcal L(a)$ is maximized at some
\(a^\star\neq0\).  At this optimizer,
\[
\log \frac{p_0^{a^\star}(0)}{p_0^{\theta_{\mathrm{old}}}(0)}=0,
\qquad
\widehat\Delta_{a^\star,\mathrm{old}}^{\mathrm{CFM}}(0)
=-C_\varepsilon(a^\star)^2\neq0 .
\]
\end{proposition}
The proof is in
\appref{sec:proof-of-proposition-onpolicy-after-cfm-ratio-gap-nonzero}.

\subsection{Why a Biased CFM Ratio May Still Be Useful}
\label{subsec:onpolicy-benignness}

The decomposition establishes bias, but by itself it does not establish a
stable optimization guarantee.  Three observations motivate the
controlled-variance experiments evaluated in \secref{sec:experiments}.  First, the CFM
ratio equals the clean ratio at the reference parameters because
\(\theta=\theta_{\mathrm{old}}\) implies
\(\Delta\mathcal G_{\varepsilon,w}+\Delta\mathcal B_\varepsilon=0\).
This is only a zero-order equality. Second, ratio clipping
bounds the proxy sample weight:
\[
\bar r_\theta^{\mathrm{CFM}}
=
\mathrm{clip}\!\left(
\exp(\widehat\Delta_{\theta,\theta_{\mathrm{old}}}^{\mathrm{CFM}}),
1-\epsilon_{\mathrm{clip}},1+\epsilon_{\mathrm{clip}}
\right),
\]
so an inexact CFM log-ratio cannot create an arbitrarily large weight on a single
sample.  Third, EMA mechanism and KL regularization control the
distance from the reference.  Closeness may reduce the relative residual under
additional continuity assumptions, but we do not prove such a bound here.
Accordingly, the experiments test the relationship between these controlled variables and
optimization stability in the experiment part.

\section{Experiments}
\label{sec:experiments}

The experiments numerically examine the decomposition identity in
\thmref{thm:pw-practical-decomposition} and test the claims in
\tabref{tab:cfm-nll-answer}.  We organize the evidence around three research
questions:
\begin{enumerate}[label={RQ\arabic*:}]
    \item Do independently estimated terms numerically demonstrate the decomposition identity in different settings?
    \item Do the claims in \tabref{tab:cfm-nll-answer} hold empirically in different settings?
    \item When the CFM-ratio surrogate is not an exact likelihood ratio, which mechanisms are associated with stable reward optimization?
\end{enumerate}
RQ1 and RQ2 use three synthetic experiments: 1D toy distributions, 2D
geometries, and high-dimensional Gaussian mixtures.  Their endpoint densities
are known, enabling direct evaluation of clean NLL, the decomposition, and
likelihood ratios.  RQ3 adds raw-pixel MNIST and CIFAR-10 as a fourth experiment to
test whether the effective mechanisms in the synthetic experiments appear in real-image models. Full protocols, seeds, numerical budgets, and detailed diagnostics are in \appref{app:experiment-details}.

For RQ1, the boundary term is defined as
\(\mathcal B_\varepsilon=-\log p_0^\theta-\ell_\varepsilon\).  Hence the
reported endpoint closure error satisfies, sample by sample,
\[
\left| -\log p_0^\theta-(H+\mathcal J_w+\mathcal G+\mathcal B_\varepsilon)\right|
=
\left|\ell_\varepsilon-(H+\mathcal J_w+\mathcal G)\right|.
\]
The endpoint NLL therefore cancels from this diagnostic.  RQ1 is a numerical
examination of the decomposition identity in \eqref{eq:pw-exact-identity}.

\subsection{RQ1: Do the independently estimated terms demonstrate the decomposition identity?}

\begin{figure}[ht]
\centering
\includegraphics[width=1\textwidth]{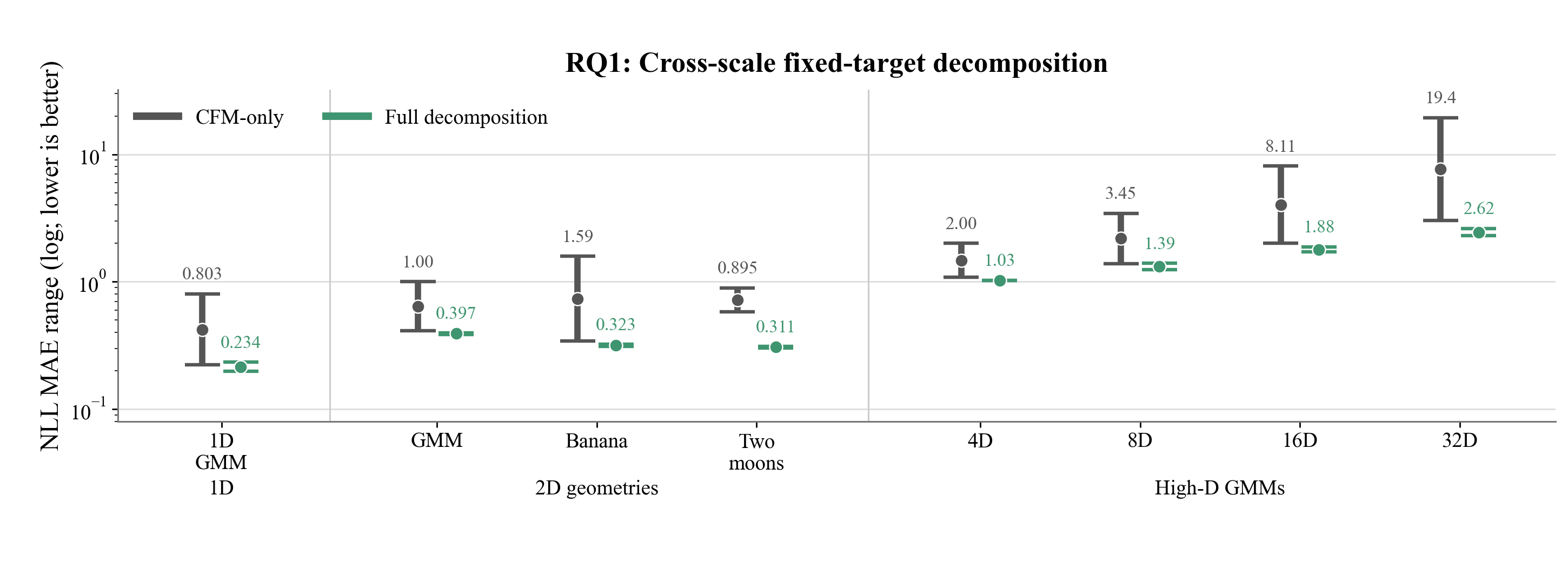}
\caption{
RQ1 cross-scale fixed-target numerical examination of the full decomposition estimation and CFM-only estimation over the three synthetic experiments.  Each vertical interval is the range over the reported fixed-target rows for the displayed setting. Lower is better, and the printed value is the interval maximum. "CFM-only" is the endpoint error of \(\mathcal J_w+H\). "Full decomposition" is computed as the endpoint error of \(\mathcal J_w+H+\mathcal G+\mathcal B\). The range construction and source rows are detailed in \appref{app:rq1-four-layer-figure-construction}.
}
\label{fig:rq1-four-layer-closure}
\end{figure}

\figref{fig:rq1-four-layer-closure} shows that \textbf{the independently estimated terms \(\mathcal J_w+H+\mathcal G\) numerically demonstrate the decomposition identity more accurately than \(\mathcal J_w+H\) estimates endpoint NLL throughout the fixed-target synthetic evaluations.}  
In 1D, the decomposition identity MAE is \(0.198\)--\(0.234\), compared with \(0.222\)--\(0.803\) for CFM-only endpoint estimates.  Across GMM, banana, and two moons in 2D, the corresponding ranges are \(0.302\)--\(0.397\) and \(0.342\)--\(1.587\).  In high-dimensional GMMs, decomposition identity error increases from about \(1.02\) in 4D to \(2.62\) in 32D.  This experiment checks the numerical evaluation of \eqref{eq:pw-exact-identity}, while the nonzero error reflects ODE, quadrature, score, and Monte Carlo approximation.

The on-policy experiments reveal a sharper numerical boundary.  The decomposition identity improves over the CFM-only endpoint estimate in 1D, 2D, and the moderate-dimensional GMM runs, but the estimated residual terms reach the \(10^9\) scale in the 32D setting.  We therefore report the 32D on-policy result as a failure of the available numerical estimator, not as evidence about the analytic identity.

\subsection{RQ2: Do the table conclusions hold empirically in different settings?}

\begin{figure}[ht]
\centering
\includegraphics[width=0.92\textwidth]{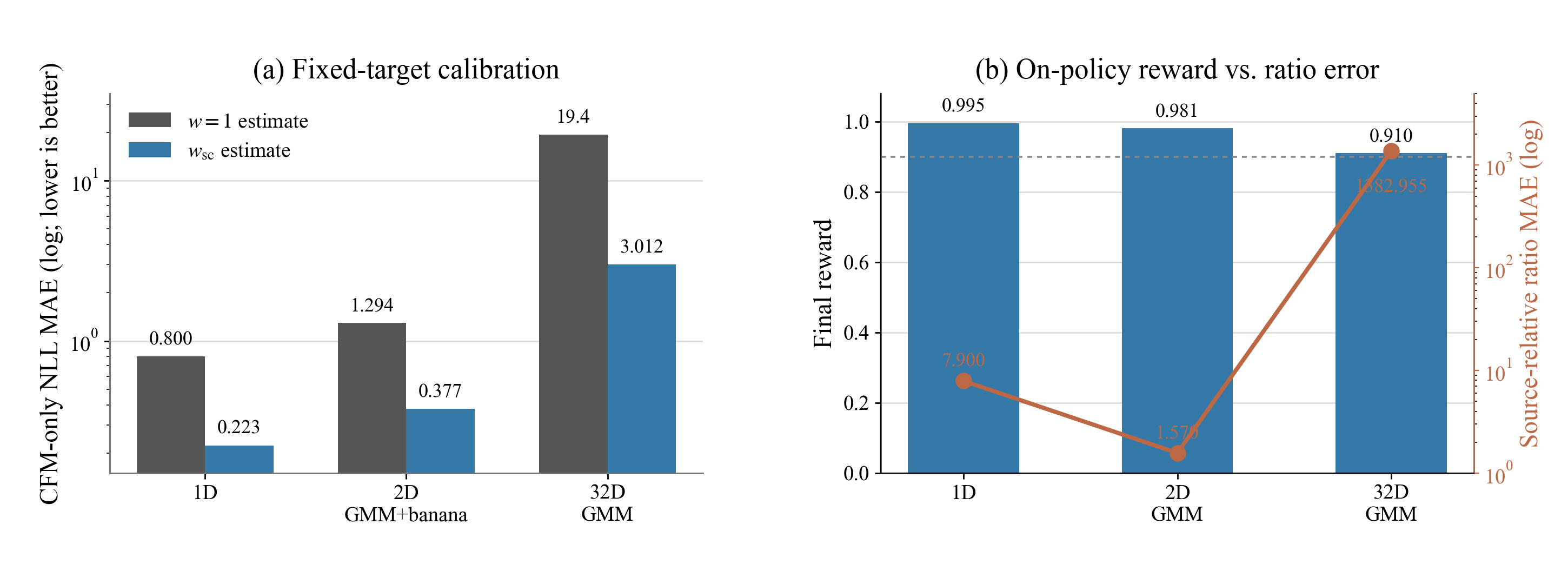}
\caption{
RQ2 cross-scale evidence for the off-policy and on-policy conclusions in \tabref{tab:cfm-nll-answer}.  Panel (a) summarizes off-policy NLL replacement: score calibration reduces the CFM-only error when boundary terms are small.  Panel (b) summarizes on-policy CFM-ratio training: high reward can coexist with large clean-ratio errors. The construction and source rows are detailed in \appref{app:rq2-four-layer-figure-construction}.
}
\vspace{-1em}
\label{fig:rq2-four-layer-table-claims}
\end{figure}

\textbf{Off-policy: fixed-target CFM.}
Panel (a) of \figref{fig:rq2-four-layer-table-claims} supports the off-policy entries of \tabref{tab:cfm-nll-answer}: ordinary CFM is not a pointwise NLL replacement, whereas the score-calibrated estimate is much closer after fixed-target optimization when the boundary term is controlled.  The displayed ordinary/score-calibrated MAEs are \(0.800/0.223\) in 1D, \(1.294/0.377\) for the 2D GMM and banana datasets, and \(19.402/3.012\) for the 32D GMM. However, on two moons, the finite-\(\varepsilon\) boundary term is large, so the CFM-only score-calibrated estimate should not be considered an NLL estimate even though the decomposition identity estimation remains numerically accurate.

\textbf{On-policy: CFM-ratio surrogate.}
Panel (b) of \figref{fig:rq2-four-layer-table-claims} supports the on-policy entries of \tabref{tab:cfm-nll-answer}: CFM-ratio training can improve reward in the tested controlled-update settings, but it is not a clean endpoint likelihood-ratio identity.  The displayed ordinary-CFM reward/source-relative-MAE pairs are \(0.995/7.900\) in 1D, \(0.981/1.570\) for the 2D GMM, and \(0.910/1382.955\) for the 32D GMM after 1,000 rounds.  Across 4D/8D/16D/32D, the ordinary-CFM rewards are \(0.997,0.995,0.991,\) and \(0.910\), whereas the score-calibrated variant degrades to \(0.209\) in 32D.  These observations associate refreshed references with smaller error and successful reward optimization.

\subsection{RQ3: Which mechanisms are associated with stable surrogate optimization?}

\begin{figure}[ht]
\centering
\includegraphics[width=0.92\textwidth]{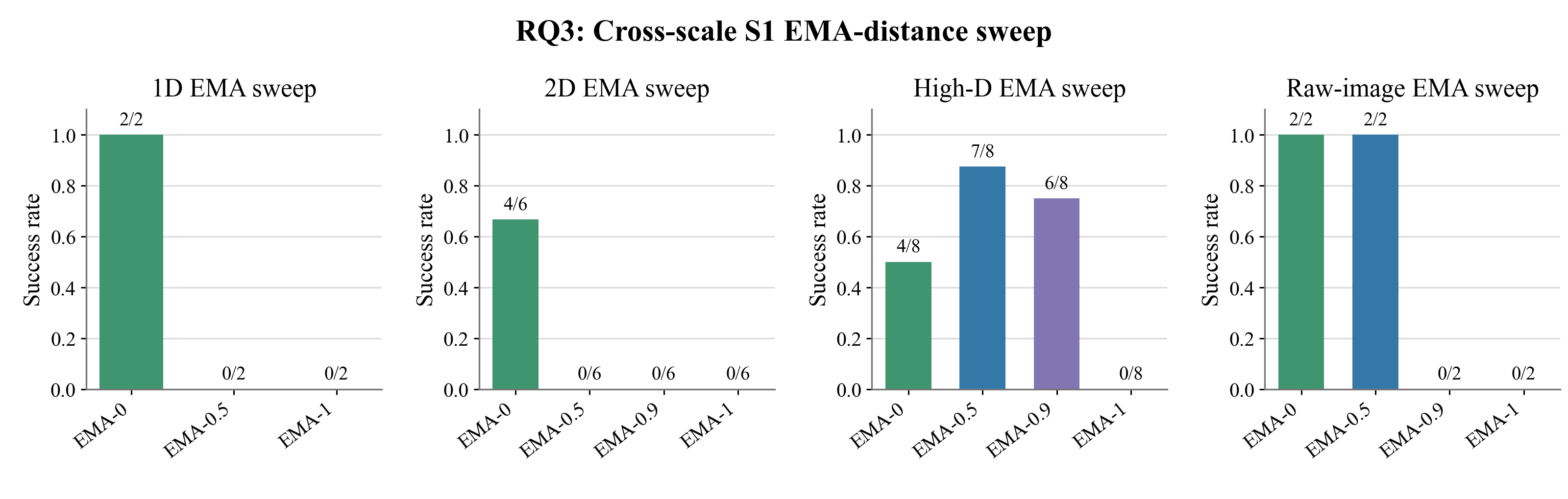}
\caption{
RQ3 cross-scale EMA-decay summary.  All four panels report the fraction of tested configurations passing the common final-reward gate \(R\geq0.90\). These are not repeated-seed success probabilities.  The raw-image panel aggregates MNIST and CIFAR-10.  EMA-0 refreshes the old-policy reference every round, EMA-0.5 and EMA-0.9 use moving-average references, and EMA-1 keeps the source checkpoint fixed. EMA decay controls the reference update rule. Exact counts and source rows are detailed in \appref{app:rq3-four-layer-figure-construction}.
}
\label{fig:rq3-four-layer-mechanisms}
\end{figure}

\figref{fig:rq3-four-layer-mechanisms} summarizes controlled associations between mechanisms and reward stability:
\begin{itemize}[leftmargin=1.35em,itemsep=0.16em,topsep=0.25em]
\item \textbf{Stability depends strongly on reference updates.} Full refresh is best in 1D/2D, whereas moderate EMA is best in High-D under fixed source-KL. EMA-1 never passes the reward gate. EMA decay itself is not a policy-distance or ratio-error measurement.
\item \textbf{Ratio clipping is helpful when the estimation is not exact.} Both clipped and unclipped settings can pass, and tighter clipping is not uniformly better. Exact sweep counts are in \appref{app:rq3-four-layer-figure-construction}.
\item \textbf{Regularization and CFM weighting are not sufficiently effective.} Vector field penalties do not rescue 2D experiments with large EMA decay rates, and score calibration does not guarantee stable high-dimensional on-policy optimization. The High-D and real image experimental results are also conditional on fixed source-KL.
\end{itemize}

\section{Conclusion}
\label{sec:conclusion}
This work characterizes when pointwise CFM quantities can replace clean endpoint likelihoods. For linear Gaussian paths, we exactly decompose endpoint NLL into entropy, a weighted CFM objective, an interior velocity--score residual, and a boundary residual. A CFM-only NLL estimate is exact when its residuals vanish, and a CFM-derived old/new log-ratio is exact when the corresponding relative residuals cancel.

This criterion separates off-policy training from on-policy alignment. At the fixed-target population optimum, ordinary CFM is not generally a pointwise NLL estimator, whereas \(w_{\mathrm{sc}}(t)=(1-t)/t\) removes the interior residual and leaves only the boundary term. During optimization, neither weighting has a general identity guarantee. In on-policy alignment, optimizing a CFM-ratio surrogate does not ensure an exact log-ratio, although the surrogate can improve reward under some mechanisms.

Synthetic experiments from 1D to 32D support the decomposition and the distinction between off-policy and on-policy regimes. Score calibration reduces off-policy CFM-only error when the boundary term is small, whereas successful reward optimization can coexist with large clean-ratio errors. Across synthetic and raw-image settings, EMA-based reference updates and ratio clipping are associated with stable optimization but do not guarantee ratio accuracy. Our analysis is limited to linear Gaussian paths with positive-time smoothing, and on-policy residual estimation remains numerically unstable in 32D. Future work should extend the identity to broader paths, explicitly control relative residuals in on-policy algorithms, and use the decomposition to adapt likelihood-based LLM post-training methods to flow matching models.

\newpage
\bibliography{references}

@article{liu2025improving,
  title={Improving video generation with human feedback},
  author={Liu, Jie and Liu, Gongye and Liang, Jiajun and Yuan, Ziyang and Liu, Xiaokun and Zheng, Mingwu and Wu, Xiele and Wang, Qiulin and Xia, Menghan and Wang, Xintao and others},
  journal={arXiv preprint arXiv:2501.13918},
  year={2025}
}

@article{song2021maximum,
  title={Maximum likelihood training of score-based diffusion models},
  author={Song, Yang and Durkan, Conor and Murray, Iain and Ermon, Stefano},
  journal={Advances in neural information processing systems},
  volume={34},
  pages={1415--1428},
  year={2021}
}

@inproceedings{lu2022maximum,
  title={Maximum Likelihood Training for Score-Based Diffusion ODEs by High Order Denoising Score Matching},
  author={Lu, Cheng and Zheng, Kaiwen and Bao, Fan and Chen, Jianfei and Li, Chongxuan and Zhu, Jun},
  booktitle={Proceedings of the 39th International Conference on Machine Learning},
  volume={162},
  pages={14429--14460},
  year={2022}
}

@inproceedings{kingma2023understanding,
  title={Understanding Diffusion Objectives as the {ELBO} with Simple Data Augmentation},
  author={Kingma, Diederik P. and Gao, Ruiqi},
  booktitle={Advances in Neural Information Processing Systems},
  volume={36},
  year={2023}
}

@inproceedings{zheng2023improved,
  title={Improved Techniques for Maximum Likelihood Estimation for Diffusion {ODE}s},
  author={Zheng, Kaiwen and Lu, Cheng and Chen, Jianfei and Zhu, Jun},
  booktitle={Proceedings of the 40th International Conference on Machine Learning},
  volume={202},
  pages={42363--42389},
  year={2023}
}

@inproceedings{benhamu2022matching,
  title={Matching Normalizing Flows and Probability Paths on Manifolds},
  author={Ben-Hamu, Heli and Cohen, Samuel and Bose, Joey and Amos, Brandon and Nickel, Maximillian and Grover, Aditya and Chen, Ricky T. Q. and Lipman, Yaron},
  booktitle={Proceedings of the 39th International Conference on Machine Learning},
  volume={162},
  pages={1749--1763},
  year={2022}
}

@inproceedings{huang2025divergence,
  title={Improving Flow Matching by Aligning Flow Divergence},
  author={Huang, Yuhao and Transue, Taos and Wang, Shih-Hsin and Feldman, William M. and Zhang, Hong and Wang, Bao},
  booktitle={Proceedings of the 42nd International Conference on Machine Learning},
  volume={267},
  pages={25813--25834},
  year={2025}
}

@inproceedings{zhou2025error,
  title={An Error Analysis of Flow Matching for Deep Generative Modeling},
  author={Zhou, Zhengyu and Liu, Weiwei},
  booktitle={Proceedings of the 42nd International Conference on Machine Learning},
  volume={267},
  pages={78903--78932},
  year={2025}
}

@article{xue2025advantage,
  title={{Advantage Weighted Matching}: Aligning {RL} with Pretraining in Diffusion Models},
  author={Xue, Shuchen and Ge, Chongjian and Zhang, Shilong and Li, Yichen and Ma, Zhi-Ming},
  journal={arXiv preprint arXiv:2509.25050},
  year={2025}
}

@article{choi2026rethinking,
  title={Rethinking the Design Space of Reinforcement Learning for Diffusion Models: On the Importance of Likelihood Estimation Beyond Loss Design},
  author={Choi, Jaemoo and Zhu, Yuchen and Guo, Wei and Molodyk, Petr and Yuan, Bo and Bai, Jinbin and Xin, Yi and Tao, Molei and Chen, Yongxin},
  journal={arXiv preprint arXiv:2602.04663},
  year={2026}
}

@article{tang2026vgrpo,
  title={{V-GRPO}: Online Reinforcement Learning for Denoising Generative Models Is Easier than You Think},
  author={Tang, Bingda and Zhang, Yuhui and Wang, Xiaohan and Mao, Jiayuan and Schmidt, Ludwig and Yeung-Levy, Serena},
  journal={arXiv preprint arXiv:2604.23380},
  year={2026}
}

@article{mcallister2025fmpg,
  title={{Flow Matching Policy Gradients}},
  author={McAllister, David and Ge, Songwei and Yi, Brent and Kim, Chung Min and Weber, Ethan and Choi, Hongsuk and Feng, Haiwen and Kanazawa, Angjoo},
  journal={arXiv preprint arXiv:2507.21053},
  year={2025}
}

@article{black2023training,
  title={Training Diffusion Models with Reinforcement Learning},
  author={Black, Kevin and Janner, Michael and Du, Yilun and Kostrikov, Ilya and Levine, Sergey},
  journal={arXiv preprint arXiv:2305.13301},
  year={2023}
}

@article{zheng2025diffusionnft,
  title={{DiffusionNFT}: Online Diffusion Reinforcement with Forward Process},
  author={Zheng, Kaiwen and Chen, Huayu and Ye, Haotian and Wang, Haoxiang and Zhang, Qinsheng and Jiang, Kai and Su, Hang and Ermon, Stefano and Zhu, Jun and Liu, Ming-Yu},
  journal={arXiv preprint arXiv:2509.16117},
  year={2025}
}

@article{liu2025flow,
  title={{Flow-GRPO}: Training Flow Matching Models via Online {RL}},
  author={Liu, Jie and Liu, Gongye and Liang, Jiajun and Li, Yangguang and Liu, Jiaheng and Wang, Xintao and Wan, Pengfei and Zhang, Di and Ouyang, Wanli},
  journal={arXiv preprint arXiv:2505.05470},
  year={2025}
}

@article{shao2024deepseekmath,
  title={{DeepSeekMath}: Pushing the Limits of Mathematical Reasoning in Open Language Models},
  author={Shao, Zhihong and Wang, Peiyi and Zhu, Qihao and Xu, Runxin and Song, Junxiao and Bi, Xiao and Zhang, Haowei and Zhang, Mingchuan and Li, YK and others},
  journal={arXiv preprint arXiv:2402.03300},
  year={2024}
}

@article{rafailov2023direct,
  title={Direct preference optimization: Your language model is secretly a reward model},
  author={Rafailov, Rafael and Sharma, Archit and Mitchell, Eric and Manning, Christopher D and Ermon, Stefano and Finn, Chelsea},
  journal={Advances in neural information processing systems},
  volume={36},
  pages={53728--53741},
  year={2023}
}

@inproceedings{wallace2024diffusion,
  title={Diffusion model alignment using direct preference optimization},
  author={Wallace, Bram and Dang, Meihua and Rafailov, Rafael and Zhou, Linqi and Lou, Aaron and Purushwalkam, Senthil and Ermon, Stefano and Xiong, Caiming and Joty, Shafiq and Naik, Nikhil},
  booktitle={Proceedings of the IEEE/CVF Conference on Computer Vision and Pattern Recognition},
  pages={8228--8238},
  year={2024}
}

@article{song2020score,
  title={Score-based generative modeling through stochastic differential equations},
  author={Song, Yang and Sohl-Dickstein, Jascha and Kingma, Diederik P and Kumar, Abhishek and Ermon, Stefano and Poole, Ben},
  journal={arXiv preprint arXiv:2011.13456},
  year={2020}
}

@article{sun2025unified,
  title={Unified Continuous Generative Models},
  author={Sun, Peng and Jiang, Yi and Lin, Tao},
  journal={arXiv preprint arXiv:2505.07447},
  year={2025}
}

@article{lipman2022flow,
  title={Flow matching for generative modeling},
  author={Lipman, Yaron and Chen, Ricky TQ and Ben-Hamu, Heli and Nickel, Maximilian and Le, Matt},
  journal={arXiv preprint arXiv:2210.02747},
  year={2022}
}

@article{albergo2022building,
  title={Building normalizing flows with stochastic interpolants},
  author={Albergo, Michael S and Vanden-Eijnden, Eric},
  journal={arXiv preprint arXiv:2209.15571},
  year={2022}
}

@article{schulman2017proximal,
  title={Proximal policy optimization algorithms},
  author={Schulman, John and Wolski, Filip and Dhariwal, Prafulla and Radford, Alec and Klimov, Oleg},
  journal={arXiv preprint arXiv:1707.06347},
  year={2017}
}

@article{liu2022flow,
  title={Flow straight and fast: Learning to generate and transfer data with rectified flow},
  author={Liu, Xingchao and Gong, Chengyue and Liu, Qiang},
  journal={arXiv preprint arXiv:2209.03003},
  year={2022}
}

@article{he2025tempflow,
  title={{TempFlow-GRPO}: When Timing Matters for {GRPO} in Flow Models},
  author={He, Xiaoxuan and Fu, Siming and Zhao, Yuke and Li, Wanli and Yang, Jian and Yin, Dacheng and Rao, Fengyun and Zhang, Bo},
  journal={arXiv preprint arXiv:2508.04324},
  year={2025}
}

@article{chen2018neural,
  title={Neural ordinary differential equations},
  author={Chen, Ricky TQ and Rubanova, Yulia and Bettencourt, Jesse and Duvenaud, David K},
  journal={Advances in neural information processing systems},
  volume={31},
  year={2018}
}

@article{li2025mixgrpo,
  title={{MixGRPO}: Unlocking Flow-Based {GRPO} Efficiency with Mixed {ODE-SDE}},
  author={Li, Junzhe and Cui, Yutao and Huang, Tao and Ma, Yinping and Fan, Chun and Yang, Miles and Zhong, Zhao},
  journal={arXiv preprint arXiv:2507.21802},
  year={2025}
}

@article{bergmeister2026reinforce,
  title={Reinforce Adjoint Matching: Scaling RL Post-Training of Diffusion and Flow-Matching Models},
  author={Bergmeister, Andreas and Jegelka, Stefanie and N{\"u}sken, Nikolas and Domingo-Enrich, Carles and Pidstrigach, Jakiw},
  journal={arXiv preprint arXiv:2605.10759},
  year={2026}
}
\bibliographystyle{plainnat}

\appendix
\newpage

\begingroup
\setlength{\parskip}{0pt}
\hypersetup{linkcolor=black}
\tableofcontents
\endgroup

\newpage

\section{Proofs}
\label{app:proofs}

\subsection{Proof of Theorem~\ref{thm:pw-practical-decomposition}}
\label{sec:proof-of-theorem-pw-practical-decomposition}

\begin{lemma}[KL derivative along the conditional Gaussian path]
\label{lem:pw-kl-derivative}
Under Assumption~\ref{ass:pw-pathwise-regularity},
\[
\frac{d}{dt}\mathrm{KL}(q_t^{x_0}\|p_t^\theta)
=
\E_{q_t^{x_0}}
\bigl[
\langle u_t^{x_0}(X_t)-v_\theta(X_t,t),
\nabla\log q_t^{x_0}(X_t) - \nabla\log p_t^\theta(X_t)\rangle
\bigr].
\]
\end{lemma}

\begin{proof}
\begin{align*}
    &\frac{d}{dt}\mathrm{KL}(q_t^{x_0}\|p_t^\theta) \\
    =& \frac{d}{dt} \int q_t^{x_0}\log\frac{q_t^{x_0}}{p_t^\theta}\,dx \\
    =& \int (\partial_t q_t^{x_0})\log\frac{q_t^{x_0}}{p_t^\theta}\,dx + \int q_t^{x_0}\partial_t\log\frac{q_t^{x_0}}{p_t^\theta}\,dx \\
    =& \int (\partial_t q_t^{x_0})\log\frac{q_t^{x_0}}{p_t^\theta}\,dx
       - \int q_t^{x_0}\partial_t\log p_t^\theta\,dx,
\end{align*}
where we used
\(\int q_t^{x_0}\partial_t\log q_t^{x_0}\,dx
=\int\partial_tq_t^{x_0}\,dx=0\).

For the first term, the continuity equation
\(\partial_tq_t^{x_0}=-\nabla\cdot(q_t^{x_0}u_t^{x_0})\)
and integration by parts give
\[
\int(\partial_tq_t^{x_0})
\log\frac{q_t^{x_0}}{p_t^\theta}\,dx
=
\int q_t^{x_0}
\left\langle
u_t^{x_0},
\nabla\log\frac{q_t^{x_0}}{p_t^\theta}
\right\rangle dx .
\]
For the second term, the model continuity equation implies
\[
\partial_t\log p_t^\theta
=-\nabla\cdot v_\theta
-\left\langle v_\theta,\nabla\log p_t^\theta\right\rangle.
\]
Therefore, another integration by parts yields
\[
-\int q_t^{x_0}\partial_t\log p_t^\theta\,dx
=
\int q_t^{x_0}\nabla\cdot v_\theta\,dx
+\int q_t^{x_0}\left\langle v_\theta,\nabla\log p_t^\theta\right\rangle dx
=
-\int q_t^{x_0}\left\langle
v_\theta,\nabla\log\frac{q_t^{x_0}}{p_t^\theta}
\right\rangle dx.
\]
Adding the two terms proves the stated derivative identity.
\end{proof}

\begin{proof}[Proof of Theorem~\ref{thm:pw-practical-decomposition}]
The cross-entropy decomposition gives
\[
\ell_\varepsilon(\theta;x_0)
=
H(q_\varepsilon^{x_0})
+
\mathrm{KL}(q_\varepsilon^{x_0}\|p_\varepsilon^\theta).
\]
Integrating Lemma~\ref{lem:pw-kl-derivative} from $\varepsilon$ to $1$ gives
\[
\mathrm{KL}(q_1^{x_0}\|p_1^\theta)
-
\mathrm{KL}(q_\varepsilon^{x_0}\|p_\varepsilon^\theta)
=
\int_\varepsilon^1
\E_{q_t^{x_0}}
\bigl[
\langle u_t^{x_0}-v_\theta, \Delta s_t^\theta(X_t;x_0)\rangle
\bigr]\,dt .
\]
Because $q_1^{x_0}=p_1^\theta=\cN(0,I_d)$, the first KL term is zero.
Rearranging and using $\Delta v_t^\theta(X_t;x_0)=v_\theta-u_t^{x_0}$ proves
\eqref{eq:pw-exact-identity}.  The entropy formula follows from
$q_\varepsilon^{x_0}=\cN((1-\varepsilon)x_0,\varepsilon^2I_d)$.

Subtracting the pointwise CFM objective term from the mixed integrand gives the
definition of $\mathcal G_{\varepsilon,w}$ in
\eqref{eq:pw-practical-decomposition}.
By definition,
\[
-\log p_0^\theta(x_0)
=
\ell_\varepsilon(\theta;x_0)
+
\mathcal B_\varepsilon(\theta;x_0),
\]
which yields \eqref{eq:pw-practical-decomposition}.  It remains to verify the
boundary limit.  Let
$X_\varepsilon=(1-\varepsilon)x_0+\varepsilon Z$ with
$Z\sim\cN(0,I_d)$ and $f_\theta(x,t)=-\log p_t^\theta(x)$.  Then
$X_\varepsilon\to x_0$ almost surely and
$f_\theta(X_\varepsilon,\varepsilon)\to f_\theta(x_0,0)$.
\assumpref{ass:pw-endpoint-regularity} gives an integrable polynomial-growth
dominating function, so dominated convergence theorem yields
\[
\ell_\varepsilon(\theta;x_0)
=
\E[f_\theta(X_\varepsilon,\varepsilon)]
\to
f_\theta(x_0,0)
=
-\log p_0^\theta(x_0).
\]
Thus $\mathcal B_\varepsilon(\theta;x_0)\to0$.
\end{proof}

\subsection{Proof of Proposition~\ref{prop:pw-linear-training-counterexample}}
\label{sec:proof-of-proposition-pw-linear-training-counterexample}

\begin{proof}
For the simple neural network $v_\theta(x,t)=a_\theta(x,t)x$, the velocity gap is
\begin{align}
\Delta v_t^\theta(x;x_0)
&= v_\theta(x,t) - u_t^{x_0}(x) \\
&= a_\theta(x,t)x - \frac{x-x_0}{t}\\
\Delta s_t^\theta(x;x_0) &= \nabla\log q_t^{x_0}(x) - \nabla\log p_t^\theta(x)
\end{align}
Now, we consider the simple case where $a_\theta = 0$. Because the velocity is zero, the distribution is unchanged, i.e.
$p_t^\theta=\cN(0,1)$ for all $t$.  At $x_0=0$,
\begin{align}
q_t^0=\cN(0,t^2),
\qquad
\nabla\log q_t^0(x)=-\frac{x}{t^2},
\qquad
\nabla\log p_t^\theta(x)=-x.
\end{align}
Therefore,
\begin{align}
\Delta v_t^\theta(x;0)=-\frac{x}{t},
\qquad
\Delta s_t^\theta(x;0)
=
-\frac{x}{t^2}+x.
\end{align}
Therefore,
\begin{align}
\Delta s_t^\theta(x;0)-w_{\mathrm{sc}}(t)\Delta v_t^\theta(x;0)
=
-\frac{x}{t^2}+x+\frac{1-t}{t^2}x
=
\left(1-\frac1t\right)x,
\end{align}
which is nonzero when $t\in(0,1)$. Since
$\E_{q_t^0}[X_t^2] = \mathrm{Var}_{q_t^0}(X_t) = t^2$,
\begin{align}
\mathcal G_{\varepsilon,w_{\mathrm{sc}}}(\theta;0)
&= \int_\varepsilon^1\E_{q_t^{0}}\bigl[\langle \Delta v_t^\theta(X_t;0), \Delta s_t^\theta(X_t;0)-w(t)\Delta v_t^\theta(X_t;0)\rangle\bigr]\,dt \\
&=\int_\varepsilon^1
\E_{q_t^0}\left[
-\frac{X_t}{t}\left(1-\frac1t\right)X_t
\right]dt \\
&=\int_\varepsilon^1(1-t)\,dt \\
&=\frac{(1-\varepsilon)^2}{2}>0.
\end{align}

\end{proof}

\subsection{Proof of Proposition~\ref{prop:pw-linear-optimum-calibration}}
\label{sec:proof-of-proposition-pw-linear-optimum-calibration}

\begin{proof}[Proof of \propref{prop:pw-linear-optimum-calibration}]
Write \(U:=X_1-X_0\). For each \(t\in(0,1)\), the conditional decomposition gives
\begin{align}
&\E\!\left[\|v(X_t,t)-U\|^2\right] \\
&=
\E\!\left[\|v(X_t,t)-\E[U\mid X_t]\|^2\right]
+
\E\!\left[\|\E[U\mid X_t] - U\|^2\right] \\
& \quad+2\E\!\left[\langle v(X_t,t)-\E[U\mid X_t], \E[U\mid X_t] - U\rangle\right]
\label{eq:app-pop-opt-regression-decomposition}
\end{align}
In the following, we prove $2\E\!\left[\langle v(X_t,t)-\E[U\mid X_t], \E[U\mid X_t] - U\rangle\right] = 0$.
\begin{align}
&\E\!\left[\langle v(X_t,t)-\E[U\mid X_t], \E[U\mid X_t] - U\rangle\right] \\
&=
\E\!\left[ \E [\langle v(X_t,t)-\E[U\mid X_t], \E[U\mid X_t] - U\rangle \mid X_t]\right] \\
&=
\E\!\left[ \langle v(X_t,t)-\E[U\mid X_t], \E[U\mid X_t] - \E[U\mid X_t]\rangle \right] \quad \text{(due to the linearity of $\E$)} \\
&= 0
\end{align}

Hence every positive time weight has the same pointwise population CFM
minimizer
\[
v_{\theta^\star}(x,t) = \E[U\mid X_t=x] = E[X_1-X_0\mid X_t=x]
\]

Let
\(m_t(x):=\E[X_0\mid X_t=x]\), $X_t = (1-t)X_0 + t X_1 \sim p_t^{\theta^\star}(x)$, $X_1 \sim \mathcal N(0, I)$ and $X_0 \sim p_0(x)$. 
\begin{align}
v_{\theta^\star}(x,t) = \E[X_1-X_0\mid X_t=x] = \E[\frac{X_t-(1-t)X_0}{t} - X_0\mid X_t=x] = \frac{x-m_t(x)}{t}
\end{align}
Because the density of the sum of random variables is the convolution of the densities of the random variables, we have
\begin{align}
p_t^{\theta^\star}(x)
&=\varphi_{t^2I_d} * p_0(x) = 
\int \varphi_{t^2I_d}(x-(1-t)y)p_0(y)\,\mathrm{d}y.
\label{eq:app-pop-opt-convolution}
\end{align}
where $\varphi_{t^2I_d}(x)$ is the density of the Gaussian distribution $\mathcal N(0,t^2I_d)$.

For \(t>0\), differentiating \eqref{eq:app-pop-opt-convolution} gives
\begin{align}
\nabla\log p_t^{\theta^\star}(x)
&=
\E\!\left[
-\frac{x-(1-t)X_0}{t^2}
\,\middle|\,X_t=x
\right]
=
\frac{(1-t)m_t(x)-x}{t^2}.
\label{eq:app-pop-opt-general-score}
\end{align}
For the conditional path fixed at \(x_0\),
\begin{align}
u_t^{x_0}(x)
&=\frac{x-x_0}{t},
&
\nabla\log q_t^{x_0}(x)
&=\frac{(1-t)x_0-x}{t^2}.
\label{eq:app-pop-opt-general-cond}
\end{align}
Therefore, we have
\begin{align}
\Delta v_t^{\theta^\star}(x;x_0)
&=v_{\theta^\star}(x,t)-u_t^{x_0}(x)
=\frac{x_0-m_t(x)}{t},\\
\Delta s_t^{\theta^\star}(x;x_0)
&=\nabla\log q_t^{x_0}(x)-\nabla\log p_t^{\theta^\star}(x)
=\frac{1-t}{t^2}\bigl(x_0-m_t(x)\bigr).
\end{align}
Therefore,
\[
\Delta s_t^{\theta^\star}(x;x_0)
=\frac{1-t}{t}\Delta v_t^{\theta^\star}(x;x_0)
=w_{\mathrm{sc}}(t)\Delta v_t^{\theta^\star}(x;x_0).
\]
Consequently,
\(\mathcal G_{\varepsilon,w_{\mathrm{sc}}}(\theta^\star;x_0)=0\).
For every fixed \(\varepsilon>0\), the Gaussian convolutions above are smooth
and positive on \([\varepsilon,1]\), so the identity in
\thmref{thm:pw-practical-decomposition} applies and gives
\[
\ell_\varepsilon(\theta^\star;x_0)
=
H(q_\varepsilon^{x_0})
+
\mathcal J_{w_{\mathrm{sc}}}^{[\varepsilon,1]}(\theta^\star;x_0).
\]
\end{proof}

\subsection{Proof of Proposition~\ref{prop:onpolicy-cfm-ratio-bias}}
\label{sec:proof-of-proposition-onpolicy-cfm-ratio-bias}

\begin{proof}
Fix $c$ and apply \thmref{thm:pw-practical-decomposition} to $\theta$ and $\theta_{\mathrm{old}}$:
\begin{align}
-\log p_0^\theta(x_0\mid c) &= H(q_\varepsilon^{x_0}) + \mathcal J_w^{[\varepsilon,1]}(\theta;x_0,c) + \mathcal G_{\varepsilon,w}(\theta;x_0,c) + \mathcal B_\varepsilon(\theta;x_0,c) \\
-\log p_0^{\theta_{\mathrm{old}}}(x_0\mid c) &= H(q_\varepsilon^{x_0}) + \mathcal J_w^{[\varepsilon,1]}(\theta_{\mathrm{old}};x_0,c) + \mathcal G_{\varepsilon,w}(\theta_{\mathrm{old}};x_0,c) + \mathcal B_\varepsilon(\theta_{\mathrm{old}};x_0,c)
\end{align}
Then, we have
\begin{align}
\log
\frac{p_0^\theta(x_0\mid c)}{p_0^{\theta_{\mathrm{old}}}(x_0\mid c)}
&=
\mathcal J_w^{[\varepsilon,1]}(\theta_{\mathrm{old}};x_0,c)
-
\mathcal J_w^{[\varepsilon,1]}(\theta;x_0,c)
-
\Delta\mathcal G_{\varepsilon,w}(x_0,c)
-
\Delta\mathcal B_\varepsilon(x_0,c).
\end{align}
The first two terms are
$\widehat\Delta_{\theta,\theta_{\mathrm{old}}}^{\mathrm{CFM}}(x_0,c)$,
which proves \eqref{eq:onpolicy-clean-cfm-ratio-decomp}.  Rearranging and
exponentiating gives \eqref{eq:onpolicy-multiplicative-bias}.
\end{proof}

\subsection{Proof of Proposition~\ref{prop:onpolicy-rotational-mismatch}}
\label{sec:proof-of-proposition-onpolicy-rotational-mismatch}

\begin{proof}
The marginal path is $p_t=\mathcal N(0,s_t^2I_2)$.
$R_0^\top=-R_0$ implies $x^\top R_0x = 0$ for all $x$, $\operatorname{tr}(R_0)=0$ and hence
\[
\nabla\cdot(p_tR_0x)
=
\nabla p_t(x)\cdot R_0x
+
p_t(x)\operatorname{tr}(R_0)
=
-\frac{p_t(x)}{s_t^2}x^\top R_0x
+
p_t(x)\operatorname{tr}(R_0)
=
0 .
\]
Because $\partial_t p^{\theta_{\mathrm{old}}}(x) + \nabla \cdot (p_t^{\theta_{\mathrm{old}}}(x) v_{\mathrm{old}}^\star(x,t)) = 0$, we have
\begin{align}
    &\partial_t p_t^{\theta}(x) + \nabla \cdot (p_t^{\theta}(x) v_a(x,t)) \\
    =& \partial_t p^{\theta_{\mathrm{old}}}(x) + \nabla \cdot (p_t^{\theta_{\mathrm{old}}}(x) v_{\mathrm{old}}^\star(x,t)) + \nabla \cdot (p_t^{\theta}(x) aR_0x) \\
    =& 0
\end{align}
where $p_0^{a}$ is the endpoint density of $v_a$. Therefore, $p_0^{a}=p_0^{\theta_{\mathrm{old}}}$ and
$\log \frac{p_0^a(x_0)}{p_0^{\mathrm{old}}(x_0)}=0$ for all $x_0$.

It remains to compute the CFM objective difference at $x_0=0$.  

For $q_t^0=\mathcal N(0,t^2I_2)$, we have $u_t^0(x)=x/t$.  Then
\begin{align}
\Delta v_t^{\theta_{\mathrm{old}}}(x;0)
&= v_{\mathrm{old}}^\star(x,t)-u_t^0(x) = \frac{2t-1}{s_t^2}x - \frac{x}{t} := c(t) x
\end{align}

Therefore,
\begin{align}
\mathcal J_{w_{\mathrm{sc}}}^{[\varepsilon,1]}(a;0)
&=
\int_\varepsilon^1
w_{\mathrm{sc}}(t)
\E_{q_t^0}\|\Delta v_t^{a}(X_t;0)\|^2_2\,dt \\
&=
\int_\varepsilon^1
w_{\mathrm{sc}}(t)
\E_{q_t^0}\|v_a(X_t,t) - u_t^0(X_t)\|^2_2\,dt \\
&=
\int_\varepsilon^1
w_{\mathrm{sc}}(t)
\E_{q_t^0}\|v_{\mathrm{old}}^\star(X_t,t) + aR_0X_t - u_t^0(X_t)\|^2_2\,dt \\
&= 
\int_\varepsilon^1
w_{\mathrm{sc}}(t)
\E_{q_t^0}\|c(t) X_t + aR_0X_t\|^2_2\,dt \\
&= 
\int_\varepsilon^1
w_{\mathrm{sc}}(t)
\E_{q_t^0} \left[\|c(t) X_t\|^2_2 + 2\langle c(t) X_t, aR_0X_t \rangle + a^2\|R_0X_t\|^2_2\right]\,dt \\
&= 
\mathcal J_{w_{\mathrm{sc}}}^{[\varepsilon,1]}({\mathrm{old}};0) + a^2\int_\varepsilon^1 w_{\mathrm{sc}}(t) \E_{q_t^0} \left[\|R_0X_t\|^2_2\right]\,dt \quad \text{(due to $x^\top R_0x = 0$)}\\
&= 
\mathcal J_{w_{\mathrm{sc}}}^{[\varepsilon,1]}({\mathrm{old}};0) + a^2\int_\varepsilon^1 w_{\mathrm{sc}}(t) \E_{q_t^0} \left[\|X_t\|^2_2\right]\,dt \quad \text{(due to $\|R_0x\|^2_2 = \|x\|^2_2$)}\\
&= 
\mathcal J_{w_{\mathrm{sc}}}^{[\varepsilon,1]}({\mathrm{old}};0) + a^2\int_\varepsilon^1 \frac{1 - t}{t} t^2 \,dt \quad \text{(due to $q_t^0=\mathcal N(0,t^2I_2)$)}\\
&=
\mathcal J_{w_{\mathrm{sc}}}^{[\varepsilon,1]}({\mathrm{old}};0) + 2a^2\int_\varepsilon^1 t(1-t)\,dt \\
&=
\mathcal J_{w_{\mathrm{sc}}}^{[\varepsilon,1]}({\mathrm{old}};0) + 
a^2
\left(
\frac13-\varepsilon^2+\frac23\varepsilon^3
\right).
\end{align}
Thus
\[
\widehat\Delta_{a,\mathrm{old}}^{\mathrm{CFM}}(0)
=
\mathcal J_{w_{\mathrm{sc}}}^{[\varepsilon,1]}({\mathrm{old}};0)
-
\mathcal J_{w_{\mathrm{sc}}}^{[\varepsilon,1]}(a;0)
=
-C_\varepsilon a^2,
\]
where
$C_\varepsilon=\frac13-\varepsilon^2+\frac23\varepsilon^3>0$ for
$\varepsilon\in(0,1)$.
\end{proof}

\subsection{Proof of Proposition~\ref{prop:onpolicy-after-cfm-ratio-gap-nonzero}}
\label{sec:proof-of-proposition-onpolicy-after-cfm-ratio-gap-nonzero}

\begin{proof}
By \propref{prop:onpolicy-rotational-mismatch},
$\widehat\Delta_{a,\mathrm{old}}^{\mathrm{CFM}}(0)=-C_\varepsilon a^2$ and
$\log \frac{p_0^a(0)}{p_0^{\theta_{\mathrm{old}}}(0)}=0$.  Let $y=a^2$.  Then
\[
\mathcal L(y)
=
-r e^{-C_\varepsilon y}
-
\lambda y,
\qquad y\ge0.
\]
Its derivative is
\[
\frac{d\mathcal L}{dy}
=
rC_\varepsilon e^{-C_\varepsilon y}
-
\lambda .
\]
If $0<\lambda<rC_\varepsilon$, the unique maximizer is
\[
y^\star
=
\frac{1}{C_\varepsilon}
\log\frac{rC_\varepsilon}{\lambda}
>0 .
\]
Thus any optimizer $a^\star$ with $(a^\star)^2=y^\star$ is nonzero.  At this
optimizer the log-ratio remains zero, while
$\widehat\Delta_{a^\star,\mathrm{old}}^{\mathrm{CFM}}(0)
=-C_\varepsilon(a^\star)^2\ne0$.
\end{proof}

\clearpage
\section{Experiment Details}
\label{app:experiment-details}

\subsection{Construction of the RQ1 Cross-Scale Summary Figure}
\label{app:rq1-four-layer-figure-construction}

\figref{fig:rq1-four-layer-closure} is a visualization aggregation of
the fixed-target audit tables in this appendix. For each displayed setting \(s\), the gray interval is
\[
\left[\min_{r\in\mathcal R_s} e^{\mathrm{CFM}}_r,\,
      \max_{r\in\mathcal R_s} e^{\mathrm{CFM}}_r\right],
\qquad
e^{\mathrm{CFM}}_r
=\mathrm{MAE}\!\left(\mathrm{NLL},\,\mathcal J_w+H\right),
\]
and the green interval is the smoothed-identity closure error
\[
\left[\min_{r\in\mathcal R_s} e^{\mathrm{close}}_r,\,
      \max_{r\in\mathcal R_s} e^{\mathrm{close}}_r\right],
\qquad
\begin{aligned}
e^{\mathrm{close}}_r
&=\mathrm{MAE}\!\left(\ell_\varepsilon,\,\mathcal J_w+\mathcal G+H\right)\\
&=\mathrm{MAE}\!\left(\mathrm{NLL},\,\mathcal J_w+\mathcal G+\mathcal B+H\right),
\end{aligned}
\]
where the second equality holds sample by sample because
\(\mathcal B_\varepsilon=\mathrm{NLL}-\ell_\varepsilon\), and
\(\mathcal R_s\) denotes the reported configurations for setting \(s\).
Thus the green interval audits \eqref{eq:pw-exact-identity} using independently
estimated \(\ell_\varepsilon\), \(\mathcal J_w\), and \(\mathcal G\); it is
not evidence for the definition or endpoint limit of
\(\mathcal B_\varepsilon\). The circular marker in each interval is the
geometric midpoint \(\sqrt{\mathrm{low}\cdot\mathrm{high}}\). The numeric
label printed next to each interval is its upper endpoint. The exact endpoints
are listed in \tabref{tab:app-rq1-four-layer-range-construction}.

\begin{table}[ht]
\centering
\caption{
Range construction for the RQ1 cross-scale summary in \figref{fig:rq1-four-layer-closure}.  Each range is the minimum and maximum over the specified fixed-target audit entries; lower is better.
}
\label{tab:app-rq1-four-layer-range-construction}
\scriptsize
\setlength{\tabcolsep}{3.6pt}
\renewcommand{\arraystretch}{1.12}
\begin{tabularx}{\textwidth}{lccX}
\toprule
Setting
& CFM-only range
& Smoothed-identity closure
& Source entries \\
\midrule
1D GMM
& \(0.222\)--\(0.803\)
& \(0.198\)--\(0.234\)
& \tabref{tab:app-1d-fixed-high-audit}, all train-weight and estimate-weight combinations. \\
2D GMM
& \(0.411\)--\(1.001\)
& \(0.387\)--\(0.397\)
& \tabref{tab:app-2d-fixed-high-audit}, GMM row, both estimate weights. \\
2D banana
& \(0.342\)--\(1.587\)
& \(0.312\)--\(0.323\)
& \tabref{tab:app-2d-fixed-high-audit}, banana row, both estimate weights. \\
2D two moons
& \(0.580\)--\(0.895\)
& \(0.302\)--\(0.311\)
& \tabref{tab:app-2d-fixed-high-audit}, two-moons row, both estimate weights. \\
GMM-4D
& \(1.081\)--\(2.001\)
& \(1.021\)--\(1.028\)
& \tabref{tab:app-highd-fixed-audit}, GMM-4D row, both estimate weights. \\
GMM-8D
& \(1.387\)--\(3.447\)
& \(1.238\)--\(1.392\)
& \tabref{tab:app-highd-fixed-audit}, GMM-8D row, both estimate weights. \\
GMM-16D
& \(2.000\)--\(8.108\)
& \(1.711\)--\(1.880\)
& \tabref{tab:app-highd-fixed-audit}, GMM-16D row, both estimate weights. \\
GMM-32D
& \(3.012\)--\(19.402\)
& \(2.290\)--\(2.617\)
& \tabref{tab:app-highd-fixed-audit}, GMM-32D row, both estimate weights. \\
\bottomrule
\end{tabularx}
\end{table}

\subsection{Construction of the RQ2 Cross-Scale Summary Figure}
\label{app:rq2-four-layer-figure-construction}

\figref{fig:rq2-four-layer-table-claims} is a visualization aggregation of the RQ2 experiment results in this appendix. Unlike \figref{fig:rq1-four-layer-closure}, which plots ranges, RQ2 plots representative values for the two claim types in \tabref{tab:cfm-nll-answer}: off-policy fixed-target NLL replacement in panel (a), and on-policy CFM-ratio surrogate behavior in panel (b). For panel (a), each displayed pair is a CFM-only NLL MAE comparison
\[
e^{(1)}_s
=\mathrm{MAE}\!\left(\mathrm{NLL},\,\mathcal J_{w=1}+H\right),
\qquad
e^{(\mathrm{sc})}_s
=\mathrm{MAE}\!\left(\mathrm{NLL},\,\mathcal J_{w_{\mathrm{sc}}}+H\right),
\]
using the source rows listed in \tabref{tab:app-rq2-four-layer-construction}. The 1D values average over the two train weights for a fixed estimate weight.  The 2D value averages the GMM and banana rows because their boundary terms $B$ are small, and two moons is excluded from the plotted average due to its large boundary term $B$. The high-dimensional value uses the 32D GMM row.  Raw images are excluded because the corresponding experiments do not provide exact raw-pixel NLL or decomposition audits.

For panel (b), each displayed point is a pair \((R_s,\rho_s)\), where \(R_s\) is the final reward and \(\rho_s\) is the MAE between the clean ratio and the CFM-only ratio with $w=1$.  The reward bars use the left axis, the ratio-MAE curve uses the log-scaled right axis, and the dashed horizontal line marks the \(0.90\) reward gate.  All three ratio errors are plotted at their reported values without clipping or off-scale substitution.  The 32D point comes from the 1,000-round ordinary-CFM run in \tabref{tab:app-highd-onpolicy-baseline}; its source-relative MAE is \(1382.955\), while the corresponding round-old MAE is \(2.617\).  These entries show the table's on-policy message directly: reward success can hold while source-relative clean-ratio error remains large.  Raw images are omitted because no exact raw-pixel clean-ratio audit is available.

\begin{table}[ht]
\centering
\caption{
Construction for the RQ2 cross-scale summary in \figref{fig:rq2-four-layer-table-claims}.  Panel (a) reports CFM-only NLL MAE, where lower is better.  Panel (b) reports final reward and clean-ratio MAE for the same on-policy run.
}
\label{tab:app-rq2-four-layer-construction}
\scriptsize
\setlength{\tabcolsep}{3.2pt}
\renewcommand{\arraystretch}{1.12}
\begin{tabularx}{\textwidth}{
>{\raggedright\arraybackslash}p{0.18\textwidth}
>{\raggedright\arraybackslash}p{0.15\textwidth}
>{\raggedright\arraybackslash}p{0.27\textwidth}
>{\raggedright\arraybackslash}X
}
\toprule
Data Types
& Quantity Types
& Quantity Values
& Sources \\
\midrule
(a) 1D
& \(e^{(1)}_s\) / \(e^{(\mathrm{sc})}_s\)
& \(0.800/0.223\)
& Means over the \(w=1\) estimate rows and over the \(w=w_{\mathrm{sc}}\) estimate rows in \tabref{tab:app-1d-fixed-high-audit}. \\
(a) 2D GMM+banana
& \(e^{(1)}_s\) / \(e^{(\mathrm{sc})}_s\)
& \(1.294/0.377\)
& Means of the GMM and banana rows in \tabref{tab:app-2d-fixed-high-audit}; two moons is excluded because its boundary term is the intended caveat. \\
(a) 32D GMM
& \(e^{(1)}_s\) / \(e^{(\mathrm{sc})}_s\)
& \(19.402/3.012\)
& GMM-32D row in \tabref{tab:app-highd-fixed-audit}. \\
\midrule
(b) 1D
& \(R_s\) / \(\rho_s\)
& \(0.995/7.900\)
& CFM-only ratio row in \tabref{tab:app-1d-onpolicy-baseline}. \\
(b) 2D GMM
& \(R_s\) / \(\rho_s\)
& \(0.981/1.570\)
& GMM ordinary-CFM row in \tabref{tab:app-2d-onpolicy-baseline}. \\
(b) 32D GMM
& \(R_s\) / \(\rho_s\)
& \(0.910/1382.955\)
& GMM-32D ordinary-CFM row in \tabref{tab:app-highd-onpolicy-baseline}; \(\rho_s\) is source-relative, and the round-old MAE is \(2.617\). \\
\bottomrule
\end{tabularx}
\end{table}

\subsection{Construction of the RQ3 Cross-Scale Summary Figure}
\label{app:rq3-four-layer-figure-construction}

\figref{fig:rq3-four-layer-mechanisms} aggregates the controlled sweeps used
to answer RQ3. The figure reports associations between the reference-update
rule and reward stability, not a causal mechanism or evidence that the
surrogate becomes an exact clean likelihood ratio. EMA-0 refreshes the
old-policy reference every round, EMA-0.5 and EMA-0.9 use moving-average
references, and EMA-1 keeps the source checkpoint fixed. EMA decay is not a
measured distance between policies.

All four panels visualize the S1 EMA-decay sweep. A configuration passes when
its final sampled reward reaches the common reward gate \(0.90\). The plotted
bar height is \(s/N\), where \(s\) is the number of passing configurations and
\(N\) is the number tested at that EMA value. It is not a repeated-seed success
probability. The text above each bar gives the exact count \(s/N\). The
raw-image panel aggregates MNIST and CIFAR-10. The High-D and raw-image S1
sweeps use the fixed source-vector-field coefficient
\(\beta_{\mathrm{src}}=0.1\), so their EMA comparison is conditional on that
regularization setting.

\begin{table}[ht]
\centering
\caption{
Construction for the RQ3 cross-scale summary in \figref{fig:rq3-four-layer-mechanisms}. Every panel reports the fraction of tested S1 configurations passing the common final sampled reward gate \(R\geq0.90\); higher is better.
}
\label{tab:app-rq3-four-layer-construction}
\scriptsize
\setlength{\tabcolsep}{3.0pt}
\renewcommand{\arraystretch}{1.12}
\begin{tabularx}{\textwidth}{
>{\RaggedRight\arraybackslash}p{0.23\textwidth}
>{\RaggedRight\arraybackslash}p{0.29\textwidth}
>{\RaggedRight\arraybackslash}X
}
\toprule
Panel / layer
& Plotted value
& Source entries \\
\midrule
1D EMA sweep
& Success rates: EMA-0 (\(2/2\)), EMA-0.5 (\(0/2\)), EMA-1 (\(0/2\))
& EMA-decay row in \tabref{tab:app-1d-mechanism-sweeps}; passing means final sampled reward at least \(0.90\). \\
2D EMA sweep
& Success rates: EMA-0 (\(4/6\)), EMA-0.5 (\(0/6\)), EMA-0.9 (\(0/6\)), EMA-1 (\(0/6\))
& EMA-decay row in \tabref{tab:app-2d-mechanism-sweeps}, aggregated over GMM, banana, and two moons. \\
High-D EMA sweep
& Success rates: EMA-0 (\(4/8\)), EMA-0.5 (\(7/8\)), EMA-0.9 (\(6/8\)), EMA-1 (\(0/8\))
& S1 row in \tabref{tab:app-highd-mechanism-sweep}, aggregated over GMM-4D, 8D, 16D, and 32D. \\
Raw-image EMA sweep
& Success rates: EMA-0 (\(2/2\)), EMA-0.5 (\(2/2\)), EMA-0.9 (\(0/2\)), EMA-1 (\(0/2\))
& S1 row in \tabref{tab:app-image-raw-mechanism-summary}, aggregated over MNIST and CIFAR-10. \\
\bottomrule
\end{tabularx}
\end{table}

\subsection{1D Toy Experiments}
\label{app:1d-toy-experiments}

The 1D GMM experiments test the three research questions behind \tabref{tab:cfm-nll-answer}:
\begin{enumerate}
    \item[(RQ1)] whether the pointwise decomposition in \thmref{thm:pw-practical-decomposition} is numerically closed;
    \item[(RQ2)] whether the off-policy and on-policy conclusions in \tabref{tab:cfm-nll-answer} hold;
    \item[(RQ3)] why CFM-only ratios can be useful in on-policy training.
\end{enumerate}

\subsubsection{Experiment Setup}
The 1D target distribution is the three-component Gaussian mixture
\[
p_{\mathrm{data}}(x)
=0.35\,\mathcal N(-2.2,0.45^2)
+0.40\,\mathcal N(0,0.75^2)
+0.25\,\mathcal N(2.0,0.35^2).
\]
All CFM models use the linear Gaussian path
\(X_t=(1-t)x_0+t z\), \(z\sim\mathcal N(0,1)\), with
\(t\in[\varepsilon,1]\) and \(\varepsilon=0.10\).  The ordinary estimate uses
\(w(t)=1\), while the score-calibrated estimate uses
\(w_{\mathrm{sc}}(t)=(1-t)/t\).  The reward used by the on-policy experiments
is a bounded Gaussian bump centered at \(2.0\),
\[
R(x)=\exp\!\left(-\frac{1}{2}\frac{(x-2.0)^2}{0.45^2}\right).
\]

\tabref{tab:app-1d-rq12-fixed-setup} and \tabref{tab:app-1d-rq12-onpolicy-setup} collect the off-policy and on-policy configurations used for RQ1 and RQ2.

\begin{table}[H]
\centering
\caption{\textbf{1D off-policy settings for RQ1 and RQ2.}}
\label{tab:app-1d-rq12-fixed-setup}
\scriptsize
\setlength{\tabcolsep}{3.4pt}
\renewcommand{\arraystretch}{1.06}
\begin{tabularx}{\textwidth}{
>{\raggedright\arraybackslash}p{0.28\textwidth}
>{\centering\arraybackslash}p{0.19\textwidth}
>{\centering\arraybackslash}p{0.21\textwidth}
>{\centering\arraybackslash}X
}
\toprule
\textbf{SETTING}
& \textbf{TRAINING}
& \textbf{FORMAL AUDIT}
& \textbf{HIGH AUDIT} \\
\midrule
\multicolumn{4}{c}{\textbf{Protocol}} \\
\midrule
Target / seed & \multicolumn{3}{c}{Three-component 1D GMM / 13} \\
Path / \(\varepsilon\) & \multicolumn{3}{c}{Linear Gaussian / 0.10} \\
Train weights & \(\{1,w_{\mathrm{sc}}\}\) & -- & -- \\
Readout weights & -- & \(\{1,w_{\mathrm{sc}}\}\) & \(\{1,w_{\mathrm{sc}}\}\) \\
Checkpoint scope & 10 fractional & All saved & Final only \\
\midrule
\multicolumn{4}{c}{\textbf{Optimization}} \\
\midrule
Hidden width & 32 & -- & -- \\
Optimizer & Adam & -- & -- \\
Learning rate & \(10^{-3}\) & -- & -- \\
Batch size & 1,024 & -- & -- \\
Optimization steps & 500,000 & -- & -- \\
\midrule
\multicolumn{4}{c}{\textbf{Audit Configuration}} \\
\midrule
Evaluation samples & -- & 128 & 256 \\
Monte Carlo samples / point & -- & 32 & 128 \\
Time-grid points & -- & 21 & 96 \\
ODE / sampling steps & -- & 160 / 160 & 160 / 160 \\
Generated samples & -- & 500,000 & 500,000 \\
Density-grid points & -- & 1,024 & 1,024 \\
\bottomrule
\end{tabularx}
\end{table}

\begin{table}[H]
\centering
\caption{\textbf{1D on-policy settings for RQ1 and RQ2.}}
\label{tab:app-1d-rq12-onpolicy-setup}
\scriptsize
\setlength{\tabcolsep}{4.0pt}
\renewcommand{\arraystretch}{1.06}
\begin{tabularx}{\textwidth}{
>{\raggedright\arraybackslash}p{0.34\textwidth}
>{\centering\arraybackslash}X
>{\centering\arraybackslash}X
}
\toprule
\textbf{SETTING}
& \textbf{ORDINARY}
& \textbf{SCORE-CALIBRATED} \\
\midrule
\multicolumn{3}{c}{\textbf{Protocol}} \\
\midrule
Target / seed & \multicolumn{2}{c}{Three-component 1D GMM / 13} \\
Source checkpoint & \multicolumn{2}{c}{Fixed-target \(w=1\), final} \\
CFM-ratio weight & \(w=1\) & \(w=w_{\mathrm{sc}}\) \\
Audit readout weights & \multicolumn{2}{c}{\(\{1,w_{\mathrm{sc}}\}\)} \\
\midrule
\multicolumn{3}{c}{\textbf{On-policy Optimization}} \\
\midrule
Hidden width & \multicolumn{2}{c}{32} \\
Rounds & \multicolumn{2}{c}{48} \\
Rollout samples / group size & \multicolumn{2}{c}{128 / 8} \\
Update steps / learning rate & \multicolumn{2}{c}{8 / \(3\times10^{-3}\)} \\
\midrule
\multicolumn{3}{c}{\textbf{Stabilization}} \\
\midrule
Advantage / ratio clip & \multicolumn{2}{c}{5 / \(10^{-5}\)} \\
Max. log-ratio / gradient norm & \multicolumn{2}{c}{5 / 10} \\
EMA decay & \multicolumn{2}{c}{0} \\
Source / EMA KL beta & \multicolumn{2}{c}{0 / 0} \\
\midrule
\multicolumn{3}{c}{\textbf{Audit Configuration}} \\
\midrule
Checkpoints / audit samples & \multicolumn{2}{c}{10 fractional / 128} \\
MC samples / time-grid points & \multicolumn{2}{c}{128 / 48} \\
ODE / sampling steps & \multicolumn{2}{c}{160 / 160} \\
Ratio references & \multicolumn{2}{c}{Source and round-old} \\
Reward samples / reward gate & \multicolumn{2}{c}{4,096 / 0.90} \\
\bottomrule
\end{tabularx}
\end{table}

The RQ3 mechanism sweeps use the same fixed-target \(w=1\) final checkpoint and vary EMA distance, ratio weight, ratio clipping, advantage clipping, source KL beta, and EMA KL beta; their controlled settings are reported separately in the RQ3 tables.

\subsubsection{RQ1: Does the smoothed identity close numerically?}
The independently estimable identity underlying
\thmref{thm:pw-practical-decomposition} is
\[
    \ell_\varepsilon(\theta;x_0)
    =
    H(q_\varepsilon^{x_0})
    +
    \mathcal J_w^{[\varepsilon,1]}(\theta;x_0)
    +
    \mathcal G_{\varepsilon,w}(\theta;x_0).
\]
\textbf{Off-policy results audit the numerical closure of this identity.}
As shown in \tabref{tab:app-1d-fixed-high-audit}, adding the independently
estimated interior term yields a small smoothed-identity closure error.  The
remaining error measures the numerical implementation---including score, ODE,
Monte Carlo, and time-quadrature approximations---rather than independently
validating the analytic theorem.
\begin{table*}[ht]
    \centering
    \caption{
    Empirical results of off-policy training on the 1D GMM.  Lower is better for all metrics.  Density \(L_1\) is the grid-normalized \(L_1\) distance \(\sum_g |\hat p_\theta(g)-p_{\mathrm{data}}(g)|\Delta g\) between the CNF endpoint density and the target density on the evaluation grid.  The \(J+H\) column is the CFM estimate used for the off-policy table
    claim.  The closure column is
    \(\mathrm{MAE}(\ell_\varepsilon,\mathcal J_w+\mathcal G+H)\); the
    displayed endpoint form is algebraically equal because
    \(\mathcal B_\varepsilon=\mathrm{NLL}-\ell_\varepsilon\).
    }
    \label{tab:app-1d-fixed-high-audit}
    \scriptsize
    \setlength{\tabcolsep}{4.5pt}
    \renewcommand{\arraystretch}{1.12}
    \begin{tabular}{llccc}
    \toprule
    Train weight & Estimate weight
    & \(\mathrm{MAE}(\mathrm{NLL},\mathcal J_w+H)\downarrow\)
    & \(\mathrm{MAE}(\ell_\varepsilon,\mathcal J_w+\mathcal G+H)\downarrow\)
    & Density \(L_1\downarrow\) \\
    \midrule
    \(w=1\) & \(w=1\) & 0.803 & 0.234 & 0.036 \\
    \(w=1\) & \(w=w_{\mathrm{sc}}\) & 0.224 & 0.199 & 0.036 \\
    \(w=w_{\mathrm{sc}}\) & \(w=1\) & 0.797 & 0.234 & 0.035 \\
    \(w=w_{\mathrm{sc}}\) & \(w=w_{\mathrm{sc}}\) & 0.222 & 0.198 & 0.035 \\
    \bottomrule
    \end{tabular}
\end{table*}
    
\textbf{On-policy results provide the same numerical audit.}
For on-policy training, the model is fine-tuned from the pretrained off-policy
checkpoint trained with $w=1$. \tabref{tab:app-1d-onpolicy-decomposition}
reports the smoothed-identity closure error separately from the CFM-only
endpoint-estimation error.
\begin{table*}[ht]
\centering
\caption{
Numerical audit on the 1D GMM at the final on-policy checkpoints.  The
\(J+H\) column reports CFM-only endpoint-estimation error. The closure column
reports \(\mathrm{MAE}(\ell_\varepsilon,J+G+H)\), written in its algebraically
equivalent endpoint form in the header. Lower is better.
}
\label{tab:app-1d-onpolicy-decomposition}
\scriptsize
\setlength{\tabcolsep}{2.6pt}
\renewcommand{\arraystretch}{1.12}
\begin{tabular}{llcccc}
\toprule
Train Variants & Estimates
& \(\mathrm{MAE}(\mathrm{NLL},\mathcal J_w+H)\downarrow\)
& \(\mathrm{MAE}(\mathrm{NLL},\mathcal J_w+\mathcal G+\mathcal B+H)\downarrow\)
& \(|G|\)
& \(|B|\) \\
\midrule
CFM-only ratio ($w=1$) & \(w=1\) & 1.928 & 0.117 & 0.418 & 2.210 \\
CFM-only ratio ($w=1$) & \(w=w_{\mathrm{sc}}\) & 2.015 & 0.112 & 0.251 & 2.232 \\
Score-calibrated ratio ($w=w_{\mathrm{sc}}$) & \(w=1\) & 2.020 & 0.118 & 0.370 & 2.270 \\
Score-calibrated ratio ($w=w_{\mathrm{sc}}$) & \(w=w_{\mathrm{sc}}\) & 2.045 & 0.117 & 0.303 & 2.276 \\
\bottomrule
\end{tabular}
\end{table*}

\begin{figure*}[ht]
\centering
\begin{minipage}{0.49\textwidth}
\centering
{\scriptsize\textbf{(a) Endpoint estimate versus smoothed closure}}\\
\includegraphics[width=\linewidth]{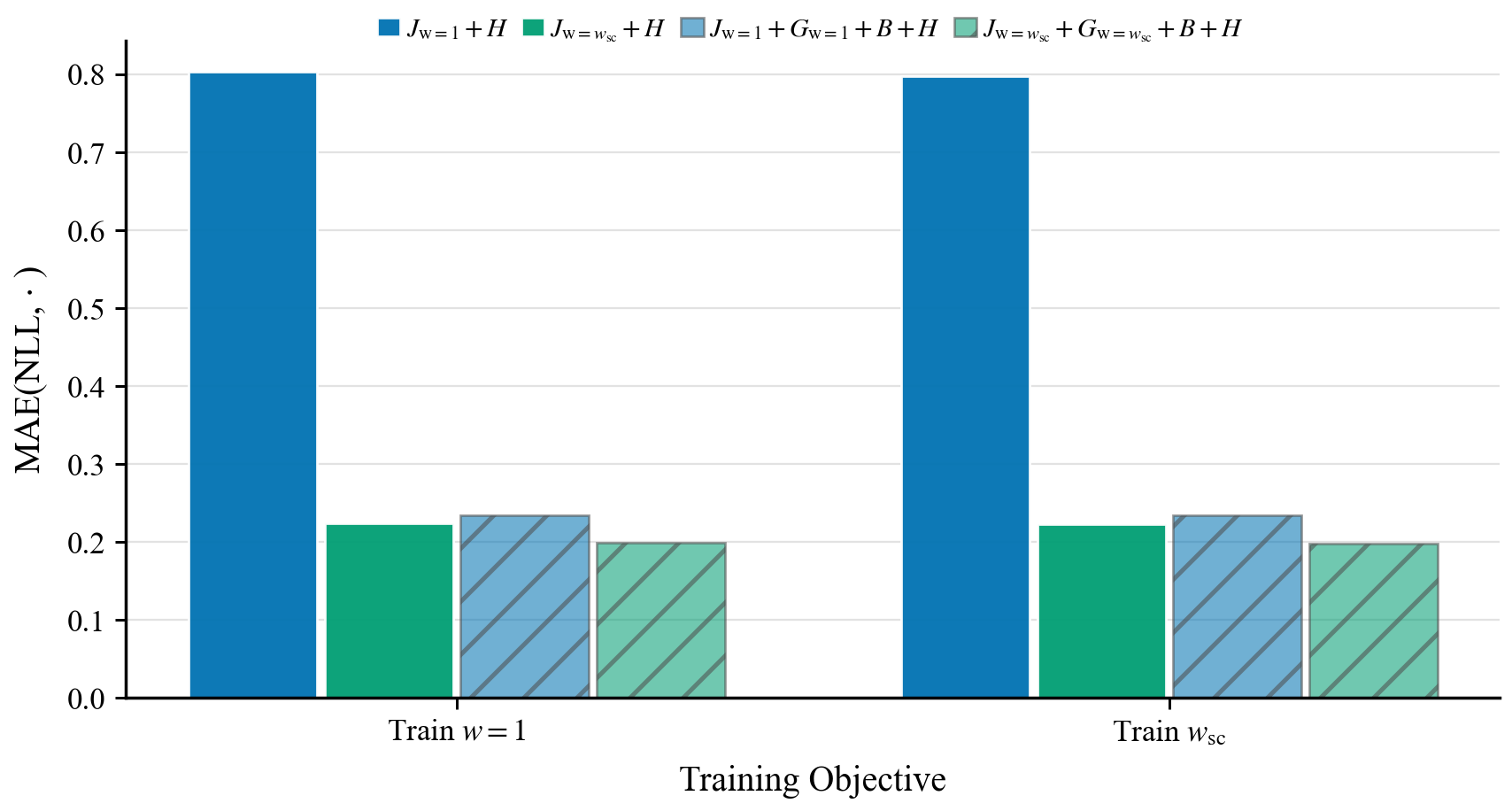}
\end{minipage}
\hfill
\begin{minipage}{0.49\textwidth}
\centering
{\scriptsize\textbf{(b) During-optimization checkpoint trajectory}}\\
\includegraphics[width=\linewidth]{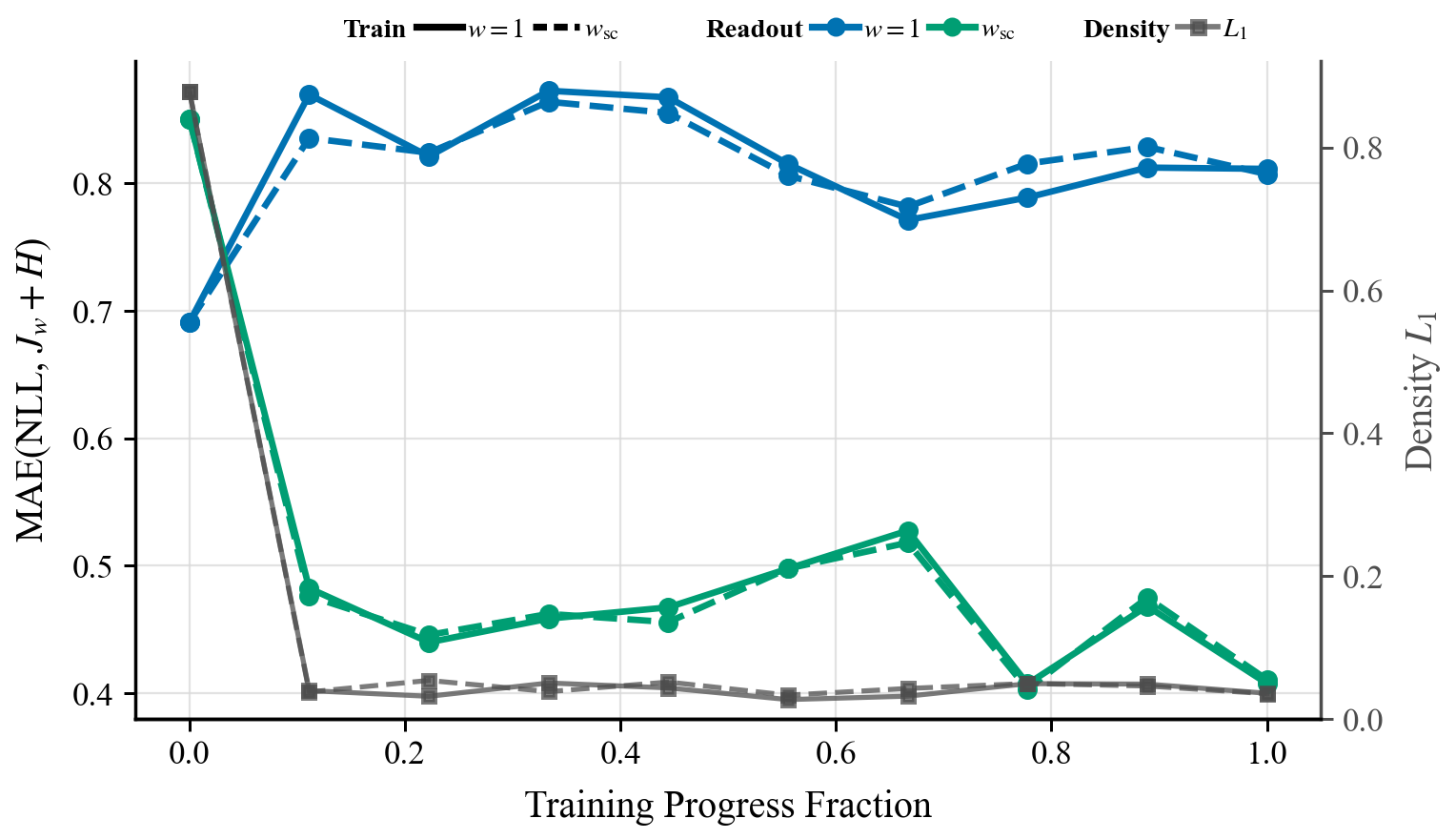}
\end{minipage}
\caption{
Off-policy 1D evaluation summary.  Panel (a) compares the final CFM endpoint
estimate \(\mathcal J_w+H\) with smoothed-identity closure
\(\mathcal J_w+\mathcal G_{\varepsilon,w}+H\) against
\(\ell_\varepsilon\).  Panel
(b) tracks \(\mathrm{MAE}(\mathrm{NLL},\mathcal J_w+H)\) across saved
checkpoints, with density \(L_1\) on the secondary axis.
}
\label{fig:app-1d-fixed-checkpoint-mae}
\end{figure*}

\subsubsection{RQ2: Do the off-policy and on-policy table conclusions hold?}
\textbf{Off-policy results support the off-policy conclusions in \tabref{tab:cfm-nll-answer}.}
\begin{itemize}[leftmargin=*]
    \item \textbf{During optimization:} In \figref{fig:app-1d-fixed-checkpoint-mae} (b), both \(w=1\) and \(w=w_{\mathrm{sc}}\) keep a non-negligible \(\mathrm{MAE}(\mathrm{NLL},\mathcal J_w+H)\) before convergence.
    \item \textbf{After optimization:} In \tabref{tab:app-1d-fixed-high-audit}, the direct CFM-only endpoint MAE is about \(0.80\) for the ordinary readout and \(0.22\) for the score-calibrated readout. Thus ordinary CFM reaches the target-distribution optimum without making its pointwise ordinary-CFM value an NLL, whereas score calibration is substantially closer when the boundary term is controlled.
\end{itemize}

\textbf{On-policy results support the on-policy conclusions in \tabref{tab:cfm-nll-answer}.}
\begin{itemize}[leftmargin=*]
    \item \textbf{During optimization:} \tabref{tab:app-1d-onpolicy-baseline} shows that the on-policy training with CFM-only ratio can achieve high reward with CFM-only ratio as the surrogate.
    \item \textbf{After optimization:} In \tabref{tab:app-1d-onpolicy-baseline}, the source-relative MAE is nonzero, so the CFM-only ratio is not exact on the evaluated samples. The round-old MAE is smaller in these runs, but this association does not prove that parameter closeness guarantees a small relative residual.
\end{itemize}


\begin{table}[ht]
    \centering
    \caption{
    Formal 1D on-policy baseline evaluation.  Ratio MAE is reported with the score-calibrated estimate.  For the rollout sample set \(S\), the reported ratio error is
    \(\mathrm{MAE}(\theta_{\mathrm{ref}})=|S|^{-1}\sum_{x_0\in S}
    |\Delta\ell_{\mathrm{clean}}^{\theta/\theta_{\mathrm{ref}}}(x_0)
    -\widehat\Delta_{\theta,\theta_{\mathrm{ref}}}^{\mathrm{CFM}}(x_0)|\).
    Source MAE sets \(\theta_{\mathrm{ref}}=\theta_{\mathrm{src}}\), the fixed-target source checkpoint; round-old MAE sets \(\theta_{\mathrm{ref}}=\theta_{\mathrm{old}}^{\mathrm{round}}\), the per-round EMA old-policy checkpoint.
    }
    \label{tab:app-1d-onpolicy-baseline}
    \scriptsize
    \setlength{\tabcolsep}{5.0pt}
    \renewcommand{\arraystretch}{1.12}
    \begin{tabular}{lcccc}
    \toprule
    Variant & Final reward \(\uparrow\) & Best reward \(\uparrow\)
    & Source MAE \(\downarrow\) & Round-old MAE \(\downarrow\) \\
    \midrule
    CFM-only ratio & 0.995 & 1.000 & 7.900 & 0.468 \\
    Score-calibrated ratio & 0.998 & 1.000 & 65.269 & 0.102 \\
    \bottomrule
    \end{tabular}
    \end{table}


\subsubsection{RQ3: Why do on-policy CFM-ratio surrogates work?}
RQ2 showed that CFM-ratio GRPO can reach high reward while its CFM-only ratios
remain different from clean likelihood ratios. Here we call a tested
configuration \emph{stable for optimization} when the nonzero ratio error does
not prevent reward improvement. The controlled experiments summarized in
\tabref{tab:app-1d-mechanism-sweeps} and detailed in
\tabref{tab:app-1d-s1-ema-distance-detail}--\tabref{tab:app-1d-s6-ema-kl-beta-detail}
vary the EMA decay, CFM estimate weight, ratio clipping, advantage clipping,
source penalty, and EMA penalty. They measure control--stability associations;
because EMA decay is not a measured policy distance, they do not directly test
a residual-versus-distance theorem.

\begin{table}[!htbp]
\centering
\caption{
Controlled 1D sweep for on-policy CFM-ratio training. Each cell reports ``best final reward; passing configurations / tested configurations'', where a configuration passes when its final sampled reward reaches the pre-specified gate \(0.90\). These fractions are not repeated-seed success probabilities. Each row states the controlled settings explicitly; all non-controlled settings use the baseline values: ratio clip \(10^{-5}\), advantage clip \(a=5\), and zero KL/EMA-KL.
}
\label{tab:app-1d-mechanism-sweeps}
\scriptsize
\setlength{\tabcolsep}{2.6pt}
\renewcommand{\arraystretch}{1.14}
\begin{tabularx}{\textwidth}{
>{\raggedright\arraybackslash}p{0.12\textwidth}
>{\raggedright\arraybackslash}p{0.28\textwidth}
>{\centering\arraybackslash}p{0.095\textwidth}
>{\centering\arraybackslash}p{0.1\textwidth}
>{\centering\arraybackslash}p{0.095\textwidth}
>{\raggedright\arraybackslash}X
}
\toprule
Sweep & Settings
& \(\mathrm{EMA}=0\)
& \(\mathrm{EMA}=0.5\)
& \(\mathrm{EMA}=1\)
& Mechanism conclusion \\
\midrule
S1: EMA distance
& \(w\in\{1,w_{\mathrm{sc}}\}\)
& 0.994; 2/2
& 0.761; 0/2
& 0.275; 0/2
& EMA $=0$ perform best; EMA $>0$ fail. \\
S2: ratio weight
& \(w\in\{1,w_{\mathrm{sc}}\}\)
& 0.994; 2/2
& 0.769; 0/2
& 0.269; 0/2
& Different train weights cannot be helpful in EMA $>0$. \\
S3: ratio clip $\delta$
& \(\delta\in\{10^{-5},0.05,0.2,\mathrm{none}\}\) \(\times\) \(w \in\{1,w_{\mathrm{sc}}\}\)
& 0.995; 7/8
& 0.998; 5/8
& 0.976; 2/8
& The ratio clipping rescues some successful EMA \(>0\) experiments, but the effect is non-monotonic. \\
S4: advantage clip $a$
& \(a\in\{0.5,1,2,5\}\) \(\times\) \(w \in\{1,w_{\mathrm{sc}}\}\)
& 0.996; 8/8
& 0.850; 0/8
& 0.290; 0/8
& Advantage clipping can perform well only with EMA \(=0\). \\
S5: KL beta $\beta$
& \(\beta\in\{0,10^{-4},10^{-3},10^{-2}\}\) \(\times\) \(w \in\{1,w_{\mathrm{sc}}\}\)
& 0.996; 8/8
& 0.800; 0/8
& 0.282; 0/8
& Source-KL regularization can perform well only with EMA \(=0\). \\
S6: EMA-KL beta $\beta_{\mathrm{ema}}$
& \(\beta_{\mathrm{ema}}\in\{0,10^{-4},10^{-3},10^{-2}\}\) \(\times\) \(w \in\{1,w_{\mathrm{sc}}\}\)
& 0.995; 8/8
& 0.881; 0/8
& 0.282; 0/8
& EMA-KL regularization can perform well only with EMA \(=0\). \\
\bottomrule
\end{tabularx}
\end{table}

The following tables report the final-checkpoint sampled reward mean for every individual sweep setting.  All results use seed 13, and bold values meet the pre-specified reward gate of \(0.90\).  ``Ordinary'' denotes the CFM-ratio variant with $w = 1$, and ``score-cal.'' denotes the CFM-ratio variant with $w = w_{sc}$.

\begin{table*}[!htbp]
\centering
\caption{
Detailed S1 results for old-policy EMA distance.  Final reward is reported for each CFM-ratio variant; higher is better.
}
\label{tab:app-1d-s1-ema-distance-detail}
\scriptsize
\setlength{\tabcolsep}{8.0pt}
\renewcommand{\arraystretch}{1.12}
\begin{tabular}{ccc}
\toprule
EMA decay & Ordinary \((w=1)\) reward \(\uparrow\) & Score-cal. \((w=w_{\mathrm{sc}})\) reward \(\uparrow\) \\
\midrule
0   & \textbf{0.994} & \textbf{0.949} \\
0.5 & 0.377 & 0.761 \\
1   & 0.238 & 0.275 \\
\bottomrule
\end{tabular}
\par\vspace{2pt}
\parbox{0.90\textwidth}{\footnotesize\textbf{Conclusion.} Only the EMA \(=0\) have rewards over 0.9 for both variants; performance drops below for EMA \(=0.5\) and EMA \(=1\).}
\end{table*}

\begin{table*}[!htbp]
\centering
\caption{
Detailed S2 results for the CFM ratio weight.  Final reward is reported for each CFM-ratio variant; higher is better.
}
\label{tab:app-1d-s2-ratio-weight-detail}
\scriptsize
\setlength{\tabcolsep}{8.0pt}
\renewcommand{\arraystretch}{1.12}
\begin{tabular}{ccc}
\toprule
EMA decay & Ordinary \((w=1)\) reward \(\uparrow\) & Score-cal. \((w=w_{\mathrm{sc}})\) reward \(\uparrow\) \\
\midrule
0   & \textbf{0.994} & \textbf{0.994} \\
0.5 & 0.352 & 0.769 \\
1   & 0.263 & 0.269 \\
\bottomrule
\end{tabular}
\par\vspace{2pt}
\parbox{0.90\textwidth}{\footnotesize\textbf{Conclusion.} Changing the ratio weight does not repair the failure with large EMA decay rate: both weights have rewards over 0.9 at EMA \(=0\), and neither does at EMA \(=0.5\) or EMA \(=1\).}
\end{table*}

\begin{table*}[!htbp]
\centering
\caption{
Detailed S3 ratio-clip sweep.  Entries are final rewards for ordinary (Ord.) and score-calibrated (Score) CFM-ratio variants; higher is better.
}
\label{tab:app-1d-s3-ratio-clip-detail}
\scriptsize
\setlength{\tabcolsep}{5.0pt}
\renewcommand{\arraystretch}{1.12}
\begin{tabular}{ccccccc}
\toprule
& \multicolumn{2}{c}{EMA \(=0\)}
& \multicolumn{2}{c}{EMA \(=0.5\)}
& \multicolumn{2}{c}{EMA \(=1\)} \\
\cmidrule(lr){2-3}\cmidrule(lr){4-5}\cmidrule(lr){6-7}
Ratio clip \(\delta\) & Ord. \(\uparrow\) & Score \(\uparrow\) & Ord. \(\uparrow\) & Score \(\uparrow\) & Ord. \(\uparrow\) & Score \(\uparrow\) \\
\midrule
\(10^{-5}\) & \textbf{0.994} & \textbf{0.949} & 0.400 & 0.776 & 0.247 & 0.282 \\
0.05          & \textbf{0.989} & \textbf{0.980} & \textbf{0.951} & \textbf{0.976} & 0.284 & 0.283 \\
0.2           & \textbf{0.995} & \textbf{0.971} & \textbf{0.986} & \textbf{0.998} & 0.393 & 0.416 \\
None          & 0.769 & \textbf{0.979} & 0.733 & \textbf{0.984} & \textbf{0.976} & \textbf{0.914} \\
\bottomrule
\end{tabular}
\par\vspace{2pt}
\parbox{0.90\textwidth}{\footnotesize\textbf{Conclusion.} The ratio-clip sweep is the only mechanism sweep with successful EMA \(>0\) settings (\(5/8\) at EMA \(=0.5\) and \(2/8\) at EMA \(=1\)), but its effect is non-monotonic across clipping widths and variants.}
\end{table*}

\begin{table*}[!htbp]
\centering
\caption{
Detailed S4 advantage-clip sweep.  Entries are final rewards for ordinary (Ord., i.e. $w = 1$) and score-calibrated (Score, i.e. $w = w_{sc}$) CFM-ratio variants; higher is better.
}
\label{tab:app-1d-s4-adv-clip-detail}
\scriptsize
\setlength{\tabcolsep}{5.0pt}
\renewcommand{\arraystretch}{1.12}
\begin{tabular}{ccccccc}
\toprule
& \multicolumn{2}{c}{EMA \(=0\)}
& \multicolumn{2}{c}{EMA \(=0.5\)}
& \multicolumn{2}{c}{EMA \(=1\)} \\
\cmidrule(lr){2-3}\cmidrule(lr){4-5}\cmidrule(lr){6-7}
Adv. clip \(a\) & Ord. \(\uparrow\) & Score \(\uparrow\) & Ord. \(\uparrow\) & Score \(\uparrow\) & Ord. \(\uparrow\) & Score \(\uparrow\) \\
\midrule
0.5 & \textbf{0.990} & \textbf{0.993} & 0.380 & 0.788 & 0.252 & 0.268 \\
1   & \textbf{0.987} & \textbf{0.995} & 0.364 & 0.850 & 0.263 & 0.260 \\
2   & \textbf{0.992} & \textbf{0.994} & 0.390 & 0.843 & 0.243 & 0.262 \\
5   & \textbf{0.988} & \textbf{0.996} & 0.360 & 0.744 & 0.290 & 0.267 \\
\bottomrule
\end{tabular}
\par\vspace{2pt}
\parbox{0.90\textwidth}{\footnotesize\textbf{Conclusion.} Advantage clipping succeeds for all eight EMA \(=0\) settings but does not produce a successful EMA \(=0.5\) or EMA \(=1\) setting.}
\end{table*}

\begin{table*}[!htbp]
\centering
\caption{
Detailed S5 source-KL sweep.  Entries are final rewards for ordinary (Ord., i.e. $w = 1$) and score-calibrated (Score, $w = w_{sc}$) CFM-ratio variants; higher is better.
}
\label{tab:app-1d-s5-kl-beta-detail}
\scriptsize
\setlength{\tabcolsep}{5.0pt}
\renewcommand{\arraystretch}{1.12}
\begin{tabular}{ccccccc}
\toprule
& \multicolumn{2}{c}{EMA \(=0\)}
& \multicolumn{2}{c}{EMA \(=0.5\)}
& \multicolumn{2}{c}{EMA \(=1\)} \\
\cmidrule(lr){2-3}\cmidrule(lr){4-5}\cmidrule(lr){6-7}
Source-KL \(\beta\) & Ord. \(\uparrow\) & Score \(\uparrow\) & Ord. \(\uparrow\) & Score \(\uparrow\) & Ord. \(\uparrow\) & Score \(\uparrow\) \\
\midrule
0           & \textbf{0.994} & \textbf{0.949} & 0.400 & 0.776 & 0.247 & 0.282 \\
\(10^{-4}\) & \textbf{0.991} & \textbf{0.996} & 0.364 & 0.632 & 0.262 & 0.250 \\
\(10^{-3}\) & \textbf{0.985} & \textbf{0.995} & 0.386 & 0.800 & 0.235 & 0.252 \\
\(10^{-2}\) & \textbf{0.965} & \textbf{0.986} & 0.357 & 0.535 & 0.250 & 0.252 \\
\bottomrule
\end{tabular}
\par\vspace{2pt}
\parbox{0.90\textwidth}{\footnotesize\textbf{Conclusion.} Source-KL regularization succeeds for all eight EMA \(=0\) settings but cannot by itself rescue experiments with EMA \(=0.5\) and EMA \(=1\).}
\end{table*}

\begin{table*}[!htbp]
\centering
\caption{
Detailed S6 EMA-KL sweep.  Entries are final rewards for ordinary (Ord., i.e. $w = 1$) and score-calibrated (Score, i.e. $w = w_{sc}$) CFM-ratio variants; higher is better.
}
\label{tab:app-1d-s6-ema-kl-beta-detail}
\scriptsize
\setlength{\tabcolsep}{5.0pt}
\renewcommand{\arraystretch}{1.12}
\begin{tabular}{ccccccc}
\toprule
& \multicolumn{2}{c}{EMA \(=0\)}
& \multicolumn{2}{c}{EMA \(=0.5\)}
& \multicolumn{2}{c}{EMA \(=1\)} \\
\cmidrule(lr){2-3}\cmidrule(lr){4-5}\cmidrule(lr){6-7}
EMA-KL \(\beta_{\mathrm{ema}}\) & Ord. \(\uparrow\) & Score \(\uparrow\) & Ord. \(\uparrow\) & Score \(\uparrow\) & Ord. \(\uparrow\) & Score \(\uparrow\) \\
\midrule
0           & \textbf{0.994} & \textbf{0.949} & 0.400 & 0.776 & 0.247 & 0.282 \\
\(10^{-4}\) & \textbf{0.991} & \textbf{0.995} & 0.355 & 0.777 & 0.262 & 0.250 \\
\(10^{-3}\) & \textbf{0.989} & \textbf{0.995} & 0.392 & 0.881 & 0.235 & 0.252 \\
\(10^{-2}\) & \textbf{0.980} & \textbf{0.987} & 0.349 & 0.716 & 0.250 & 0.252 \\
\bottomrule
\end{tabular}
\par\vspace{2pt}
\parbox{0.90\textwidth}{\footnotesize\textbf{Conclusion.} EMA-KL regularization also succeeds only with EMA $= 0$ and does not independently repair experiments with EMA $=0.5$ and EMA $=1$.}
\end{table*}

Together, \tabref{tab:app-1d-mechanism-sweeps} and the six detailed tables
\tabref{tab:app-1d-mechanism-sweeps}--\tabref{tab:app-1d-s6-ema-kl-beta-detail}
show an EMA-dependent ordering in the tested configurations: EMA \(=0\) is
most consistently associated with passing the reward gate, while the ratio-clip
sweep contains the most passing configurations beyond EMA \(=0\). This is an
empirical association at one seed, not a causal or local-accuracy result.

\subsection{2D Toy Experiments}
\label{app:2d-toy-experiments}

The 2D toy experiments use the same audit logic as the 1D experiments, but
stress the claims under known distributions with different visible geometries.
They answer three research questions:
\begin{itemize}
    \item[(RQ1)] Whether the pointwise decomposition still closes when \(x_0\in\mathbb R^2\);
    \item[(RQ2)] Whether the off-policy and on-policy conclusions in \tabref{tab:cfm-nll-answer} remain empirically supported;
    \item[(RQ3)] Whether the effective mechanisms in 1D settings can also be effective in 2D settings.
\end{itemize}

\subsubsection{Experiment Setup}
The 2D targets are a five-component diagonal Gaussian mixture, a banana distribution, and a two-moons distribution.  The GMM target is
\[
p_{\mathrm{data}}(x)
=\sum_{k=1}^5 \pi_k\,\mathcal N(x;\mu_k,\mathrm{diag}(\sigma_k^2)),
\]
with
\[
\begin{gathered}
\pi=(0.20,0.18,0.24,0.20,0.18),\\
\mu=\{(-2.0,-1.2),(-1.8,1.3),(0.2,0.0),(2.0,-0.8),(2.1,1.2)\},\\
\sigma=\{(0.34,0.42),(0.38,0.34),(0.55,0.48),(0.36,0.44),(0.32,0.38)\}.
\end{gathered}
\]
The banana target samples \(u\sim\mathcal N(0,1)\) and \(\eta\sim\mathcal N(0,0.35^2)\), then sets \((x_1,x_2)=(u,0.35(u^2-1)+\eta)\).  The two-moons target samples \(\theta\sim\mathrm{Unif}[0,\pi]\), chooses one of the two arcs uniformly, uses the arc centers \((\cos\theta,\sin\theta)\) and \((1-\cos\theta,-\sin\theta-0.45)\), subtracts the center \((0.5,-0.225)\), and adds isotropic Gaussian noise with standard deviation \(0.07\).  The on-policy reward is the shared 2D bump
\[
R(x)=\exp\!\left(-\frac{\|x-c\|_2^2}{2\cdot 0.65^2}\right),
\]
with \(c=(2.1,1.2)\) for GMM, \(c=(1.2,0.154)\) for banana, and
\(c=(-0.5,1.05)\) for two moons.

All 2D CFM runs use the linear Gaussian path
\(X_t=(1-t)x_0+t z\), \(z\sim\mathcal N(0,I_2)\), with
\(\varepsilon=0.10\). \tabref{tab:app-2d-rq12-fixed-setup}
and \tabref{tab:app-2d-rq12-onpolicy-setup} collect the off-policy and on-policy configurations used for RQ1 and RQ2.

\begin{table}[H]
\centering
\caption{\textbf{2D off-policy settings for RQ1 and RQ2.}  The formal audit covers the saved-checkpoint trajectory, whereas the high-budget audit evaluates only the final checkpoint.}
\label{tab:app-2d-rq12-fixed-setup}
\scriptsize
\setlength{\tabcolsep}{3.4pt}
\renewcommand{\arraystretch}{1.06}
\begin{tabularx}{\textwidth}{
>{\raggedright\arraybackslash}p{0.28\textwidth}
>{\centering\arraybackslash}p{0.19\textwidth}
>{\centering\arraybackslash}p{0.21\textwidth}
>{\centering\arraybackslash}X
}
\toprule
\textbf{SETTING}
& \textbf{TRAINING}
& \textbf{FORMAL AUDIT}
& \textbf{HIGH AUDIT} \\
\midrule
\multicolumn{4}{c}{\textbf{Protocol}} \\
\midrule
Targets & \multicolumn{3}{c}{GMM, banana, and two moons} \\
Seeds & \multicolumn{3}{c}{13, 17, and 23} \\
Path / \(\varepsilon\) & \multicolumn{3}{c}{Linear Gaussian / 0.10} \\
Train weights & \(\{1,w_{\mathrm{sc}}\}\) & -- & -- \\
Readout weights & -- & \(\{1,w_{\mathrm{sc}}\}\) & \(\{1,w_{\mathrm{sc}}\}\) \\
Checkpoint scope & 10 fractional & All saved & Final only \\
\midrule
\multicolumn{4}{c}{\textbf{Optimization}} \\
\midrule
Hidden width & 64 & -- & -- \\
Optimizer & Adam & -- & -- \\
Learning rate & \(10^{-3}\) & -- & -- \\
Batch size & 2,048 & -- & -- \\
Optimization steps & 500,000 & -- & -- \\
\midrule
\multicolumn{4}{c}{\textbf{Audit Configuration}} \\
\midrule
Evaluation samples & -- & 192 & 256 \\
Monte Carlo samples / point & -- & 32 & 64 \\
Time-grid points & -- & 21 & 33 \\
ODE / sampling steps & -- & 160 / 160 & 200 / 200 \\
Generated samples & -- & 500,000 & 500,000 \\
Density grid & -- & \(160\times160\) & \(192\times192\) \\
\bottomrule
\end{tabularx}
\end{table}

\begin{table}[H]
\centering
\caption{\textbf{2D on-policy settings for RQ1 and RQ2.}  Ordinary and score-calibrated denote the CFM-ratio readout used by GRPO.}
\label{tab:app-2d-rq12-onpolicy-setup}
\scriptsize
\setlength{\tabcolsep}{4.0pt}
\renewcommand{\arraystretch}{1.06}
\begin{tabularx}{\textwidth}{
>{\raggedright\arraybackslash}p{0.34\textwidth}
>{\centering\arraybackslash}X
>{\centering\arraybackslash}X
}
\toprule
\textbf{SETTING}
& \textbf{ORDINARY}
& \textbf{SCORE-CALIBRATED} \\
\midrule
\multicolumn{3}{c}{\textbf{Protocol}} \\
\midrule
Targets / seed & \multicolumn{2}{c}{GMM, banana, and two moons / 13} \\
Source checkpoint & \multicolumn{2}{c}{Corresponding fixed-target \(w=1\), final} \\
CFM-ratio weight & \(w=1\) & \(w=w_{\mathrm{sc}}\) \\
Audit readout weights & \multicolumn{2}{c}{\(\{1,w_{\mathrm{sc}}\}\)} \\
\midrule
\multicolumn{3}{c}{\textbf{On-policy Optimization}} \\
\midrule
Hidden width & \multicolumn{2}{c}{64} \\
Rounds & \multicolumn{2}{c}{100} \\
Rollout samples / group size & \multicolumn{2}{c}{256 / 8} \\
Update steps / learning rate & \multicolumn{2}{c}{4 / \(5\times10^{-4}\)} \\
\midrule
\multicolumn{3}{c}{\textbf{Stabilization}} \\
\midrule
Advantage / ratio clip & \multicolumn{2}{c}{5 / \(10^{-5}\)} \\
Max. log-ratio / gradient norm & \multicolumn{2}{c}{2 / 1} \\
EMA decay & \multicolumn{2}{c}{0} \\
Source / EMA KL beta & \multicolumn{2}{c}{\(10^{-2}\) / 0} \\
\midrule
\multicolumn{3}{c}{\textbf{Audit Configuration}} \\
\midrule
Checkpoints / audit samples & \multicolumn{2}{c}{10 fractional / 256} \\
MC samples / time-grid points & \multicolumn{2}{c}{32 / 21} \\
ODE / sampling steps & \multicolumn{2}{c}{160 / 160} \\
Ratio references & \multicolumn{2}{c}{Source and round-old} \\
Reward samples / reward gate & \multicolumn{2}{c}{8,192 / 0.90} \\

\bottomrule
\end{tabularx}
\end{table}

The RQ3 mechanism sweeps use the off-policy checkpoints with train weight $w=1$ and reward gate 0.90, and separately vary EMA distance, estimate weight, ratio clipping, advantage clipping, source KL, and EMA KL; their controlled settings are reported separately in the RQ3 tables.

\subsubsection{RQ1: Does the smoothed identity close numerically in 2D?}
\textbf{Off-policy 2D results audit the numerical closure of the smoothed identity.}
\tabref{tab:app-2d-fixed-high-audit} shows the CFM-only endpoint error and
the separate smoothed-identity closure error.  The latter is
\(\mathrm{MAE}(\ell_\varepsilon,\mathcal J_w+\mathcal G+H)\); it is not an
independent empirical proof of the theorem or of the boundary-term limit.

\begin{table}[ht]
\centering
\caption{
High-budget off-policy 2D evaluation.  Entries are means over seeds 13, 17, and 23 and over the two train weights.  Lower is better.  Density \(L_1\) is the grid-normalized \(L_1\) distance between the CNF endpoint density and the target density.  The \(J+H\) columns report endpoint-estimation error, while the closure columns report \(\mathrm{MAE}(\ell_\varepsilon,\mathcal J_w+\mathcal G+H)\).
}
\label{tab:app-2d-fixed-high-audit}
\scriptsize
\setlength{\tabcolsep}{4.0pt}
\renewcommand{\arraystretch}{1.12}
\begin{tabular}{lccccc}
\toprule
Dataset
& Density \(L_1\downarrow\)
& \multicolumn{2}{c}{\(\mathrm{MAE}(\mathrm{NLL},\mathcal J_w+H)\downarrow\)}
& \multicolumn{2}{c}{\(\mathrm{MAE}(\ell_\varepsilon,\mathcal J_w+\mathcal G+H)\downarrow\)} \\
\cmidrule(lr){3-4}\cmidrule(lr){5-6}
& & \(w=1\) Estimate & \(w=w_{\mathrm{sc}}\) Estimate
& \(w=1\) Estimate & \(w=w_{\mathrm{sc}}\) Estimate \\
\midrule
GMM & 0.076 & 1.001 & 0.411 & 0.397 & 0.387 \\
Banana & 0.043 & 1.587 & 0.342 & 0.323 & 0.312 \\
Two moons & 0.326 & 0.580 & 0.895 & 0.311 & 0.302 \\
\bottomrule
\end{tabular}
\end{table}

\begin{figure*}[ht]
\centering
\includegraphics[width=0.99\textwidth]{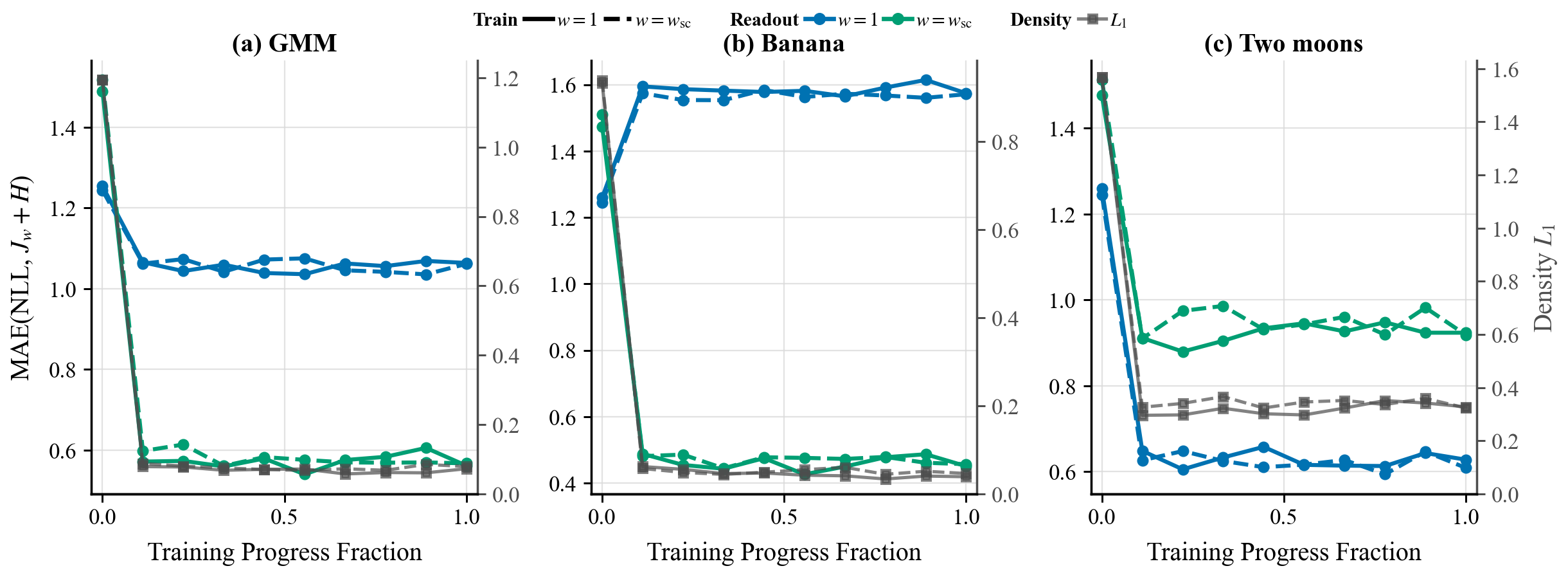}
\caption{
Off-policy 2D checkpoint trajectories, averaged over seeds 13, 17, and 23. The left axis reports \(\mathrm{MAE}(\mathrm{NLL},\mathcal J_w+H)\); color denotes the readout weight and line style denotes the fixed-target training weight.  Gray square curves report density \(L_1\) on the secondary axes, which are scaled separately for each target family.
}
\label{fig:app-2d-fixed-checkpoint-mae}
\end{figure*}

\textbf{On-policy 2D results provide the same numerical audit.}
For on-policy training, the model is fine-tuned from the pretrained off-policy
checkpoint trained with $w=1$. \tabref{tab:app-2d-onpolicy-decomposition}
reports both endpoint-estimation error and smoothed-identity closure error.

\begin{table*}[ht]
\centering
\caption{
Final 2D numerical audit for the on-policy experiments. $\mathcal J_w + H$
denotes the CFM-only endpoint estimate. The closure columns report
\(\mathrm{MAE}(\ell_\varepsilon,\mathcal J_w+\mathcal G+H)\).
}
\label{tab:app-2d-onpolicy-decomposition}
\scriptsize
\setlength{\tabcolsep}{3.2pt}
\renewcommand{\arraystretch}{1.12}
\begin{tabular}{llcccc}
\toprule
Dataset & Train weight
& \multicolumn{2}{c}{\(\mathrm{MAE}(\mathrm{NLL},\mathcal J_w+H)\downarrow\)}
& \multicolumn{2}{c}{\(\mathrm{MAE}(\ell_\varepsilon,\mathcal J_w+\mathcal G+H)\downarrow\)} \\
\cmidrule(lr){3-4}\cmidrule(lr){5-6}
& & \(w=1\) Estimate & \(w=w_{\mathrm{sc}}\) Estimate
& \(w=1\) Estimate & \(w=w_{\mathrm{sc}}\) Estimate \\
\midrule
GMM & $w = 1$
& 2.069 & 2.057 & 0.250 & 0.231\\
GMM & $w = w_{\mathrm{sc}}$
& 49.054 & 206.343 & 0.841 & 0.836 \\
Banana & $w = 1$
& 1.374 & 1.415 & 0.214 & 0.230 \\
Banana & $w = w_{\mathrm{sc}}$
& 2.331 & 2.396 & 0.378 & 0.389 \\
Two moons & $w = 1$
& 3.132 & 3.183 & 0.362 & 0.349 \\
Two moons & $w = w_{\mathrm{sc}}$
& 1.909 & 2.347 & 0.286 & 0.257 \\
\bottomrule
\end{tabular}
\end{table*}

\textbf{Remark:} Without the source KL regularization, some on-policy runs become non-finite during training, so their numerical closure audits are also non-finite.

\subsubsection{RQ2: Do the table conclusions hold beyond 1D?}
\textbf{The off-policy conclusions in \tabref{tab:cfm-nll-answer} hold for 2D settings.}
\begin{itemize}[leftmargin=*]
    \item \textbf{During optimization:} In \figref{fig:app-2d-fixed-checkpoint-mae}, all three target families retain a non-negligible \(\mathrm{MAE}(\mathrm{NLL},\mathcal J_w+H)\) under both training and readout weights before and at convergence.
    \item \textbf{After optimization:} In \tabref{tab:app-2d-fixed-high-audit}, the direct ordinary-CFM endpoint MAE is \(1.001\) on GMM and \(1.587\) on banana, while the score-calibrated readout reduces it to \(0.411\) and \(0.342\). The two-moons result shows that score calibration alone is insufficient when the finite-\(\varepsilon\) boundary term is large.
\end{itemize}

\textbf{The on-policy conclusions in \tabref{tab:cfm-nll-answer} hold for 2D settings.}
\begin{itemize}[leftmargin=*]
    \item \textbf{During optimization:} \tabref{tab:app-2d-onpolicy-baseline} shows that the on-policy training with CFM-only ratio can achieve high reward with CFM-only ratio as the surrogate.
    \item \textbf{After optimization:} In \tabref{tab:app-2d-onpolicy-baseline}, the source-relative MAEs are nonzero, so the CFM-only ratios are not exact on the evaluated samples. The round-old MAEs are smaller in these runs, but this association is not a theorem that reference closeness forces residual cancellation.
\end{itemize}

\begin{table*}[ht]
\centering
\caption{
Verification of the on-policy and after-optimization conclusions in \tabref{tab:cfm-nll-answer} for 2D settings. "Oridinary" means $w = 1$, "Score-calibrated" means $w = w_{\mathrm{sc}}$. "Source" means the ratio is estimated on the source, "Round-old" means the ratio is estimated on the old policy.
}
\label{tab:app-2d-onpolicy-baseline}
\scriptsize
\setlength{\tabcolsep}{4.0pt}
\renewcommand{\arraystretch}{1.12}
\begin{tabular}{lcccccc}
\toprule
Dataset
& Ordinary reward \(\uparrow\)
& Score-calibrated reward \(\uparrow\)
& \multicolumn{2}{c}{Ordinary ratio MAE \(\downarrow\)}
& \multicolumn{2}{c}{Score-calibrated ratio MAE \(\downarrow\)} \\
\cmidrule(lr){4-5}\cmidrule(lr){6-7}
& & & Source & Round-old & Source & Round-old \\
\midrule
GMM & 0.981 & 0.713 & 1.570 & 0.046 & 138.296 & 1.209 \\
Banana & 0.967 & 0.986 & 2.068 & 0.059 & 2.444 & 0.082 \\
Two moons & 0.979 & 0.934 & 12.197 & 0.666 & 3.015 & 0.054 \\
\bottomrule
\end{tabular}
\end{table*}

\subsubsection{RQ3: Which controls are associated with stable 2D CFM-ratio updates?}
As in 1D, a tested configuration passes when its final sampled reward reaches
at least \(0.90\). \tabref{tab:app-2d-mechanism-sweeps} shows that
\textbf{EMA $=0$ is most consistently associated with passing the reward
gate}. The ratio-clip sweep is the only controlled sweep with passing
EMA \(>0\) configurations, but its effect is non-monotonic. These
single-seed configuration fractions do not establish causality or a local
ratio-error guarantee.

\begin{table}[!htbp]
\centering
\caption{Controlled 2D sweep for on-policy CFM-ratio training. Each cell reports ``best final reward; passing configurations / tested configurations'' aggregated over the three 2D datasets at seed 13. These fractions are not repeated-seed success probabilities.}
\label{tab:app-2d-mechanism-sweeps}
\scriptsize
\setlength{\tabcolsep}{2.0pt}
\renewcommand{\arraystretch}{1.14}
\begin{tabularx}{\textwidth}{
>{\raggedright\arraybackslash}p{0.12\textwidth}
>{\raggedright\arraybackslash}p{0.2\textwidth}
>{\centering\arraybackslash}p{0.095\textwidth}
>{\centering\arraybackslash}p{0.1\textwidth}
>{\centering\arraybackslash}p{0.1\textwidth}
>{\centering\arraybackslash}p{0.095\textwidth}
>{\raggedright\arraybackslash}X
}
\toprule
Sweep & Settings
& \(\mathrm{EMA}=0\)
& \(\mathrm{EMA}=0.5\)
& \(\mathrm{EMA}=0.9\)
& \(\mathrm{EMA}=1\)
& Mechanism conclusion \\
\midrule
S1: EMA distance
& \(w\in\{1,w_{\mathrm{sc}}\}\)
& 0.984; 4/6
& 0.388; 0/6
& 0.327; 0/6
& 0.289; 0/6
& Only EMA \(=0\) reaches the reward gate 0.9; small EMA decay rate is the primary condition. \\
S2: ratio weight
& \(w\in\{1,w_{\mathrm{sc}}\}\)
& 0.984; 5/6
& 0.381; 0/6
& 0.326; 0/6
& 0.294; 0/6
& Train weight choice changes one EMA-0 outcome but does not repair experiments with large EMA decay rate. \\
S3: ratio clip $\delta$
& \(\delta\in\{10^{-5},0.05,0.2,\mathrm{none}\}\) \(\times\) \(w \in\{1,w_{\mathrm{sc}}\}\)
& 0.999; 21/24
& 0.998; 17/24
& 0.998; 11/24
& 0.985; 4/24
& This is the only sweep with EMA \(>0\) successes, but the effect is non-monotonic and EMA-1 successes use no clip. \\
S4: advantage clip $a$
& \(a\in\{0.5,1,2,5\}\) \(\times\) \(w \in\{1,w_{\mathrm{sc}}\}\)
& 0.985; 20/24
& 0.410; 0/24
& 0.328; 0/24
& 0.295; 0/24
& Advantage clipping succeeds only with EMA \(=0\) and cannot rescue experiments with large EMA decay rate. \\
S5: source KL $\beta$
& \(\beta\in\{0,10^{-4},10^{-3},10^{-2}\}\) \(\times\) \(w \in\{1,w_{\mathrm{sc}}\}\)
& 0.993; 17/24
& 0.468; 0/24
& 0.346; 0/24
& 0.302; 0/24
& Positive source KL penalties avoid the three \(\beta=0\) non-finite failures, but do not rescue experiments with large EMA decay rate. \\
S6: EMA KL $\beta_{\mathrm{old}}$
& \(\beta_{\mathrm{old}}\in\{0,10^{-4},10^{-3},10^{-2}\}\) \(\times\) \(w \in\{1,w_{\mathrm{sc}}\}\)
& 0.985; 19/24
& 0.408; 0/24
& 0.328; 0/24
& 0.297; 0/24
& EMA-KL regularization also succeeds only with EMA \(=0\) in this sweep. \\
\bottomrule
\end{tabularx}
\end{table}

The following tables report the final-checkpoint sampled reward mean for every individual sweep setting.  All results use seed 13.  Each paired entry is ``ordinary / score-calibrated'' training weight. Bold values means it meets the final reward gate of \(0.90\). \textsc{NF} denotes a run with status \texttt{failed\_nonfinite}, rather than a measured zero reward.

\begin{table}[H]
\centering
\caption{
Detailed S1 old-policy EMA-decay sweep. Each entry reports ordinary / score-calibrated train weight final reward; higher is better. EMA decay specifies the update rule and is not a measured policy distance.
}
\label{tab:app-2d-s1-ema-distance-detail}
\scriptsize
\setlength{\tabcolsep}{7.0pt}
\renewcommand{\arraystretch}{1.12}
\begin{tabular}{lcccc}
\toprule
Dataset & EMA \(=0\) & EMA \(=0.5\) & EMA \(=0.9\) & EMA \(=1\) \\
\midrule
GMM & \textbf{0.982} / 0.804 & 0.229 / 0.292 & 0.187 / 0.211 & 0.161 / 0.170 \\
Banana & \textbf{0.970} / \textbf{0.984} & 0.350 / 0.359 & 0.287 / 0.300 & 0.265 / 0.271 \\
Two moons & \textbf{0.983} / 0.899 & 0.349 / 0.388 & 0.305 / 0.327 & 0.277 / 0.289 \\
\bottomrule
\end{tabular}
\par\vspace{2pt}
\parbox{0.94\textwidth}{\footnotesize\textbf{Conclusion.} Only the EMA \(=0\) produces successful settings (\(4/6\)); none of the EMA \(=0.5,0.9,1\) settings reaches the reward gate 0.9.}
\end{table}

\begin{table}[H]
\centering
\caption{
Detailed S2 ratio-weight sweep.  Each entry reports ordinary / score-calibrated train weight final reward; higher is better.
}
\label{tab:app-2d-s2-ratio-weight-detail}
\scriptsize
\setlength{\tabcolsep}{7.0pt}
\renewcommand{\arraystretch}{1.12}
\begin{tabular}{lcccc}
\toprule
Dataset & EMA \(=0\) & EMA \(=0.5\) & EMA \(=0.9\) & EMA \(=1\) \\
\midrule
GMM & \textbf{0.982} / 0.779 & 0.238 / 0.288 & 0.182 / 0.195 & 0.168 / 0.164 \\
Banana & \textbf{0.970} / \textbf{0.984} & 0.348 / 0.366 & 0.285 / 0.299 & 0.267 / 0.262 \\
Two moons & \textbf{0.983} / \textbf{0.941} & 0.371 / 0.381 & 0.307 / 0.326 & 0.283 / 0.294 \\
\bottomrule
\end{tabular}
\par\vspace{2pt}
\parbox{0.94\textwidth}{\footnotesize\textbf{Conclusion.} The score-calibrated weight changes one EMA-0 outcome, giving \(5/6\) successes, but neither weight produces a success with large EMA decay rate.}
\end{table}

\begin{table}[H]
\centering
\caption{
Detailed S3 ratio-clip sweep.  Each entry reports ordinary / score-calibrated train weight final reward; higher is better.
}
\label{tab:app-2d-s3-ratio-clip-detail}
\scriptsize
\setlength{\tabcolsep}{5.0pt}
\renewcommand{\arraystretch}{1.10}
\begin{tabular}{llcccc}
\toprule
Dataset & Ratio clip \(\delta\) & EMA \(=0\) & EMA \(=0.5\) & EMA \(=0.9\) & EMA \(=1\) \\
\midrule
\multirow{4}{*}{GMM} & \(10^{-5}\) & \textbf{0.982} / 0.804 & 0.224 / 0.284 & 0.197 / 0.202 & 0.168 / 0.165 \\
& 0.05 & \textbf{0.995} / \textbf{0.945} & \textbf{0.985} / \textbf{0.976} & 0.339 / 0.434 & 0.187 / 0.194 \\
& 0.2 & \textbf{0.995} / \textbf{0.948} & \textbf{0.997} / \textbf{0.976} & 0.882 / \textbf{0.929} & 0.254 / 0.267 \\
& None & \textbf{0.995} / \textbf{0.997} & \textbf{0.997} / \textbf{0.998} & \textbf{0.995} / \textbf{0.972} & 0.712 / 0.861 \\
\midrule
\multirow{4}{*}{Banana} & \(10^{-5}\) & \textbf{0.970} / \textbf{0.984} & 0.346 / 0.358 & 0.292 / 0.291 & 0.269 / 0.270 \\
& 0.05 & \textbf{0.995} / \textbf{0.997} & \textbf{0.991} / \textbf{0.992} & 0.726 / 0.636 & 0.316 / 0.317 \\
& 0.2 & \textbf{0.997} / \textbf{0.999} & \textbf{0.996} / \textbf{0.998} & \textbf{0.987} / \textbf{0.973} & 0.479 / 0.405 \\
& None & \textbf{0.997} / \textbf{0.998} & \textbf{0.993} / \textbf{0.998} & \textbf{0.997} / \textbf{0.995} & \textbf{0.985} / \textbf{0.984} \\
\midrule
\multirow{4}{*}{Two moons} & \(10^{-5}\) & \textbf{0.983} / 0.899 & 0.369 / 0.408 & 0.307 / 0.320 & 0.288 / 0.292 \\
& 0.05 & \textbf{0.994} / \textbf{0.991} & \textbf{0.994} / \textbf{0.989} & 0.513 / 0.588 & 0.299 / 0.314 \\
& 0.2 & \textbf{0.997} / \textbf{0.996} & \textbf{0.998} / \textbf{0.987} & \textbf{0.983} / \textbf{0.938} & 0.370 / 0.369 \\
& None & 0.567 / \textbf{0.990} & 0.784 / \textbf{0.995} & \textbf{0.998} / \textbf{0.989} & \textbf{0.971} / \textbf{0.912} \\
\bottomrule
\end{tabular}
\par\vspace{2pt}
\parbox{0.94\textwidth}{\footnotesize\textbf{Conclusion.} S3 is the only sweep with EMA \(>0\) successes, but the effect is not monotonic in clip width.  At EMA \(=0.9\), \(\delta=0.2\) and no clipping pass \(5/6\) and \(6/6\) settings, respectively; at EMA \(=1\), all four successes occur with no clipping.}
\end{table}

\begin{table}[H]
\centering
\caption{
Detailed S4 advantage-clip sweep.  Each entry reports ordinary / score-calibrated train weight final reward; higher is better.
}
\label{tab:app-2d-s4-adv-clip-detail}
\scriptsize
\setlength{\tabcolsep}{5.0pt}
\renewcommand{\arraystretch}{1.10}
\begin{tabular}{llcccc}
\toprule
Dataset & Advantage clip \(a\) & EMA \(=0\) & EMA \(=0.5\) & EMA \(=0.9\) & EMA \(=1\) \\
\midrule
\multirow{4}{*}{GMM} & 0.5 & \textbf{0.962} / 0.806 & 0.238 / 0.278 & 0.200 / 0.195 & 0.165 / 0.165 \\
& 1 & \textbf{0.976} / 0.766 & 0.234 / 0.275 & 0.189 / 0.204 & 0.163 / 0.166 \\
& 2 & \textbf{0.981} / 0.764 & 0.237 / 0.273 & 0.191 / 0.203 & 0.167 / 0.169 \\
& 5 & \textbf{0.981} / 0.754 & 0.215 / 0.285 & 0.181 / 0.217 & 0.165 / 0.165 \\
\midrule
\multirow{4}{*}{Banana} & 0.5 & \textbf{0.939} / \textbf{0.981} & 0.346 / 0.336 & 0.292 / 0.282 & 0.268 / 0.268 \\
& 1 & \textbf{0.959} / \textbf{0.983} & 0.342 / 0.351 & 0.283 / 0.295 & 0.266 / 0.268 \\
& 2 & \textbf{0.970} / \textbf{0.983} & 0.339 / 0.359 & 0.291 / 0.298 & 0.267 / 0.269 \\
& 5 & \textbf{0.972} / \textbf{0.985} & 0.339 / 0.359 & 0.282 / 0.301 & 0.269 / 0.267 \\
\midrule
\multirow{4}{*}{Two moons} & 0.5 & \textbf{0.959} / \textbf{0.910} & 0.347 / 0.388 & 0.300 / 0.312 & 0.287 / 0.286 \\
& 1 & \textbf{0.984} / \textbf{0.935} & 0.378 / 0.387 & 0.309 / 0.326 & 0.284 / 0.295 \\
& 2 & \textbf{0.983} / \textbf{0.922} & 0.380 / 0.410 & 0.317 / 0.318 & 0.285 / 0.288 \\
& 5 & \textbf{0.978} / \textbf{0.937} & 0.373 / 0.383 & 0.318 / 0.328 & 0.281 / 0.288 \\
\bottomrule
\end{tabular}
\par\vspace{2pt}
\parbox{0.94\textwidth}{\footnotesize\textbf{Conclusion.} Advantage clipping produces \(20/24\) successes at EMA \(=0\), but no setting reaches the reward gate at EMA \(=0.5,0.9,1\).}
\end{table}

\begin{table}[H]
\centering
\caption{
Detailed S5 source KL regularization sweep. Each entry reports ordinary / score-calibrated train weight final reward; higher is better.
}
\label{tab:app-2d-s5-source-vf-detail}
\scriptsize
\setlength{\tabcolsep}{5.0pt}
\renewcommand{\arraystretch}{1.10}
\begin{tabular}{llcccc}
\toprule
Dataset & Source VF \(\beta\) & EMA \(=0\) & EMA \(=0.5\) & EMA \(=0.9\) & EMA \(=1\) \\
\midrule
\multirow{4}{*}{GMM} & 0 & \textbf{0.986} / \textsc{NF} & 0.219 / 0.311 & 0.198 / 0.212 & 0.168 / 0.168 \\
& \(10^{-4}\) & \textbf{0.986} / 0.793 & 0.215 / 0.322 & 0.192 / 0.216 & 0.164 / 0.169 \\
& \(10^{-3}\) & \textbf{0.986} / 0.779 & 0.235 / 0.290 & 0.195 / 0.214 & 0.169 / 0.174 \\
& \(10^{-2}\) & \textbf{0.981} / 0.754 & 0.215 / 0.285 & 0.181 / 0.217 & 0.165 / 0.165 \\
\midrule
\multirow{4}{*}{Banana} & 0 & \textbf{0.978} / \textbf{0.993} & 0.337 / 0.385 & 0.295 / 0.297 & 0.270 / 0.270 \\
& \(10^{-4}\) & \textbf{0.976} / \textbf{0.993} & 0.342 / 0.379 & 0.288 / 0.299 & 0.267 / 0.270 \\
& \(10^{-3}\) & \textbf{0.977} / \textbf{0.992} & 0.350 / 0.371 & 0.290 / 0.308 & 0.268 / 0.269 \\
& \(10^{-2}\) & \textbf{0.972} / \textbf{0.985} & 0.339 / 0.359 & 0.282 / 0.301 & 0.269 / 0.267 \\
\midrule
\multirow{4}{*}{Two moons} & 0 & \textsc{NF} / \textsc{NF} & 0.352 / 0.468 & 0.304 / 0.329 & 0.290 / 0.300 \\
& \(10^{-4}\) & \textbf{0.987} / \textbf{0.926} & 0.361 / 0.434 & 0.314 / 0.346 & 0.283 / 0.302 \\
& \(10^{-3}\) & \textbf{0.983} / 0.840 & 0.361 / 0.445 & 0.315 / 0.330 & 0.286 / 0.289 \\
& \(10^{-2}\) & \textbf{0.978} / \textbf{0.937} & 0.373 / 0.383 & 0.318 / 0.328 & 0.281 / 0.288 \\
\bottomrule
\end{tabular}
\par\vspace{2pt}
\parbox{0.94\textwidth}{\footnotesize\textbf{Conclusion.} The unregularized EMA-0 slice contains three non-finite failures.  Positive source KL penalties eliminate these non-finite outcomes and yield \(14/18\) successes at EMA \(=0\), but no experiments with large EMA decay rate succeeds.}
\end{table}

\begin{table}[H]
\centering
\caption{
Detailed S6 EMA-KL regularization sweep. Each entry reports ordinary / score-calibrated train weight final reward; higher is better.
}
\label{tab:app-2d-s6-round-old-vf-detail}
\scriptsize
\setlength{\tabcolsep}{5.0pt}
\renewcommand{\arraystretch}{1.10}
\begin{tabular}{llcccc}
\toprule
Dataset & Round-old VF \(\beta_{\mathrm{old}}\) & EMA \(=0\) & EMA \(=0.5\) & EMA \(=0.9\) & EMA \(=1\) \\
\midrule
\multirow{4}{*}{GMM} & 0 & \textbf{0.982} / 0.804 & 0.224 / 0.284 & 0.197 / 0.202 & 0.168 / 0.165 \\
& \(10^{-4}\) & \textbf{0.980} / 0.774 & 0.225 / 0.282 & 0.191 / 0.206 & 0.163 / 0.167 \\
& \(10^{-3}\) & \textbf{0.981} / 0.775 & 0.236 / 0.274 & 0.192 / 0.204 & 0.167 / 0.170 \\
& \(10^{-2}\) & \textbf{0.981} / 0.733 & 0.215 / 0.278 & 0.183 / 0.214 & 0.164 / 0.163 \\
\midrule
\multirow{4}{*}{Banana} & 0 & \textbf{0.970} / \textbf{0.984} & 0.346 / 0.358 & 0.292 / 0.291 & 0.269 / 0.270 \\
& \(10^{-4}\) & \textbf{0.969} / \textbf{0.984} & 0.351 / 0.364 & 0.284 / 0.290 & 0.267 / 0.270 \\
& \(10^{-3}\) & \textbf{0.970} / \textbf{0.985} & 0.349 / 0.352 & 0.293 / 0.299 & 0.267 / 0.267 \\
& \(10^{-2}\) & \textbf{0.971} / \textbf{0.984} & 0.330 / 0.365 & 0.280 / 0.300 & 0.269 / 0.265 \\
\midrule
\multirow{4}{*}{Two moons} & 0 & \textbf{0.983} / 0.899 & 0.369 / 0.408 & 0.307 / 0.320 & 0.288 / 0.292 \\
& \(10^{-4}\) & \textbf{0.985} / \textbf{0.934} & 0.350 / 0.398 & 0.310 / 0.328 & 0.283 / 0.297 \\
& \(10^{-3}\) & \textbf{0.980} / \textbf{0.924} & 0.372 / 0.402 & 0.317 / 0.317 & 0.288 / 0.287 \\
& \(10^{-2}\) & \textbf{0.978} / \textbf{0.935} & 0.364 / 0.384 & 0.317 / 0.328 & 0.281 / 0.288 \\
\bottomrule
\end{tabular}
\par\vspace{2pt}
\parbox{0.94\textwidth}{\footnotesize\textbf{Conclusion.} EMA-KL regularization yields \(19/24\) successes at EMA \(=0\), but none of the EMA \(=0.5,0.9,1\) settings reaches the reward gate 0.9.}
\end{table}

Together, \tabref{tab:app-2d-mechanism-sweeps} and
\tabref{tab:app-2d-s1-ema-distance-detail}--\tabref{tab:app-2d-s6-round-old-vf-detail}
show an EMA-dependent ordering among the tested configurations. A small EMA
decay is most consistently associated with passing the reward gate. The
ratio-clip sweep contains the only passing large-EMA configurations, but the
direction is non-monotonic. These single-seed results establish neither a
causal mechanism nor local or global clean-ratio calibration.

\subsection{High-D Toy Experiments}
\label{app:high-d-toy-experiments}

The high-dimensional toy experiments use the same audit logic as the 1D/2D
experiments, but stress the claims for known distributions whose geometry is
hard to inspect directly. They test whether the conclusions persist as the
dimension scales over \(d\in\{4,8,16,32\}\), following the same three RQs:
\begin{enumerate}
    \item[(RQ1)] Whether the pointwise decomposition holds.
    \item[(RQ2)] Whether the off-policy and on-policy conclusions in \tabref{tab:cfm-nll-answer} still hold.
    \item[(RQ3)] Whether the effective mechanisms in 1D and 2D toy experiments are still effective.
\end{enumerate}

\subsubsection{Experiment Setup}
The high-dimensional targets are 8-component Gaussian mixtures in \(\mathbb R^d\), with \(d\in\{4,8,16,32\}\).  We denote the datasets by GMM-4D, GMM-8D, GMM-16D, and GMM-32D.  For each dimension \(d\), the target is
\[
p_{\mathrm{data}}(x)
=\frac{1}{8}\sum_{k=0}^{7}
\mathcal N\!\left(x;\mu_k,\mathrm{diag}(\sigma_k^2)\right).
\]
For coordinate \(j\in\{0,\ldots,d-1\}\), the component means and diagonal standard deviations are generated deterministically as
\[
\mu_{k,j}=1.25\left(2\,\mathrm{bit}_{j\bmod 3}(k)-1\right),
\qquad
\sigma_{k,j}=0.55+0.15\left(((k+j)\bmod 3)-1\right),
\]
where \(\mathrm{bit}_{r}(k)\in\{0,1\}\) is the \(r\)-th binary bit of the component index \(k\).  Thus \(\mu_{k,j}\in\{-1.25,1.25\}\) and \(\sigma_{k,j}\in\{0.40,0.55,0.70\}\).  The reward is a bounded Gaussian bump centered at the first GMM component, \(c_{\mathrm{rew}}=\mu_0\),
\[
R(x)=\exp\!\left(-\frac{\|x-c_{\mathrm{rew}}\|_2^2}
{2\sigma_{\mathrm{rew}}^2}\right),
\]
where the reward-width parameter \(\sigma_{\mathrm{rew}}\) is set to \(0.92,1.30,1.84,2.60\) for dimensions 4, 8, 16, and 32, respectively.  A larger \(\sigma_{\mathrm{rew}}\) makes the high-reward region wider.

All High-D CFM runs use the same linear Gaussian path as the 1D/2D experiments,
with \(\varepsilon=0.10\).  \tabref{tab:app-highd-rq12-fixed-setup} and
\tabref{tab:app-highd-rq12-onpolicy-setup} collect the fixed-target and
on-policy configurations used for RQ1 and RQ2.

\begin{table}[H]
\centering
\caption{\textbf{High-D off-policy settings for RQ1 and RQ2.}  The formal audit covers the saved-checkpoint trajectory, whereas the high-budget audit evaluates only the final checkpoint.}
\label{tab:app-highd-rq12-fixed-setup}
\scriptsize
\setlength{\tabcolsep}{3.4pt}
\renewcommand{\arraystretch}{1.06}
\begin{tabularx}{\textwidth}{
>{\raggedright\arraybackslash}p{0.28\textwidth}
>{\centering\arraybackslash}p{0.19\textwidth}
>{\centering\arraybackslash}p{0.21\textwidth}
>{\centering\arraybackslash}X
}
\toprule
\textbf{SETTING}
& \textbf{TRAINING}
& \textbf{FORMAL AUDIT}
& \textbf{HIGH AUDIT} \\
\midrule
\multicolumn{4}{c}{\textbf{Protocol}} \\
\midrule
Targets / seed & \multicolumn{3}{c}{8-component GMMs in 4D, 8D, 16D, and 32D / 13} \\
Path / \(\varepsilon\) & \multicolumn{3}{c}{Linear Gaussian / 0.10} \\
Train weights & \(\{1,w_{\mathrm{sc}}\}\) & -- & -- \\
Readout weights & -- & \(\{1,w_{\mathrm{sc}}\}\) & \(\{1,w_{\mathrm{sc}}\}\) \\
Checkpoint scope & 10 fractional & All saved & Final only \\
\midrule
\multicolumn{4}{c}{\textbf{Optimization}} \\
\midrule
Hidden width & 96 & -- & -- \\
Optimizer & Adam & -- & -- \\
Learning rate & \(10^{-3}\) & -- & -- \\
Batch size & 2,048 & -- & -- \\
Optimization steps & 500,000 & -- & -- \\
Gradient-norm clip & 10 & -- & -- \\
\midrule
\multicolumn{4}{c}{\textbf{Audit Configuration}} \\
\midrule
Evaluation samples & -- & 64 & 128 \\
Monte Carlo samples / point & -- & 8 & 16 \\
Time-grid points & -- & 11 & 21 \\
ODE / sampling steps & -- & 100 / 100 & 160 / 160 \\
Generated samples & -- & 4,096 & 20,000 \\
Sample-quality metric & -- & RBF-MMD & RBF-MMD \\
\bottomrule
\end{tabularx}
\end{table}

\begin{table}[H]
\centering
\caption{\textbf{High-D on-policy settings for RQ1 and RQ2.}  Ordinary and score-calibrated denote the CFM-ratio readout used by GRPO.}
\label{tab:app-highd-rq12-onpolicy-setup}
\scriptsize
\setlength{\tabcolsep}{4.0pt}
\renewcommand{\arraystretch}{1.06}
\begin{tabularx}{\textwidth}{
>{\raggedright\arraybackslash}p{0.34\textwidth}
>{\centering\arraybackslash}X
>{\centering\arraybackslash}X
}
\toprule
\textbf{SETTING}
& \textbf{ORDINARY}
& \textbf{SCORE-CALIBRATED} \\
\midrule
\multicolumn{3}{c}{\textbf{Protocol}} \\
\midrule
Targets / seed & \multicolumn{2}{c}{GMM-4D, GMM-8D, GMM-16D, and GMM-32D / 13} \\
Source checkpoint & \multicolumn{2}{c}{Corresponding fixed-target \(w=1\), final} \\
CFM-ratio weight & \(w=1\) & \(w=w_{\mathrm{sc}}\) \\
Audit readout weights & \multicolumn{2}{c}{\(\{1,w_{\mathrm{sc}}\}\)} \\
\midrule
\multicolumn{3}{c}{\textbf{On-policy Optimization}} \\
\midrule
Hidden width & \multicolumn{2}{c}{96} \\
Rounds & \multicolumn{2}{c}{1,000} \\
Rollout samples / group size & \multicolumn{2}{c}{256 / 16} \\
Update steps / learning rate & \multicolumn{2}{c}{2 / \(2.5\times10^{-4}\)} \\
\midrule
\multicolumn{3}{c}{\textbf{Stabilization}} \\
\midrule
Advantage / ratio clip & \multicolumn{2}{c}{5 / 0.05} \\
Max. log-ratio / gradient norm & \multicolumn{2}{c}{2 / 1} \\
EMA decay & \multicolumn{2}{c}{0} \\
Fixed-source / round-old VF-MSE \(\beta\) & \multicolumn{2}{c}{0.1 / 0} \\
\midrule
\multicolumn{3}{c}{\textbf{Audit Configuration}} \\
\midrule
Checkpoints / audit samples & \multicolumn{2}{c}{20 fractional / 256} \\
MC samples / time-grid points & \multicolumn{2}{c}{8 / 11} \\
ODE / sampling steps & \multicolumn{2}{c}{100 / 100} \\
Ratio references & \multicolumn{2}{c}{Source and round-old} \\
Reward samples / reward gate & \multicolumn{2}{c}{4,096 / 0.90} \\
\bottomrule
\end{tabularx}
\end{table}

Grid-based density integration is unavailable in high dimension, so RQ1/RQ2 use RBF-MMD as the sample-quality proxy rather than density \(L_1\).  RQ3 separately studies the old-policy EMA decay (S1) and its interaction with the PPO ratio-clip width (S3).

\subsubsection{RQ1: Does the smoothed identity close numerically in high dimension?}
\textbf{The off-policy results audit numerical closure from 4D through 32D.}
\tabref{tab:app-highd-fixed-audit} reports both training weights rather than
only the \(w=1\) source checkpoints used by on-policy training. Across all
eight dataset/train-weight pairs, the independently estimated smoothed-identity
closure error is \(1.000\)--\(2.637\) for the ordinary readout and
\(1.028\)--\(2.296\) for the score-calibrated readout. The closure error grows
with dimension and measures the available numerical estimator, not independent
empirical truth of \thmref{thm:pw-practical-decomposition}.

\begin{table}[H]
\centering
\caption{
High-dimensional off-policy results at seed 13.  Each checkpoint is
evaluated with 128 points, 16 Monte Carlo samples per point, a 21-point time grid, and 160 ODE steps.  The sample-quality column is an RBF-MMD proxy.  Lower is better for all metrics.
The final two columns report smoothed-identity closure
\(\mathrm{MAE}(\ell_\varepsilon,\mathcal J_w+\mathcal G+H)\).
}
\label{tab:app-highd-fixed-audit}
\scriptsize
\setlength{\tabcolsep}{3.6pt}
\renewcommand{\arraystretch}{1.12}
\begin{tabular}{llccccc}
\toprule
Dataset
& Train weight
& MMD \(\downarrow\)
& \multicolumn{2}{c}{\(\mathrm{MAE}(\mathrm{NLL},\mathcal J_w+H)\downarrow\)}
& \multicolumn{2}{c}{\(\mathrm{MAE}(\ell_\varepsilon,\mathcal J_w+\mathcal G+H)\downarrow\)} \\
\cmidrule(lr){4-5}\cmidrule(lr){6-7}
& & & \(w=1\) & \(w=w_{\mathrm{sc}}\)
& \(w=1\) & \(w=w_{\mathrm{sc}}\) \\
\midrule
GMM-4D  & \(w=1\)               & 0.028 & 2.001  & 1.081 & 1.021 & 1.028 \\
        & \(w=w_{\mathrm{sc}}\) & 0.026 & 1.967  & 1.088 & 1.000 & 1.033 \\
GMM-8D  & \(w=1\)               & 0.033 & 3.447  & 1.387 & 1.392 & 1.238 \\
        & \(w=w_{\mathrm{sc}}\) & 0.037 & 3.485  & 1.386 & 1.392 & 1.226 \\
GMM-16D & \(w=1\)               & 0.023 & 8.108  & 2.000 & 1.880 & 1.711 \\
        & \(w=w_{\mathrm{sc}}\) & 0.029 & 8.156  & 2.009 & 1.895 & 1.743 \\
GMM-32D & \(w=1\)               & 0.035 & 19.402 & 3.012 & 2.617 & 2.290 \\
        & \(w=w_{\mathrm{sc}}\) & 0.027 & 19.201 & 2.912 & 2.637 & 2.296 \\
\bottomrule
\end{tabular}
\end{table}

\textbf{The on-policy closure estimator works at moderate dimension but fails numerically at 32D.}  \tabref{tab:app-highd-onpolicy-closure} shows moderate closure errors for several 4D--16D runs, but the estimated residual terms grow to the \(10^9\) scale for the ordinary checkpoint in 32D. We report these values solely as a failure of the available estimator; they neither support nor refute the analytic identity.

\begin{table}[H]
\centering
\caption{
High-dimensional on-policy numerical audit at seed 13.  \(\mathcal J_w+H\)
is the CFM-only endpoint MAE; ``Closure'' is
\(\mathrm{MAE}(\ell_\varepsilon,\mathcal J_w+\mathcal G+H)\). Lower is better.
}
\label{tab:app-highd-onpolicy-closure}
\scriptsize
\setlength{\tabcolsep}{3.6pt}
\renewcommand{\arraystretch}{1.12}
\begin{tabular}{llcccc}
\toprule
Dataset
& Train weight
& \multicolumn{2}{c}{\(w=1\) estimate}
& \multicolumn{2}{c}{\(w=w_{\mathrm{sc}}\) estimate} \\
\cmidrule(lr){3-4}\cmidrule(lr){5-6}
& & \(\mathcal J_w+H\) & Closure & \(\mathcal J_w+H\) & Closure \\
\midrule
GMM-4D  & $w = 1$ & 8.586 & 0.758 & 8.578 & 2.548 \\
        & $w = w_{\mathrm{sc}}$ & 9.946 & 2.944 & 9.461 & 2.191 \\
GMM-8D  & $w = 1$ & 11.641 & 1.011 & 10.968 & 0.913 \\
        & $w = w_{\mathrm{sc}}$ & \(1.056{\times}10^3\) & 8.535 & \(4.324{\times}10^3\) & 6.843 \\
GMM-16D & $w = 1$ & 13.643 & 2.480 & 14.900 & 2.187 \\
        & $w = w_{\mathrm{sc}}$ & \(1.862{\times}10^3\) & \(4.632{\times}10^3\) & \(1.065{\times}10^4\) & \(4.815{\times}10^3\) \\
GMM-32D & $w = 1$ & \(3.149{\times}10^3\) & \(6.233{\times}10^8\) & \(1.634{\times}10^4\) & \(3.224{\times}10^9\) \\
        & $w = w_{\mathrm{sc}}$ & \(1.227{\times}10^4\) & \(2.702{\times}10^7\) & \(7.166{\times}10^4\) & \(2.261{\times}10^7\) \\
\bottomrule
\end{tabular}
\end{table}

\subsubsection{RQ2: Do the off-policy and on-policy table conclusions scale?}
\textbf{The off-policy conclusions in \tabref{tab:cfm-nll-answer} hold across dimensions and training weights.}
\begin{itemize}[leftmargin=*]
    \item \textbf{During optimization:} In
    \figref{fig:app-highd-fixed-checkpoint-mae}, all four dimensions retain a non-negligible \(\mathrm{MAE}(\mathrm{NLL},\mathcal J_w+H)\) under both training weights. After convergence, the score-calibrated CFM-only MAE is more accurate than the ordinary CFM-only MAE.
    \item \textbf{After optimization:} In \tabref{tab:app-highd-fixed-audit}, the ordinary CFM-only endpoint MAE grows from \(1.967\) to \(19.402\), whereas the score-calibrated readout has MAE \(1.081\)--\(3.012\). Thus the population optimum learned by ordinary CFM does not make its pointwise objective a calibrated NLL, while the score-calibrated readout is substantially closer in these boundary-controlled GMM settings.
\end{itemize}

\begin{figure}[H]
\centering
\includegraphics[width=0.99\textwidth]{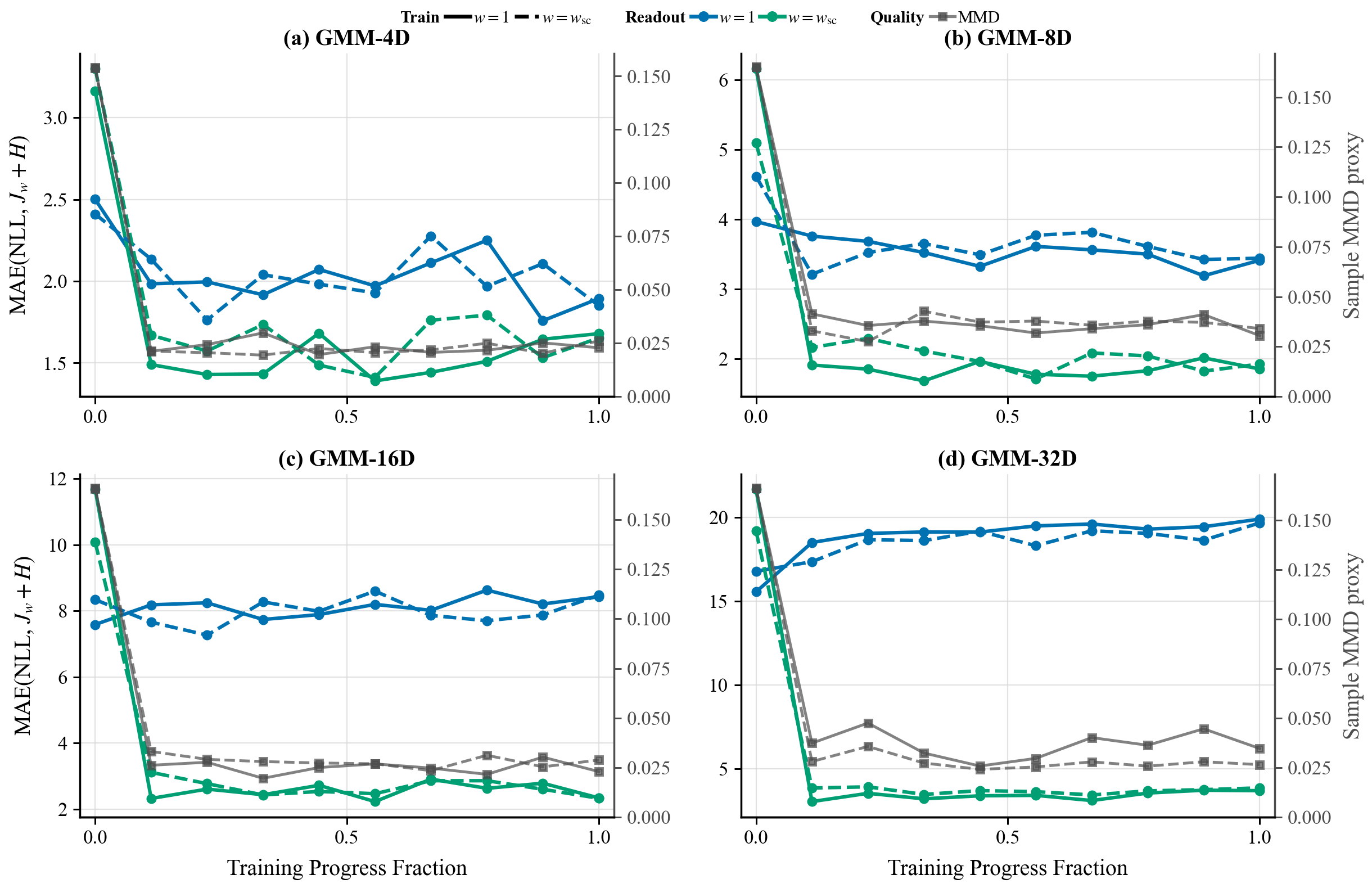}
\caption{
High-D off-policy checkpoint trajectories at seed 13.  The left axis reports the \(\mathrm{MAE}(\mathrm{NLL},\mathcal J_w+H)\); color denotes the readout weight and line style denotes the off-policy training weight.  Gray square curves report the RBF-MMD sample-quality proxy on the secondary axes, whose ranges are scaled separately by dimension.  Each checkpoint is evaluated on samples generated by that checkpoint, so the curves are self-sample audits rather than fixed-test-set learning curves.
}
\label{fig:app-highd-fixed-checkpoint-mae}
\end{figure}

\textbf{The on-policy conclusions are supported for ordinary CFM-ratio, with a scaling limitation for the score-calibrated variant.}
\begin{itemize}[leftmargin=*]
    \item \textbf{During optimization:} In \tabref{tab:app-highd-onpolicy-baseline}, ordinary CFM-ratio reaches final reward \(0.997,0.995,0.991,\) and \(0.910\) from 4D to 32D, passing the pre-specified 0.90 gate in every dimension.  Score-calibrated CFM-ratio passes only in 4D \((0.998)\) and falls to \(0.643,0.481,\) and \(0.209\), so the reward result supports a local optimization claim for the ordinary surrogate. We think this is because the weight $w = w_{\mathrm{sc}}$ will enlarge the noise scale of the CFM-ratio estimate, making the update signal unclear.
    \item \textbf{After optimization:} In \tabref{tab:app-highd-onpolicy-baseline}, for the ordinary variant with \(w=1\) readout, source-relative ratio MAE is larger than round-old-relative MAE in every dimension.  The separation grows from \(9.454\) versus \(0.815\) in 4D to \(1382.955\) versus \(2.617\) in 32D.  Therefore the CFM-only ratio is not a globally calibrated clean likelihood ratio. The association between refreshed references, smaller round-old error, and reward success does not by itself establish a local error guarantee.
\end{itemize}

\begin{table}[H]
\centering
\caption{
High-dimensional on-policy RQ2 results at seed 13 after 1,000 rounds.  Final
reward uses 4,096 generated samples.  Ratio MAE uses the ordinary CFM-ratio
variant with \(w=1\) readout at the final checkpoint; ``source'' compares with
the fixed-target source checkpoint and ``round-old'' with the current round's
old policy on the same rollout samples.  Bold reward values pass the 0.90 gate.
}
\label{tab:app-highd-onpolicy-baseline}
\scriptsize
\setlength{\tabcolsep}{4.5pt}
\renewcommand{\arraystretch}{1.12}
\begin{tabular}{lcccc}
\toprule
Dataset
& CFM final reward \(\uparrow\)
& SC-CFM final reward \(\uparrow\)
& CFM source MAE \(\downarrow\)
& CFM round-old MAE \(\downarrow\) \\
\midrule
GMM-4D  & \textbf{0.997} & \textbf{0.998} & 9.454 & 0.815 \\
GMM-8D  & \textbf{0.995} & 0.643 & 13.911 & 0.746 \\
GMM-16D & \textbf{0.991} & 0.481 & 18.298 & 1.574 \\
GMM-32D & \textbf{0.910} & 0.209 & 1382.955 & 2.617 \\
\bottomrule
\end{tabular}
\end{table}

\subsubsection{RQ3: Which control--stability associations persist in high dimension?}
As in the 1D and 2D studies, a setting is successful when its final sampled reward reaches at least \(0.90\).  The formal High-D E3 study contains two mechanism sweeps: S1 varies the old-policy EMA decay and S3 jointly varies the EMA decay and ratio clipping value.  Both sweeps use seed 13, run for 1,000 rounds, compare ordinary and score-calibrated CFM-ratio weights, and retain the source KL beta \(\beta = 0.1\). The ratio clipping value set is \(\{10^{-5},0.05,0.2,\mathrm{none}\}\).

\begin{table}[!htbp]
\centering
\caption{
Controlled High-D sweep for on-policy CFM-ratio training. Each cell reports ``best final reward; passing configurations / tested configurations'' at seed 13. Passing means final sampled reward at least \(0.90\); the fractions are not repeated-seed success probabilities.
}
\label{tab:app-highd-mechanism-sweep}
\scriptsize
\setlength{\tabcolsep}{2.0pt}
\renewcommand{\arraystretch}{1.14}
\begin{tabularx}{\textwidth}{
>{\raggedright\arraybackslash}p{0.15\textwidth}
>{\raggedright\arraybackslash}p{0.25\textwidth}
>{\centering\arraybackslash}p{0.12\textwidth}
>{\centering\arraybackslash}p{0.12\textwidth}
>{\centering\arraybackslash}p{0.12\textwidth}
>{\raggedright\arraybackslash}X
}
\toprule
Sweep & Settings
& \(\mathrm{EMA}=0\)
& \(\mathrm{EMA}=0.5\)
& \(\mathrm{EMA}=0.9\)
& \(\mathrm{EMA}=1\)
\\
\midrule
S1: EMA distance
& \(w\in\{1,w_{\mathrm{sc}}\}\)
& 0.998; 4/8
& 0.995; 7/8
& 0.985; 6/8
& 0.113; 0/8 \\
S3: 32D ratio clip \(\delta\)
& \(\delta\in\{10^{-5},0.05,0.2,\mathrm{none}\}\) \(\times\) \(w\in\{1,w_{\mathrm{sc}}\}\)
& 0.910; 1/8
& 0.982; 3/8
& 0.977; 2/8
& 0.459; 0/8 \\
\bottomrule
\end{tabularx}
\end{table}

The following tables report the detailed results.  Each paired entry is ``ordinary / score-calibrated''; bold values meet the \(0.90\) reward gate.
\begin{table}[H]
\centering
\caption{
Detailed High-D S1 old-policy EMA-decay sweep. Each entry reports ordinary
/ score-calibrated train-weight final reward at seed 13; higher is better.
Cells use the setting-dependent random streams of the screening run.
}
\label{tab:app-highd-s1-ema-distance-detail}
\scriptsize
\setlength{\tabcolsep}{7.0pt}
\renewcommand{\arraystretch}{1.12}
\begin{tabular}{lcccc}
\toprule
Dataset & EMA \(=0\) & EMA \(=0.5\) & EMA \(=0.9\) & EMA \(=1\) \\
\midrule
GMM-4D
& \textbf{0.997} / \textbf{0.998}
& \textbf{0.986} / \textbf{0.967}
& \textbf{0.985} / \textbf{0.926}
& 0.098 / 0.094 \\
GMM-8D
& \textbf{0.996} / 0.464
& \textbf{0.995} / \textbf{0.969}
& \textbf{0.979} / \textbf{0.914}
& 0.080 / 0.090 \\
GMM-16D
& \textbf{0.992} / 0.377
& \textbf{0.991} / \textbf{0.952}
& \textbf{0.959} / 0.892
& 0.074 / 0.093 \\
GMM-32D
& 0.590 / 0.188
& \textbf{0.971} / 0.083
& \textbf{0.951} / 0.211
& 0.078 / 0.113 \\
\bottomrule
\end{tabular}
\par\vspace{2pt}
\parbox{0.94\textwidth}{\footnotesize\textbf{Conclusion.} EMA \(=0.5\) has the largest fraction of passing S1 configurations (\(7/8\)), followed by EMA \(=0.9\) (\(6/8\)) and EMA \(=0\) (\(4/8\)); all EMA \(=1\) configurations fail. The score-calibrated weight becomes increasingly fragile with dimension. These single-seed associations do not establish a causal stabilization mechanism.}
\end{table}

\begin{table}[H]
\centering
\caption{
Detailed High-D S3 ratio-clip sweep.  Each entry reports ordinary /
score-calibrated train-weight final reward at seed 13; higher is better.
}
\label{tab:app-highd-s3-ratio-clip-detail}
\scriptsize
\setlength{\tabcolsep}{4.0pt}
\renewcommand{\arraystretch}{1.06}
\begin{tabular}{llcccc}
\toprule
Dataset & Ratio clip \(\delta\) & EMA \(=0\) & EMA \(=0.5\) & EMA \(=0.9\) & EMA \(=1\) \\
\midrule
\multirow{4}{*}{GMM-4D}
& \(10^{-5}\) & \textbf{0.905} / 0.649 & 0.480 / 0.506 & 0.168 / 0.238 & 0.082 / 0.085 \\
& 0.05 & \textbf{0.997} / \textbf{0.999} & \textbf{0.987} / \textbf{0.969} & \textbf{0.980} / \textbf{0.930} & 0.101 / 0.092 \\
& 0.2 & \textbf{0.996} / \textbf{0.998} & \textbf{0.996} / \textbf{0.999} & \textbf{0.990} / \textbf{0.962} & 0.118 / 0.100 \\
& None & \textbf{0.996} / \textbf{0.998} & \textbf{0.996} / \textbf{0.998} & \textbf{0.997} / \textbf{0.995} & 0.379 / 0.645 \\
\midrule
\multirow{4}{*}{GMM-8D}
& \(10^{-5}\) & \textbf{0.993} / 0.752 & 0.701 / 0.647 & 0.216 / 0.385 & 0.078 / 0.091 \\
& 0.05 & \textbf{0.996} / 0.486 & \textbf{0.996} / \textbf{0.964} & \textbf{0.973} / \textbf{0.917} & 0.076 / 0.089 \\
& 0.2 & \textbf{0.996} / 0.864 & \textbf{0.996} / \textbf{0.905} & \textbf{0.994} / \textbf{0.912} & 0.089 / 0.105 \\
& None & \textbf{0.995} / 0.899 & \textbf{0.997} / 0.850 & \textbf{0.993} / \textbf{0.986} & 0.216 / 0.543 \\
\midrule
\multirow{4}{*}{GMM-16D}
& \(10^{-5}\) & \textbf{0.957} / 0.452 & \textbf{0.926} / 0.738 & 0.352 / 0.529 & 0.073 / 0.092 \\
& 0.05 & \textbf{0.990} / 0.371 & \textbf{0.992} / \textbf{0.950} & \textbf{0.959} / 0.896 & 0.080 / 0.091 \\
& 0.2 & \textbf{0.974} / 0.271 & \textbf{0.995} / \textbf{0.957} & \textbf{0.987} / \textbf{0.932} & 0.086 / 0.104 \\
& None & \textbf{0.995} / 0.192 & \textbf{0.990} / \textbf{0.944} & \textbf{0.953} / 0.009 & 0.122 / 0.499 \\
\midrule
\multirow{4}{*}{GMM-32D}
& \(10^{-5}\) & 0.531 / 0.166 & \textbf{0.950} / 0.117 & 0.748 / 0.227 & 0.070 / 0.104 \\
& 0.05 & \textbf{0.910} / 0.199 & \textbf{0.982} / 0.116 & \textbf{0.948} / 0.196 & 0.072 / 0.111 \\
& 0.2 & 0.185 / 0.162 & \textbf{0.980} / 0.022 & \textbf{0.977} / 0.235 & 0.074 / 0.120 \\
& None & 0.225 / 0.157 & 0.056 / 0.558 & 0.889 / 0.064 & 0.095 / 0.459 \\
\bottomrule
\end{tabular}
\par\vspace{2pt}
\parbox{0.96\textwidth}{\footnotesize\textbf{Conclusion.} The ratio-clip effect is EMA-dependent and non-monotonic. As the dimension increases, EMA$ = 0$ performs worse than EMA$ = 0.5$ and EMA$ = 0.9$, and EMA$ = 1$ fails in all dimensions. The training stability is because of the training objective rather than the ratio estimation.}
\end{table}

Together, \tabref{tab:app-highd-mechanism-sweep},
\tabref{tab:app-highd-s1-ema-distance-detail}, and
\tabref{tab:app-highd-s3-ratio-clip-detail} support the following conclusions:
\begin{itemize}[leftmargin=*]
    \item In high dimension, training becomes more unstable, EMA$ = 0.5$ or EMA$ = 0.9$ is more stable than EMA$ = 0$ or EMA $= 1$. The instability is because of the training objective rather than the ratio estimation.
    \item The score-calibrated weight (i.e. $w = w_{\mathrm{sc}}$) becomes increasingly fragile with dimension and does not perform well in 32D at any EMA. This is because the score-calibrated weight can enlarge the noise scale.
    \item In 32D, after stabilizing the training, the ratio clipping value can be helpful to the update signal.
\end{itemize}

\subsection{Real Image Experiments}
\label{app:real-image-experiments}

\subsubsection{Experiment Setup}
The experiments in this section operate directly on raw MNIST and CIFAR-10
pixels.  Raw-pixel CNF likelihood and pointwise decomposition audits are inaccurate and prohibitively expensive at dimensions 784 and 3,072. We therefore do not use these runs as evidence for RQ1 or RQ2. The experiments' purpose is to test whether the two most useful on-policy controls from the synthetic experiments---EMA and ratio clipping---continue to affect reward optimization with image-valued velocity fields.
For each dataset, we first train an unconditional U-Net velocity field with ordinary fixed-target CFM and then freeze its EMA-weight checkpoint at step 50,000 as the initialization policy. The reward is the frozen classifier probability assigned to class 6. Thus the reported reward is bounded in \([0,1]\), but it is a task reward rather than a likelihood or an image-quality metric. \tabref{tab:app-image-raw-setup} gives the complete source and on-policy configurations used in the reported experiments.

\begin{table}[H]
\centering
\caption{\textbf{Raw-image settings for RQ3.}  ``Source checkpoint'' denotes the off-policy checkpoint loaded by every on-policy run.  Rollout and final evaluation sample counts are global across eight GPUs.}
\label{tab:app-image-raw-setup}
\scriptsize
\setlength{\tabcolsep}{3.2pt}
\renewcommand{\arraystretch}{1.06}
\begin{tabularx}{\textwidth}{
>{\raggedright\arraybackslash}p{0.27\textwidth}
>{\centering\arraybackslash}p{0.20\textwidth}
>{\centering\arraybackslash}p{0.22\textwidth}
>{\centering\arraybackslash}X
}
\toprule
\textbf{SETTING} & \textbf{SHARED} & \textbf{MNIST} & \textbf{CIFAR-10} \\
\midrule
\multicolumn{4}{c}{\textbf{Data and source model}} \\
\midrule
Input / dimension & Raw pixels & \(1\!\times\!28\!\times\!28\) / 784
& \(3\!\times\!32\!\times\!32\) / 3,072 \\
Data used for source CFM & Full class mixture & Train and test splits & Train split \\
Velocity field & U-Net & Two resolution levels & Four resolution levels with attention \\
Model range & -- & \([0,1]\) & \([-1,1]\) \\
Path / \(\varepsilon\) & Linear Gaussian / \(10^{-4}\) & -- & -- \\
Source objective & Ordinary CFM, \(w(t)=1\) & -- & -- \\
Source optimizer / LR & AdamW / \(10^{-4}\), cosine schedule & -- & -- \\
Source batch per GPU & -- & 256 & 64 \\
Configured source horizon & -- & 50,000 steps & 50,000 steps \\
Source checkpoint & Step 50,000, EMA weights & \multicolumn{2}{c}{Weight EMA decay 0.999} \\
Reward & Frozen class-6 probability & MNIST CNN & CIFAR-10 ResNet \\
\midrule
\multicolumn{4}{c}{\textbf{On-policy optimization}} \\
\midrule
Seed / hardware & \multicolumn{3}{c}{13 / 8 GPUs} \\
Ratio surrogate & \multicolumn{3}{c}{Ordinary CFM-only ratio} \\
Rounds / update steps & \multicolumn{3}{c}{100 / 1 per round} \\
Update LR / rollout samples & \multicolumn{3}{c}{\(2\!\times\!10^{-5}\) / 64} \\
Group size / advantage clip & \multicolumn{3}{c}{16 / 2} \\
Max log ratio / gradient clip & \multicolumn{3}{c}{2 / 0.1} \\
Source-VF MSE \(\beta_{\mathrm{src}}\) & \multicolumn{3}{c}{0.1 (configured through \texttt{kl\_beta})} \\
Round-old VF MSE \(\beta_{\mathrm{old}}\) & \multicolumn{3}{c}{0} \\
ODE steps / final samples & \multicolumn{3}{c}{100 / 128} \\
Sample projection & Final clamp & \([0,1]\) & \([-1,1]\) \\
Reward success gate & \multicolumn{3}{c}{Final sampled reward mean \(\geq 0.90\)} \\
\bottomrule
\end{tabularx}
\end{table}

We run two mechanism sweeps.  S1 varies the EMA decay \(\rho\in\{0,0.5,0.9,1\}\) while fixing the GRPO ratio clip to \(\delta=0.05\).  Here \(\rho=0\) refreshes the reference every round and \(\rho=1\) keeps the source checkpoint fixed.  S3 crosses the same four EMA values with ratio clipping value \(\delta\in\{10^{-5},0.05,0.2,\mathrm{None}\}\).  Every other setting is held fixed as in \tabref{tab:app-image-raw-setup}.

\subsubsection{Do EMA and ratio clipping remain useful on raw images?}

\paragraph{Summary.}
\tabref{tab:app-image-raw-mechanism-summary} reports the best final reward
and the number of settings passing the \(0.90\) reward gate for each EMA value. Both S1 and S3 results show that $EMA = 0$ perform best and ratio clipping is helpful to the update signal.

\begin{table}[H]
\centering
\caption{Raw-image RQ3 mechanism summary at seed 13.  Each cell is
``best final reward; successful settings / tested settings.''  The success
gate is final sampled reward mean \(\geq0.90\).}
\label{tab:app-image-raw-mechanism-summary}
\scriptsize
\setlength{\tabcolsep}{7.0pt}
\renewcommand{\arraystretch}{1.08}
\begin{tabular}{lcccc}
\toprule
Sweep & EMA \(=0\) & EMA \(=0.5\) & EMA \(=0.9\) & EMA \(=1\) \\
\midrule
S1: EMA distance & \textbf{1.000}; 2/2 & \textbf{0.984}; 2/2 & 0.818; 0/2 & 0.541; 0/2 \\
S3: Ratio clip & \textbf{1.000}; 6/8 & \textbf{1.000}; 6/8 & \textbf{1.000}; 3/8 & \textbf{0.992}; 1/8 \\
\bottomrule
\end{tabular}
\end{table}

\paragraph{S1: EMA distance.}
\tabref{tab:app-image-raw-s1-ema-detail} gives the individual S1 rewards. $EMA = 0$ performs best and $EMA = 1$ fails in all datasets.

\begin{table}[H]
\centering
\caption{Detailed raw-image S1 EMA-decay sweep. Entries are final-checkpoint
sampled reward means for ordinary CFM-ratio training at seed 13; higher is
better.  Bold values meet the \(0.90\) reward gate.}
\label{tab:app-image-raw-s1-ema-detail}
\scriptsize
\setlength{\tabcolsep}{9.0pt}
\renewcommand{\arraystretch}{1.08}
\begin{tabular}{lcccc}
\toprule
Dataset & EMA \(=0\) & EMA \(=0.5\) & EMA \(=0.9\) & EMA \(=1\) \\
\midrule
MNIST & \textbf{1.000} & \textbf{0.984} & 0.818 & 0.541 \\
CIFAR-10 & \textbf{1.000} & \textbf{0.942} & 0.799 & 0.272 \\
\bottomrule
\end{tabular}
\end{table}

\paragraph{S3: Ratio clipping.}
The detailed results in \tabref{tab:app-image-raw-s3-clip-detail} show that EMA $=0$ performs well and ratio clipping can help the training for larger EMA.

\begin{table}[H]
\centering
\caption{Detailed raw-image S3 ratio-clip sweep.  Entries are final-checkpoint
sampled reward means for ordinary CFM-ratio training at seed 13; higher is
better.  Bold values meet the \(0.90\) reward gate based on the unrounded
reward.}
\label{tab:app-image-raw-s3-clip-detail}
\scriptsize
\setlength{\tabcolsep}{7.0pt}
\renewcommand{\arraystretch}{1.08}
\begin{tabular}{llcccc}
\toprule
Dataset & Ratio clip \(\delta\) & EMA \(=0\) & EMA \(=0.5\) & EMA \(=0.9\) & EMA \(=1\) \\
\midrule
\multirow{4}{*}{MNIST}
& \(10^{-5}\) & \textbf{1.000} & \textbf{0.933} & 0.738 & 0.702 \\
& 0.05 & \textbf{1.000} & \textbf{0.984} & 0.742 & 0.571 \\
& 0.2 & \textbf{1.000} & \textbf{1.000} & 0.847 & 0.812 \\
& None & \textbf{1.000} & \textbf{1.000} & \textbf{1.000} & \textbf{0.992} \\
\midrule
\multirow{4}{*}{CIFAR-10}
& \(10^{-5}\) & \textbf{1.000} & 0.815 & \textbf{0.918} & 0.673 \\
& 0.05 & \textbf{0.952} & 0.799 & 0.775 & 0.769 \\
& 0.2 & 0.731 & \textbf{0.942} & 0.870 & 0.614 \\
& None & 0.647 & \textbf{1.000} & \textbf{0.953} & 0.899 \\
\bottomrule
\end{tabular}
\end{table}

\end{document}